\documentclass[11pt]{article}

\usepackage{newtxtext,newtxmath}
\usepackage[pdftex,xetex]{graphicx}
\usepackage[letterpaper,margin=1in]{geometry}

\renewenvironment{abstract}
	{\quotation}
	{\endquotation}

\date{}

\makeatletter
\renewcommand{\fnum@figure}{\textbf{Figure \thefigure}}
\renewcommand{\fnum@table}{\textbf{Table \thetable}}
\makeatother

\usepackage{scicite}

\usepackage{url}

\usepackage[dvipsnames,svgnames,table]{xcolor}
\usepackage{subcaption,booktabs,adjustbox,multirow}

\newcommand{\new}[1]{{#1}}
\newenvironment{newenv}{}{}

\newcommand{\aed}[1]{\begin{aligned} #1 \end{aligned}}
\newcommand{\beq}[1]{\begin{equation}#1\end{equation}}
\newcommand{\beqs}[1]{\begin{equation*}#1\end{equation*}}

\newcommand{\inner}[2]{\left\langle #1,#2 \right\rangle}
\newcommand{\rbr}[1]{\left(#1\right)}
\newcommand{\sbr}[1]{\left[#1\right]}
\newcommand{\cbr}[1]{\left\{#1\right\}}

\newcommand{\abr}[1]{\left|#1\right|}
\newcommand{\abs}[1]{\left|#1\right|}
\newcommand{\norm}[1]{\left\lVert #1 \right\rVert}

\newcommand{\cmark}{\checkmark}
\newcommand{\xmark}{\textsf{X}}
\renewcommand{\deg}{\ensuremath{^\circ}}

\providecommand\f[2]{\frac{#1}{#2}}

\DeclareMathOperator*{\argmin}{argmin}

\DeclareMathOperator*{\E}{\mathbb{E}}

\providecommand{\ind}[1]{\mathbf{1}_{\cbr{#1}}}

\newtheorem{theorem}{Theorem}
\newtheorem{proof}{Proof}

\newtheorem{lemma}[theorem]{Lemma}

\def \s {\sigma}
\def \w {\omega}

\def \a {\alpha}
\def \b {\beta}

\def \e {\epsilon}
\def \g {\gamma}

\def \l {\lambda}

\def \D {\Delta}

\def \L {\Lambda}

\def \Om {\Omega}

\def \LL {{\mathcal{L}}}

\def \BB {{\mathcal{B}}}

\def \integers {\mathbb{Z}}

\def \reals {\mathbb{R}}

\def \R {\mathbb{R}}

\def \B {\BB}

\def \ff {{\text{F}^3}}
\def \vv {{\text{V}^3}}
\def \ii {{\text{I}^3}}
\def \flow {{\text{v}}}
\def \disp {{\text{d}}}

\newcommand{\KL}{\text{KL}}

\graphicspath{{images/}}

\def\scititle{
	Fast Feature Field (F$^3$): A Predictive Representation of Events
}
\title{\bfseries \boldmath \scititle}

\author{
	Richeek~Das$^\ast$,
	Kostas~Daniilidis,
	Pratik~Chaudhari$^\ast$\and
	\small Computer and Information Science, University of Pennsylvania, Philadelphia, 19104, USA.\and
	\small$^\ast$Corresponding authors. Email: richeek@seas.upenn.edu, pratikac@seas.upenn.edu
}

\begin{document} 

\maketitle

\begin{abstract} \bfseries \boldmath
This paper develops a mathematical argument and algorithms for building representations of data from event-based cameras, that we call Fast Feature Field ($\ff$). We learn this representation by predicting future events from past events and show that it preserves scene structure and motion information. $\ff$ exploits the sparsity of event data and is robust to noise and variations in event rates. It can be computed efficiently using ideas from multi-resolution hash encoding and deep sets---achieving 120 Hz at HD and 440 Hz at VGA resolutions. $\ff$ represents events within a contiguous spatiotemporal volume as a multi-channel image, enabling a range of downstream tasks. We obtain state-of-the-art performance on optical flow estimation, semantic segmentation, and monocular metric depth estimation, on data from three robotic platforms (a car, a quadruped robot and a flying platform), across different lighting conditions (daytime, nighttime), environments (indoors, outdoors, urban, as well as off-road) and dynamic vision sensors (resolutions and event rates). Our implementations can predict these tasks at 25–75 Hz at HD resolution.
All code is available at \url{https://www.seas.upenn.edu/~richeek/f3}.
\end{abstract}




\section{Introduction}
\label{s:intro}

\begin{figure}
\centering
\includegraphics[width=0.7\linewidth]{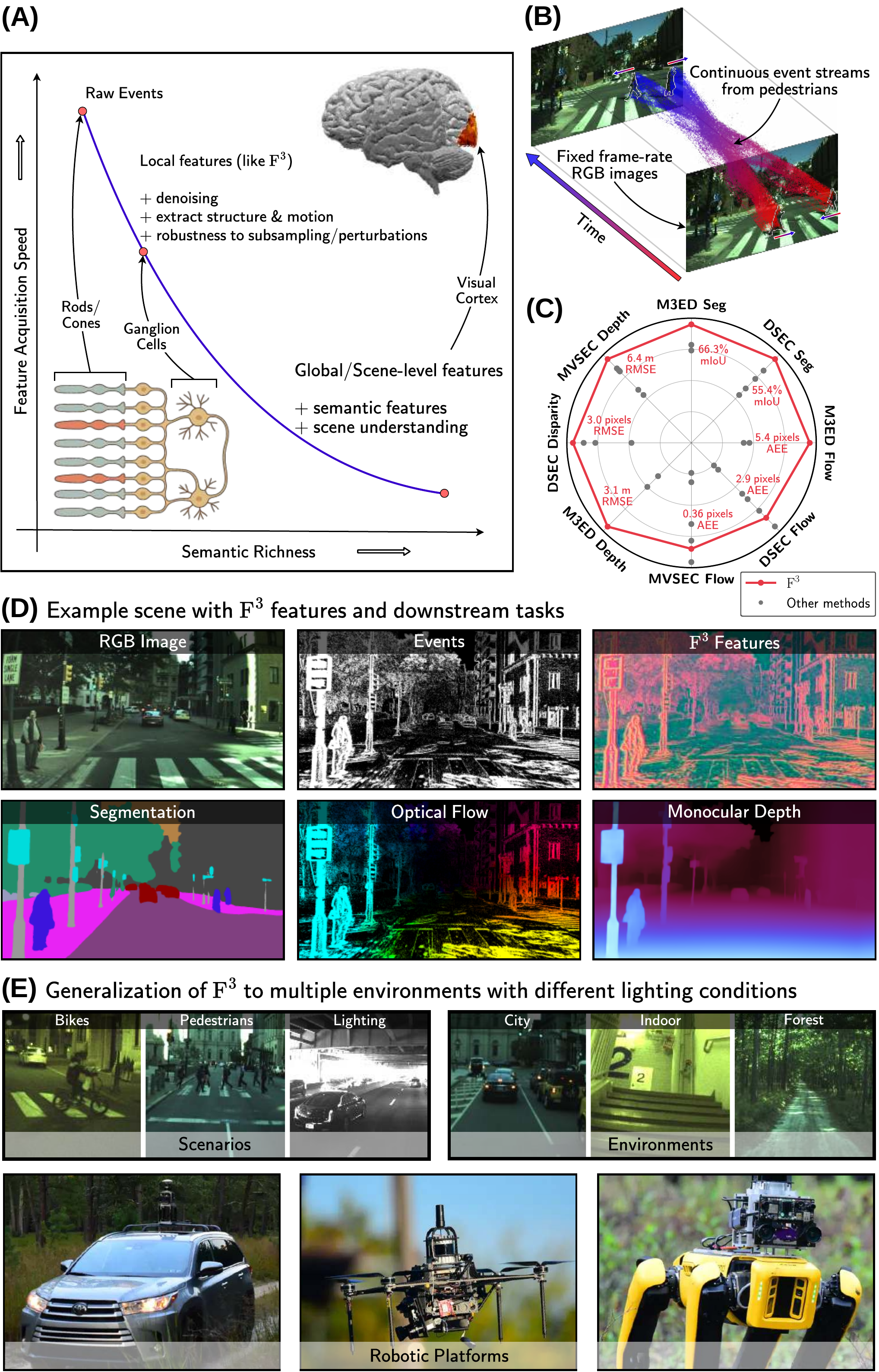}
\caption{
\small
\textbf{Overview.}
\textbf{(A)} Event cameras emulate rods and cones, which transduce light in the retina. Their measurements are fast, but local, redundant, and noisy. Successive parts of the retina produce more global, less redundant, and less noisy features.
$\ff$ resembles features of the ganglion cells in the retina that are fast, informative of structure and motion in the scene, and robust to noise.
\textbf{(B)} Event cameras are asynchronous sensors that do not suffer from temporal aliasing.
\textbf{(C)} $\ff$ representation is effectively a multi-channel image.
It can be used with any computer vision algorithm or architecture. $\ff$-based approaches outperform existing methods (in gray) on semantic segmentation, optical flow estimation, and monocular depth estimation.
\textbf{(D)} A typical urban driving scene with an RGB image, events from the past few milliseconds (displayed as an image), $\ff$ features (top three principal components), segmentation masks, optical flow from ego motion (colors denote direction), and monocular depth estimate.
\textbf{(E)} This paper studies $\ff$ in different environments, lighting conditions and robotic platforms.
}
\label{fig:fig1}
\end{figure}


In some of the earliest work on neuromorphic perception, Mead and Mahowald sought to emulate the marvelous information processing of the biological retina in silicon \cite{mead1988silicon}.
Event-based cameras exemplify the state of the art of this field today. They offer a comparable dynamic range to the retina (80--120 dB vs.\ $\sim$140 dB, respectively), much higher temporal resolution (less than 100 $\mu$s vs.\ 10--30 ms), and similar precision ($\sim$1 ns).
Their spatial acuity is lower (1 MP vs.\ $\sim$600 MP) but remarkable nonetheless.
Event cameras are very attractive to roboticists. They are asynchronous sensors. They can be much more effective than RGB cameras. These sensors have the potential to allow robots to operate in different lighting conditions (harsh daylight as well as night-time, indoors and outdoors), at very high speeds (a quadrotor rotates at $\sim$700 $\deg$s$^{-1}$ when it flips but event cameras do not suffer from motion blur), and require little power and weight (they weigh only as much as an RGB camera).

The biological retina is a three-layer feed-forward circuit. Event cameras emulate the first layer, rods and cones.
They register whether the logarithmic intensity at a pixel exceeds a threshold since the last such ``event''.
The hallmark of information processing in the retina, however, is the neural circuitry beyond these elementary units.
\new{The second layer consists of $\sim$10 bipolar cell types that perform computations such as distinguishing bright from dark spots. About 20 types of ganglion cells in the third layer detect luminance contrast, motion along different directions, and color. About $\sim$40 types of ``interneurons''---horizontal and amacrine cells remove irrelevant parts of the stimulus by inhibition.
}
As a consequence, the output from an event camera on a typical robot is not as succinct as that of the biological retina, $\sim$50M events per second versus $\sim$0.5 Hz features from ganglion cells in the retina \cite{koch2006much}.
This imposes severe constraints on the latency of algorithms that process event data, thereby increasing power consumption.
The retina averages its inputs and uses negative feedback to work around thermal or synaptic noise.
Typical event sensors have limited filtering to average over readout and refractory noise.
As a consequence, algorithms need to be resilient to noise in event camera data.

Taking forward the intellectual program of Mead and Mahowald requires us to develop effective and efficient representations for event data. This is one of the most significant problems in the field today. While there are many existing approaches specialized to various settings, no general principles exist for representation learning on event data.

\paragraph{Contributions of this paper}

\begin{itemize}
\item We show that a representation of event data that uses past events to predict future events retains information about structure and motion in the scene.

\item We instantiate this argument to develop a representation called Fast Feature Field ($\ff$). $\ff$ exploits the sparsity of event data using a multi-resolution hash encoder and permutation-invariant architecture. We develop an objective to train $\ff$ that makes it robust to noise and varying event rates.

\item $\ff$ represents events in a contiguous spatiotemporal volume as a multi-channel image. It can be used in any standard computer vision algorithm or architecture built for RGB data. We show that $\ff$-based approaches achieve state-of-the-art performance in four robot perception tasks (supervised semantic segmentation, unsupervised and supervised optical flow, unsupervised stereo matching, and supervised monocular metric depth estimation) on data from three platforms (driving, quadruped locomotion, and a flying platform).

\item $\ff$ enables real-time event-based robot perception.
Our implementation of $\ff$ runs at 120 Hz and 440 Hz on HD and VGA resolutions, respectively. We can predict different downstream tasks at 25--75 Hz at HD resolution. This is roughly 2--5 times faster than the state of the art.

\item $\ff$ enables robust event perception---without additional training---across data from different robotic platforms, lighting and environmental conditions (daytime vs. night-time, indoors vs. outdoors, urban vs. off-road) and dynamic vision sensors (with different resolutions and event rates).
\end{itemize}

\new{Fig.~\ref{fig:fig1} provides an overview of these contributions and gives examples of the different robotic platforms, tasks, and environmental conditions addressed in this paper.}

\section{Results}
\label{s:results}




\begin{figure}
    \centering
    \includegraphics[width=0.75\linewidth]{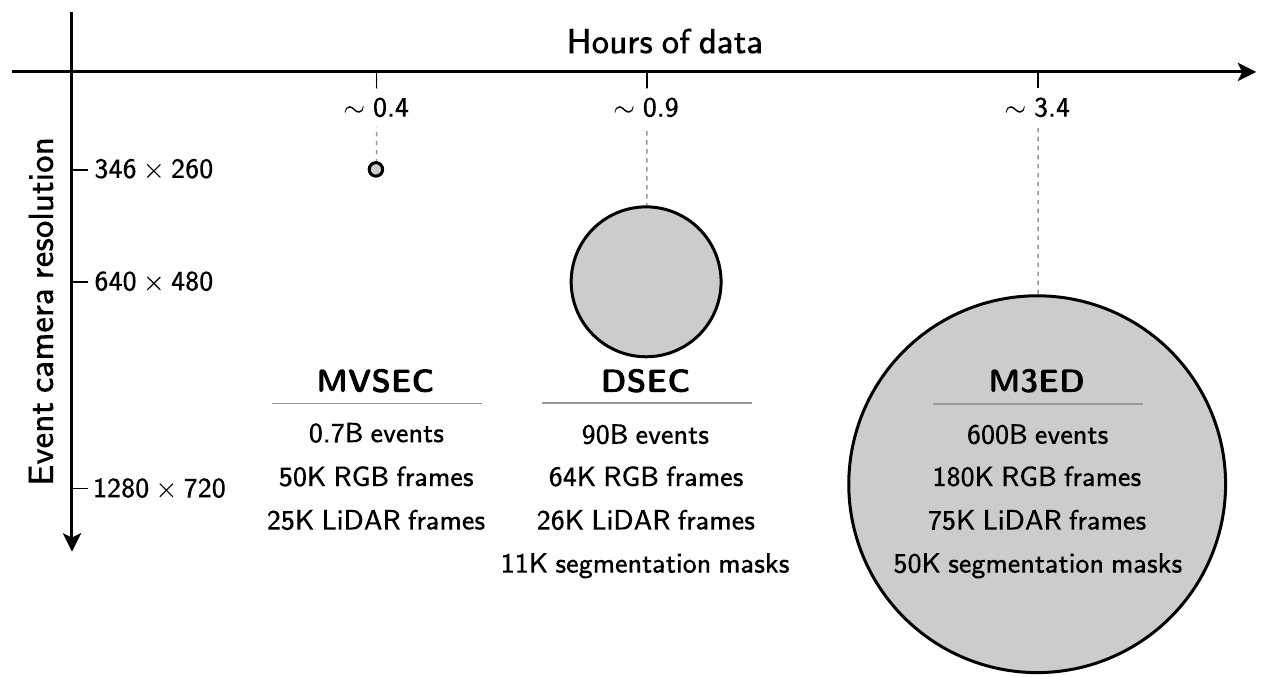}
    \caption{
    \small
    \textbf{Overview of the datasets used in this paper.} This work focuses on data collected from different robotic platforms (human-driven car, a flying platform, and a quadruped robot), (ii) different event cameras and sensors (VGA and HD resolutions in addition to a smaller 346 $\times$ 260 sensor), and (iii) experimental settings (urban and off-road environments, both indoor and outdoor locomotion tasks, different lighting conditions).}
    \label{fig:datasets}
\end{figure}

\subsection{Robotic platforms, sensors, experimental settings and tasks}
\new{Before developing theoretical aspects of $\ff$, we discuss the data used in this paper. Fig.~\ref{fig:fig1} (E) and Fig.~\ref{fig:datasets} give an idea of different nuisances in event data, robotic platforms, auxiliary sensors used for computing ground-truth, and operating environments.}
MVSEC \cite{mvsec} contains $\sim$0.4 hours of urban driving data with different lighting conditions (daytime and nighttime). Event data is obtained from a 346 $\times$ 260 resolution DAVIS346 camera, ground-truth depth and optical flow are obtained from LiDAR.
DSEC \cite{Gehrig21ral} is a slightly larger, with $\sim$0.9 hours of data from urban areas and mountains under different illumination conditions. Events are recorded from a 640 $\times$ 480 resolution Prophesee Gen 3.1 camera. Depth is obtained from LiDAR.
M3ED \cite{chaney2023m3ed} is a large dataset. It contains $\sim$3.4 hours of event data from daytime, nighttime, illumination changes in tunnels, etc. It contains three robotic platforms: car, quadruped robot, and a flying platform. Outdoor scenes have pedestrians, cars, bikes, and dense foliage; indoor scenes are cluttered lab environments. M3ED uses a time-synchronized sensor suite with stereo 1280 $\times$ 720 resolution Prophesee EVK4 HD cameras, LiDAR, and stereo RGB cameras.
Sec.~\ref{s:app:details} provides more details of these datasets.






\subsection{Effectiveness of predictive representations of events in capturing structure and motion information}
\label{s:theory}

\new{This section develops a mathematical argument for $\ff$ using ideas in signal processing.
We show that a representation of past events that can predict future events within a small spatio-temporal window retains information about structure and motion in the scene.
This argument is instantiated in two parts.
First, we show that learning optimal representations of noisy spatio-temporal data is equivalent to identifying a basis in which the original data is sparse.
In essence, this basis is the architecture for $\ff$.
Next, we show that predicting future events from past events can recover both the optimal basis and the dynamics of the representation.
This suggests a training objective for $\ff$.
Sec.~\ref{s:instantiation} instantiates this argument into a practical implementation for $\ff$.
}

Consider a fixed Lambertian scene under constant illumination.
Pixels in an event camera respond to changes in the logarithmic intensity of incident light. An event $e(t, u) \in \{+1, -1\}$ is ``triggered'' at time $t \in \integers_+$ and pixel $u = (u_1, u_2) \in \Om \subset \integers^2$ when the log-intensity changes by some pre-defined threshold.%
\footnote{We work with event timestamps discretized into 1 ms intervals, i.e., events within 1 ms of each other at the same pixel are identical.}
Let $e(t, u) = 0$ when events are not triggered. We ignore the polarity of events, and therefore $e(t, u) \in \{1,0\}$ indicates the presence or absence of an event.
Given viewpoints along a camera trajectory $x \equiv (x(t))_{t \in \integers_+}$ with $x(t) \in \text{SE}(3)$, let $\xi$ be a sufficient statistic of the scene for the events, i.e.,
\beq{
    e(t,u) = \pi(\xi(t, u; x);\nu)
    \label{eq:event_generation}
}
for some function $\pi$ and noise $\nu$. The event $e(t, u)$ is a random variable that depends on the past camera trajectory $x$ and the statistic $\xi(t,u; x)$. For example, consider Marr's primal sketch (edges and depth discontinuities) \cite{marrVisionComputationalInvestigation2010}. The statistic $\xi(t,u; x)$ is the projection of the primal sketch upon the image plane, followed by differentiation in time, with $\pi$ being the acquisition mechanism of an event camera.
We will design an estimator $\hat \xi$ using events $e$ without knowledge of the camera trajectory $x$. The dependence on $x$ is omitted for clarity.

\paragraph{Ideal denoising to obtain sparse representations}
For a fixed $t \in \integers_+$ and $s \in [0, t)$, let events be $\reals \ni e(s, u) = \xi(s, u) + \nu$ where $\nu \sim N(0, 1)$ is assumed to be unit-variance for clarity.
This simple model will be useful to develop the argument.
Estimating $\hat \xi$ is a denoising problem because the acquisition mechanism $\pi$ adds Gaussian noise to the statistic, which can be solved by minimizing the risk
\beq{
    R(\hat\xi,\xi)=\norm{\hat \xi - \xi}_2^2
    \triangleq \sum_{s=0}^t \sum_{u \in \Om} \rbr{\hat \xi(s, u) - \xi(s, u)}^2,
    \label{eq:estimator_basic}
}
on average over noise $\nu$.
Suppose we have a basis $\BB = \{ \varphi_1, \dots, \varphi_n \}$ with each function $\varphi_i$ supported on $\integers_+ \times \Om$.
For example, $\BB$ could be the Fourier or wavelet bases.
We will estimate coefficients $(\xi_\BB)_i = \inner{\xi}{\varphi_i}$ of $\xi$ projected on this basis.
When the basis is orthonormal, Parseval's identity gives
\(
R(\hat\xi,\xi)=\norm{\hat \xi_\BB - \xi_\BB}^2 \triangleq \sum_{i=1}^n \rbr{ (\hat\xi_\BB)_i - (\xi_\BB)_i}^2.
\)
It is difficult to minimize this risk in general. But diagonal denoisers, when $(\hat \xi_\BB)_i$ only depends on the input projected on the $i^{\text{th}}$ basis element $(e_\BB)_i$, are an important special case. The minimal risk for such denoisers is $R(\xi, \BB) = \sum_i \min((\xi_\BB)_i^2, 1)$ \cite{donoho1994idealspatial}. If the hard threshold non-linearity is $\s_\tau(v) = v \ind{\abs{v > \tau}}$ for $v,\tau \in \reals$, then the estimator
\beq{
    (\hat \xi_\BB)_i = \s_{\sqrt{2 \log n}}((e_\BB)_i), \quad \forall i \in \cbr{1,\dots, n},
    \label{eq:hat_e_BB_i}
}
has $R(\hat \xi, \xi) \leq (2 \log n + 2.4) (1 + R(\xi, \BB))$, which is sub-optimal only by a $\log n$ factor.
Diagonal denoisers are nearly optimal if the signal is sparse in $\BB$ \cite{mallat2008wavelet}.

But it is challenging to identify a basis in which real-world data is sparse. But one could search over a library $\LL$, e.g., Fourier basis elements supported on different frequency bands, and select
\beq{
    \hat \BB = \argmin_{\BB \in \LL} \sum_{i=1}^n \min((e_\BB)_i^2, \L_n)
    \label{eq:hat_BB_L}
}
where $\L_n = \l^2 (1 + \sqrt{2 \log M_n})^2$, $\l > 8$ and $M_n$ is the number of basis elements in the library $\LL$. The hard-threshold estimator in this basis
\(
    (\hat \xi_{\hat{\BB}})_i = \s_{\sqrt \L_n} ((e_{\hat \BB})_i),
\)
for $i \leq n$ satisfies
\[
    R(\hat\xi,\xi) \leq (1 - 8/\l)^{-1} \L_n\ \min_{\BB \in \LL} R(\xi, \BB),
\]
with probability greater than $(1 - e/M_n)$ where $e$ is Euler's number \cite{donoho1994ideal}. Typically, the ideal risk $\min_{\BB \in \LL} R(\xi, \BB)$ grows as a polynomial of $n$. For example, in a library of wavelet packets, the number of basis elements is $M_n \sim n\log_2 n$. In such a library, this strategy is only sub-optimal by a factor of $\L_n \sim \log n$. Similar results hold for more general bases, e.g., Riesz basis or frames~\cite{mallat2008wavelet}[Chapter 11]. This discussion suggests that we can search over a library to (i) identify the basis in which input data are sparse, and (ii) denoise the data using hard thresholding in this basis to obtain a representation.

\paragraph{Ideal denoising and spatiotemporal prediction to obtain event representations}
We next specialize the argument above to event data which have spatiotemporal dynamics. The key idea is to simultaneously identify a basis to denoise events and the dynamics of the sufficient statistic $\xi$. Let $\xi^-(s, u) \in \reals$ denote $\xi$ restricted to $s \in [t-\D t, t)$. Similarly, let $\xi^+(s, u)$ denote the restriction to $s \in [t,t+\D t)$. The notations $e^+$ and $e^-$ are analogous.
Suppose
\beq{
    \xi^+ = A  \xi^- + \zeta,
    \label{eq:dynamics}
}
where $\zeta \sim N(0, I_{m \times m})$ is Gaussian noise with $m = \D t \times \abs{\Om}$. The linear map $A: \reals^m \rightarrow \reals^m$ transforms past statistics to future ones, up to noise $\zeta$. Under two technical assumptions, namely that the operator norm of $A$ is bounded and that $A$ does not lead to large changes in the sparsity, we can prove the following theorem.

\begin{theorem}
\label{thm:main}
The dynamics and the denoising basis
\beq{
    \hat A, \hat \BB = \argmin_{A,\ \BB \in \LL} \cbr{\norm{e^+ - A \hat \xi^-}_2^2 + \L_n \norm{\hat \xi^-_\BB}_0}
    \label{eq:objective}
}
where $(\hat \xi^-_\BB)_i = \s_{\sqrt \L_n} ((e^-_\BB)_i)$ for $i \leq n$ and $\hat \xi^-$ and $\hat \xi^-_\BB$ denotes the same signal represented in the cardinal and $\BB$ bases respectively, satisfy
\beq{
    \norm{\xi^+ - \hat A \hat \xi^-}_2^2 \leq c_1 R_{\text{dynamics}}(\xi, \LL) + c_2 \L_n R_{\text{denoising}}({\xi}^-, \LL),
    \label{eq:thm}
}
with probability at least $1 - c e/M_n$. Here $c$ is a constant greater than 10. Constants $c_1$ and $c_2$ only depend on the assumed properties of the dynamics operator $A$.
The quantities $R_{\text{dynamics}}(\xi, \LL)$ and $R_{\text{denoising}}(\xi^-, \LL)$ are oracle risks that depend upon (i) the ideal denoising basis for past events $e^-$, (ii) the coordinates $i$ where $(e^-_\BB)_i$ is large, and (iii) the true future statistic $\xi^+$.
\end{theorem}

Sec.~\ref{s:app:proofs} provides a proof. We can interpret Theorem~\ref{thm:main} as separating the structure ($\hat \xi^-$) and motion ($\hat A$) in the scene. In Eqn.~\ref{eq:objective}, the first term measures the discrepancy between future events $e^+$ and hard thresholded past events $\hat \xi^-$  transformed by the dynamics operator $A$. The second term encourages sparsity. This theorem suggests a broad principle for learning event representations. If we identify a basis in which past events are sparse, and use the denoised past events to predict future events, then we can recover the true future statistic $\xi^+$. Doing so incurs a risk that is off from the oracle risk by a constant factor $c_1$, and the ideal denoising risk by a logarithmic factor $c_2 \L_n \sim \log n$.

\begin{figure}[t]
    \centering
    \includegraphics[width=0.75\linewidth]{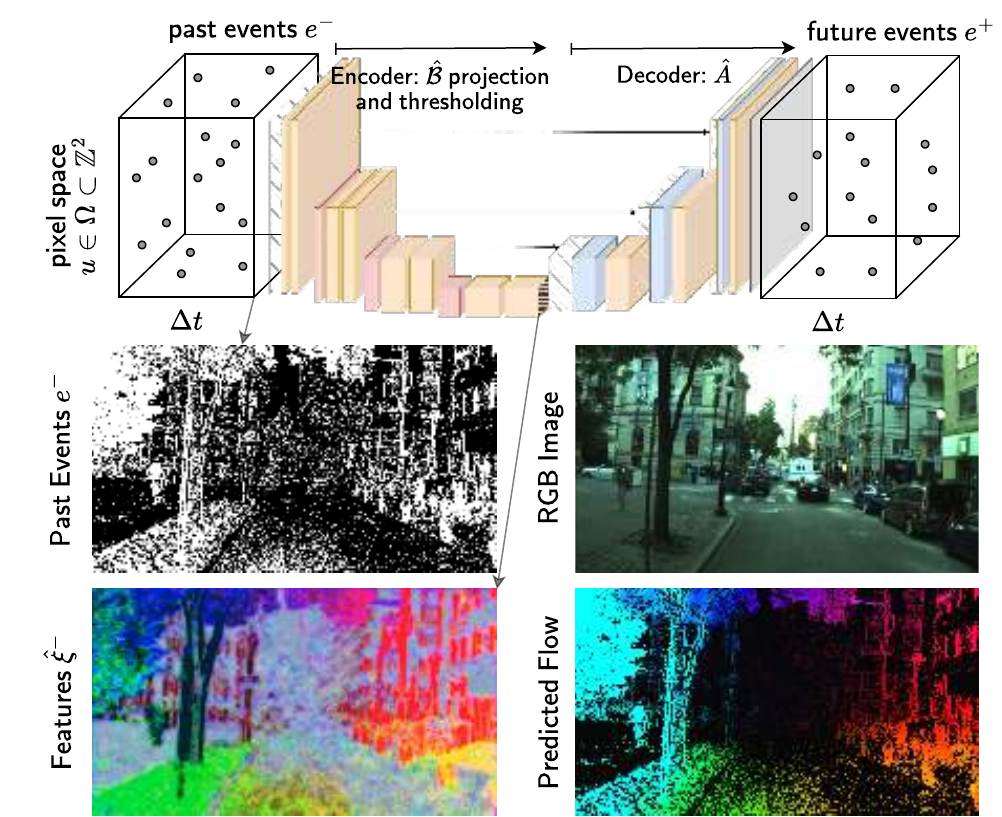}
    \caption{
    \small
    \textbf{A simple instantiation of Theorem~\ref{thm:main}.}
    A U-Net-based architecture \cite{ronneberger2015u} learns to predict future events $e^+$ from past events $e^-$ (top left image). Features of this network (bottom left) can be used to compute optical flow (bottom right) using techniques developed in Sec.~\ref{s:flow}.
    }
    \label{fig:unet}
\end{figure}

\begin{newenv}
A neural architecture that takes, say, a voxel grid of past events as input and predicts the voxel grid of future events would suffice to instantiate this theorem.
For example, the encoder in the U-Net in Fig.~\ref{fig:unet} implements the basis $\hat \BB$ and projects past events $e^-$ to obtain the representation $\hat \xi^-$.
The decoder implements the dynamics operator $\hat A$ and uses $\hat \xi^-$ to predict future events.
We train this architecture by modifying Eqn.~\ref{eq:objective} to have (i) a weighted mean-squared error to address event imbalance, and (ii) an $\ell_1$ norm in place of $\ell_0$ to make the optimization tractable.
When features of the bottleneck layer are used to predict optical flow using an unsupervised method described in Sec.~\ref{s:flow}, it obtains an average end-point error (AEE) of 7.71 pixels and a mean angular error (MAE) of 17.63$\deg$ on M3ED HD event data.
With this architecture, we can perform inference at 50 Hz.
In the next section, we develop two key ideas that improve upon this basic architecture to obtain $\ff$.
An $\ff$-based optical flow predictor has an AEE of 5.35 and an MAE of 12.58$\deg$. It can predict at 75 Hz on HD event data and requires about 10$\times$ less time to train.
\end{newenv}

\subsection{A fast neural architecture for learning representations of event data}
\label{s:instantiation}

\new{
An HD event camera gives $\sim$850 million voxels/sec at a time-discretization of 1 ms. About 95\% of these voxels are zero in typical environments.
The U-Net based instantiation of Theorem~\ref{thm:main} in Fig.~\ref{fig:unet} is slow because it does not exploit the sparsity of event data.
}
A good architecture for events should:
(i) not process empty voxels;
(ii) be invariant to the order in which events occur within a spatiotemporal region; and
(iii) be able to ingest a variable number of events.
These desiderata suggest that we need a set-based architecture for representing events.

Consider timestamps $i^-(u) = \{(s: s \in [t-\D t, t), e(s, u) = 1\}$ of events at a pixel $u \in \Om$ within the past $\D t$ time intervals.
The cardinality $\abr{i^-(u)}$ is different at different pixels $u$.
Let the neighborhood of the pixel $u$ be $\Om_u \subset \integers^2$.
A representation $\hat \xi(t, u)$ of the set $i^-(u)$ is invariant to permutations of its elements if and only if it can be decomposed as
\beq{
    \aed{
    \reals^n \ni \bar \varphi(t,u) &= \sum_{s \in i^-(u)} \varphi(s, u)\\
    \reals^p \ni \hat \xi(t, u) &= \rho\ \sbr{\bar \varphi(t, \cdot)}(u)
    }
    \label{eq:ff}
}
for appropriately chosen functions $\varphi: \integers_+ \times \Om \to \reals^n$ and $\rho: \Om_u \times \reals^n \to \Om_u \times \reals^p$ \cite{zaheer2017deep}. We next discuss how we choose these functions for event data.

The key is to think of $\varphi(s, u)$ as the ``feature'' of an individual event. This feature is the projection of coordinates $(s, u) \in \integers_+ \times \Om$ onto a learned basis. In view of our downstream tasks, we pick a basis with multi-scale spatiotemporal features as follows.
See the architecture in Fig.~\ref{fig:architecture}. It maintains $L$ resolution levels and assigns indices to each grid-point by hashing its coordinates. Given coordinates of the event $(s, u)$, the function $\varphi$ interpolates $F$-dimensional features corresponding to grid-points enclosing $(s, u)$ at each scale to give $n = L F$ dimensional feature $\varphi(s, u)$. Such bases are popular in the literature on neural rendering fields for encoding 3D space. Our implementation is identical to that of \cite{mueller2022instant}, except that one of our coordinates is time.
A hash encoder provides fast training and inference because
(i) forward passes using lookup operations are cheaper compared to, say, a basis projection implemented via a multi-layer perceptron (MLP), and
(ii) backward passes are also cheaper compared to those of an MLP because only the features that were accessed in the forward pass are updated.
We pick $\rho$ to be a convolutional network that smooths $\sum_{s \in i^-(u)} \varphi(s, u)$ at each pixel in the spatial domain.
This performs local marginalization of translation nuisances \cite{brunaInvariantScatteringConvolution2013}.

\begin{figure}[!t]
    \centering
    \includegraphics[width=\linewidth]{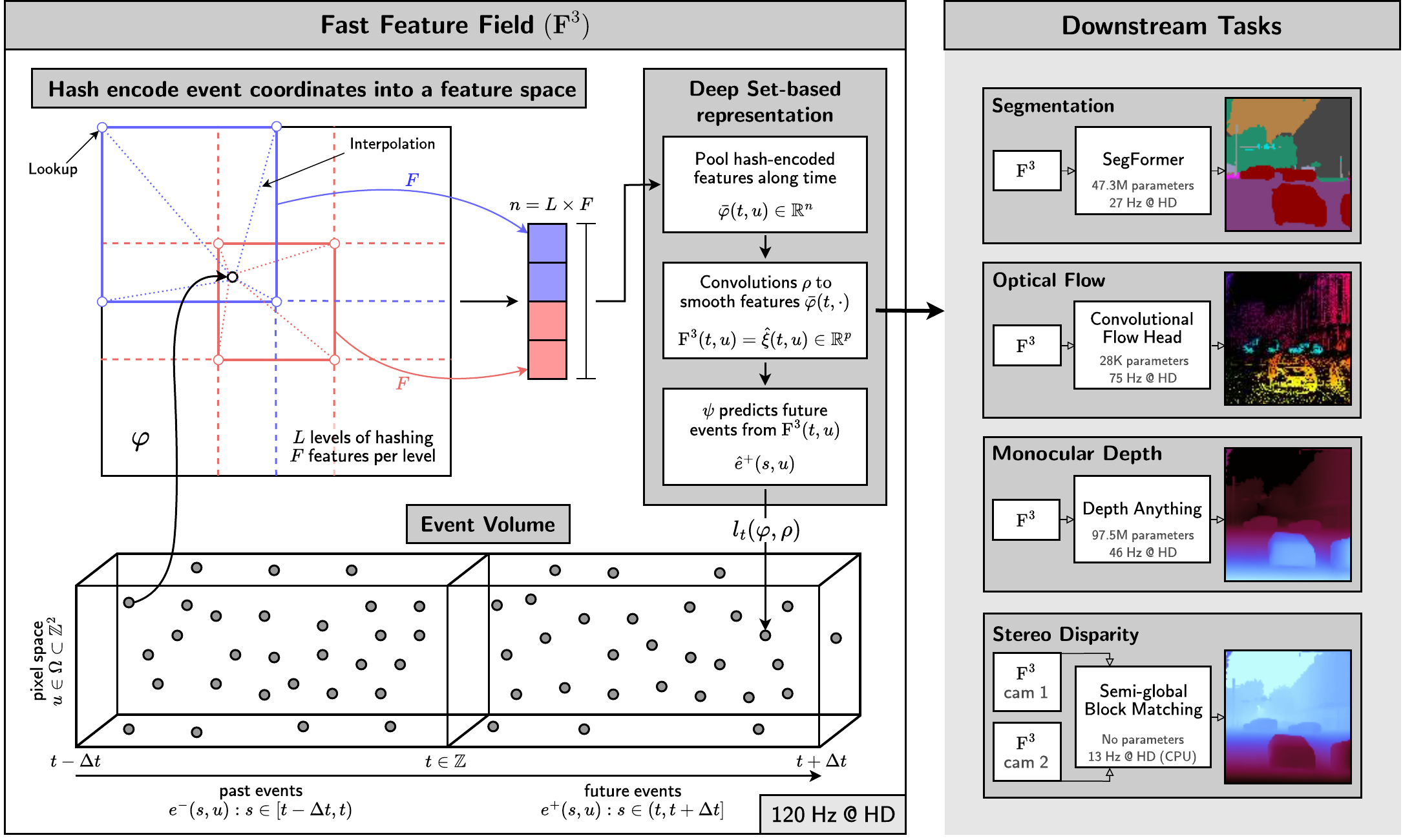}
    \caption{
    \small
    \textbf{An overview of the neural architecture for Fast Feature Field $\ff$.}
    Time and pixel coordinates of past events $e^-(s, u)$ between times $s \in [t-\D t, t)$ and pixels $u \in \Om$ are encoded using a hash encoder into a feature space, pooled across time, and smoothed in space to obtain the feature $\ff(t, u)$ at a pixel $u$ and time $t$. $\ff$ can be computed very quickly because the hash encoder does not encode coordinates with no events. $\ff$ is trained to predict future events $e^+(s, u)$ using a linear layer ($\psi$) using a class-weighted focal loss. $\ff$ essentially encodes the event volume into a multi-channel image. Supervised tasks such as semantic segmentation and monocular depth estimation can be easily learned using $\ff$. Unsupervised tasks such as optical flow or stereo disparity prediction can also be performed by matching these features across time or stereo cameras, respectively.
    }
    \label{fig:architecture}
\end{figure}
\paragraph{Fast feature field} is defined as
\[
     \reals^p \ni \ff(t, u) \equiv \hat \xi(t, u).
\]
In essence, the trained multi-resolution hash encoder is the basis $\hat \BB$ in~Eqn.~\ref{eq:objective}.
We therefore consider $\ff(t, u)$ to be the denoised event data.
For all experiments, unless otherwise noted, we choose $\D t = 20$ ms, an encoder $\varphi$ with $L=4$ levels, a hash table of size $2^{19}$, $F = 2$ dimensional features per entry, and a receptive field of $\Om_u = 37 \times 37$ with $p=32$ output channels for the convolutional network $\rho$ in Eqn.~\ref{eq:ff}.
Note that this architecture can also handle events with real-valued timestamps $t \in \reals_+$.

Fast feature field $\ff(t,\cdot)$ in Eqn.~\ref{eq:ff} is a $p$-channel ``image'' that represents information in the event stream across the time interval $[t-\D t, t)$ and spatial neighborhood $\Om \ni u$.
As baselines, we can consider two naive event representations,
\beq{
    \aed{
    \text{Event voxel grid: } \reals \ni \vv(t, s, u) &= e(s, u) \quad \text{for } s \in [t-\D t, t), \text{ and}\\
    \text{Event frames: } \reals \ni \ii(t, r, u) &= \ind{\abr{i^-_r(u)} \geq 1} \text{ for } r \in \{+1, -1\},\\
    }
    \label{eq:voxels_frames}
}
here $i^-_r(u)  = \{(s: s \in [t-\D t, t), e(s, u) = r\}$ for polarity $r = \{+1, -1\}$. Event voxel grid represents past events within a time window $[t-\D t, t)$ in the cardinal basis, similar to \cite{zhu2019unsupervised}. Event frames record the presence of an event at a pixel $u$ within a time window \cite{maqueda2018event}. $\vv$ retains temporal information but ignores polarity, while $\ii$ retains polarity but ignores time. This is why we expect both baselines to be useful for downstream tasks.
\new{We also consider two velocity-invariant representations, TOS \cite{TOS} and AAE \cite{ESVO2}. They were designed for event visual odometry to be robust to different motion speeds. We expect these representations to be effective for tasks where structure is more important than motion, e.g., segmentation and depth estimation.}

\paragraph{Training objectives to build robustness to noise and event rates}
We next elaborate upon the training objective for $\ff$. We need a more robust objective than Eqn.~\ref{eq:objective} to mitigate noise and class imbalance while predicting future events.
For events with coordinates in $\cup_{u \in \Om} i^+(u)$ in the interval $[t, t+\D t]$, we use the focal loss \cite{focallosscalibration} given by
\beq{
    \ell_t(\varphi, \rho, \psi)
    = - \sum_{s \in [t, t + \D t]} \sum_{u \in \Om}
    \a e (1 - \hat e)^\g \log \hat e
    + (1-\a) (1 - e)  {\hat e}^\g \log (1 - \hat e),
    \label{eq:objective_detail}
}
where
\[
e \equiv e(s,u),\ \hat e \equiv \psi(s, u;\ \ff(t,u)),
\]
and $\a = \abr{\cup_{u \in \Om} i^+(u)}/(\D t \abr{\Om})$ is the fraction of voxels containing events. The dynamics $\hat A$ in Theorem~\ref{thm:main} and Eqn.~\ref{eq:objective} is represented by $\psi: \integers_+ \times \Om \times \reals^p \to [0,1]$. It predicts a future event at $(s, u)$ using the representation of the past events $\ff(t, u)$. Given events in the interval $[0, N \D t]$, the objective for training $\ff$
\[
    \textstyle \ell(\varphi, \rho, \psi) = \sum_{i=1}^{N-1} \ell_{i \D t}(\varphi, \rho, \psi)
\]
for a hyper-parameter $\D t$. The regularization term $\L_n\ \norm{\hat \xi_\BB^-}_0$ in Eqn.~\ref{eq:objective} is replaced with weight-decay. We use a fully-connected linear layer for the dynamics $\psi$ in Eqn.~\ref{eq:objective_detail} for all experiments.
Sec.~\ref{s:app:focal_loss} elaborates upon the focal loss. We show that when the scene is described by non-overlapping surfaces, it is necessary to reweigh the loss using $\a$ to prevent trivial features, e.g., which predict $\hat e(s,u) = 0$ for all $s, u$. We also show that minimizing the focal loss with a non-zero $\g$ corresponds to entropy regularization of $\hat e(s,u)$. This is important to prevent overfitting to noisy event data.

\subsection{Supervised semantic segmentation}
\label{s:segmentation}

\begin{figure}[!ht]
    \centering
    \begin{subfigure}[b]{0.37\linewidth}
    \centering
    \begin{subfigure}[b]{\linewidth}
        \centering
        \begin{adjustbox}{width=\linewidth}
        \renewcommand{\arraystretch}{1.1}
        \begin{tabular}{l r r r}
            \toprule
            \textbf{Train and} & \textbf{Latency} & \textbf{mIoU} & \textbf{Accuracy} \\
            \textbf{Test on DSEC} & \textbf{(ms)} & \textbf{(\%)} & \textbf{(\%)}\\
            \midrule
            EV-SegNet \cite{alonso2019EvSegNet} && 51.76 & 88.61 \\
            ESS \cite{essdaniel} & 29 & 51.57 & 89.25 \\
            $\ii$  & 14 &  48.78 & 89.90 \\
            $\vv$  & 14& 49.95 & 90.39 \\
            $\ff$  & 14 & 55.41 & 91.91 \\
            $\ff (\D t = 50)$ & 14 & 55.95 & 92.17 \\
            \bottomrule\\
            \toprule
            \textbf{Train on M3ED,} & \textbf{Latency} & \textbf{mIoU} & \textbf{Accuracy} \\
            \textbf{Test on DSEC}  & \textbf{(ms)} & \textbf{(\%)} & \textbf{(\%)}\\
            \midrule
            ESS \cite{essdaniel} & 29 & 45.38 & 84.17 \\
            $\ii$ & 14 & 39.23 & 84.51 \\
            $\vv$ & 14 & 41.26 & 85.47 \\
            $\ff$ & 14 & 49.87 & 88.82 \\
            \bottomrule
        \end{tabular}
        \end{adjustbox}
        \caption{DSEC (all)}
        \label{fig:seg:dsec}
    \end{subfigure}
    
    \vspace*{1ex}
    \begin{subfigure}[b]{\linewidth}
        \centering
        \begin{adjustbox}{width=\linewidth}
        \renewcommand{\arraystretch}{1.1}
        \begin{tabular}{l r r r}
            \toprule
            \textbf{Train and} & \textbf{Latency} & \textbf{mIoU} & \textbf{Accuracy} \\
            \textbf{Test on M3ED}  & \textbf{(ms)} & \textbf{(\%)} & \textbf{(\%)}\\
            \midrule
            \new{ESS} (our impl.) & \new{52} & \new{35.10} & \new{75.11} \\
            \new{TOS} (our impl.) \cite{TOS}       & \new{36} & \new{56.12} & \new{88.90} \\
            \new{AAE} (our impl.) \cite{ESVO2}     & \new{32} & \new{48.23} & \new{83.85} \\
            $\ii$ & 29 & 51.71 & 86.60 \\
            $\vv$ & 30 & 54.92 & 87.89 \\
            $\ff$ & 38 & 66.28 & 92.99 \\
            \bottomrule
        \end{tabular}
        \end{adjustbox}
        \caption{M3ED (Urban driving)}
        \label{fig:seg:m3ed}
    \end{subfigure}
    \end{subfigure}
    \hfill
    \begin{subfigure}[b]{0.6\linewidth}
    \centering
    \begin{subfigure}[b]{\linewidth}
    \centering
    \includegraphics[width=\linewidth]{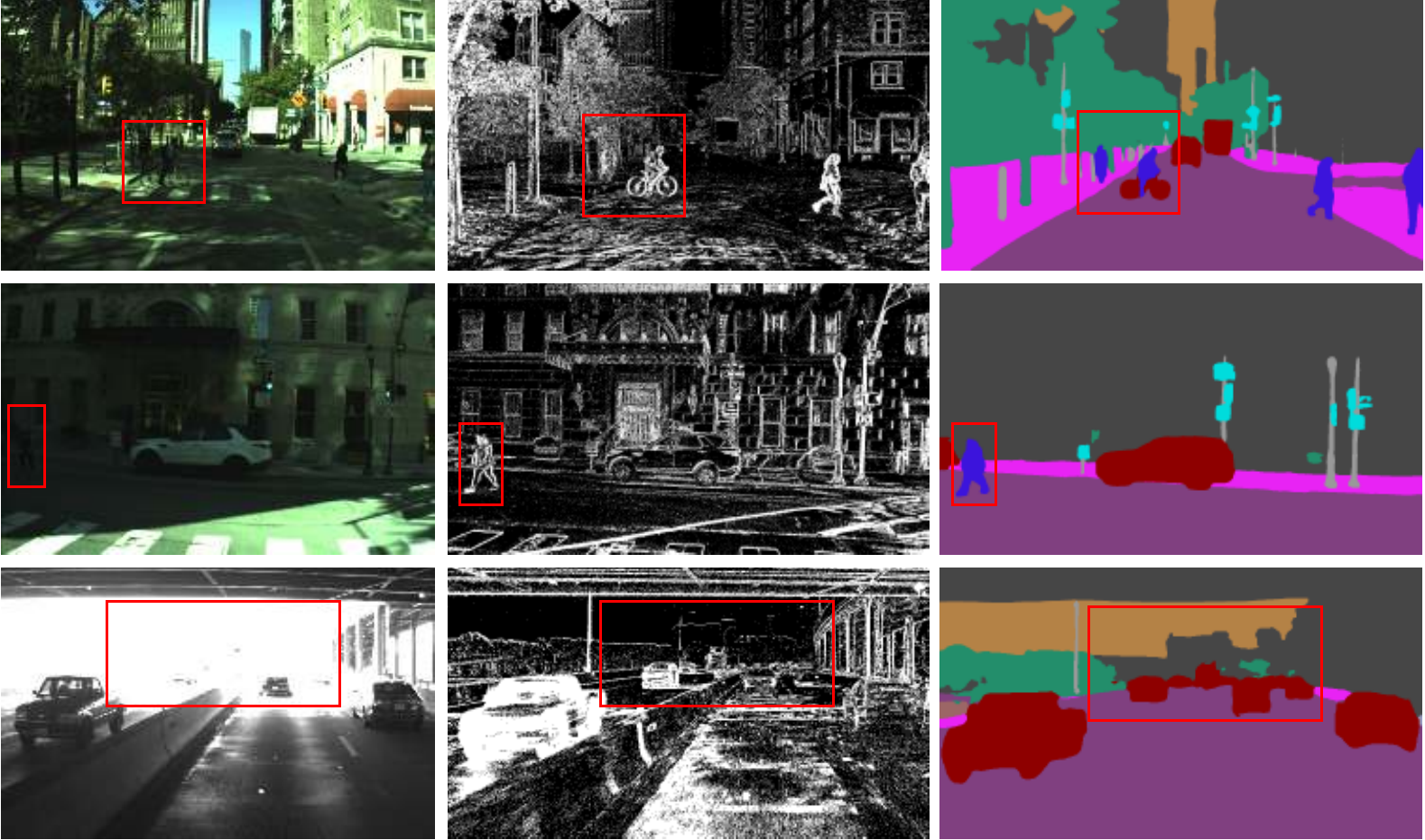}
    \caption{Qualitative results on urban driving data in M3ED}
    \label{fig:seg:qualitative}
    \end{subfigure}
    
    \vspace*{1ex}
    \begin{subfigure}[b]{0.75\linewidth}
    \centering
    \includegraphics[width=\linewidth]{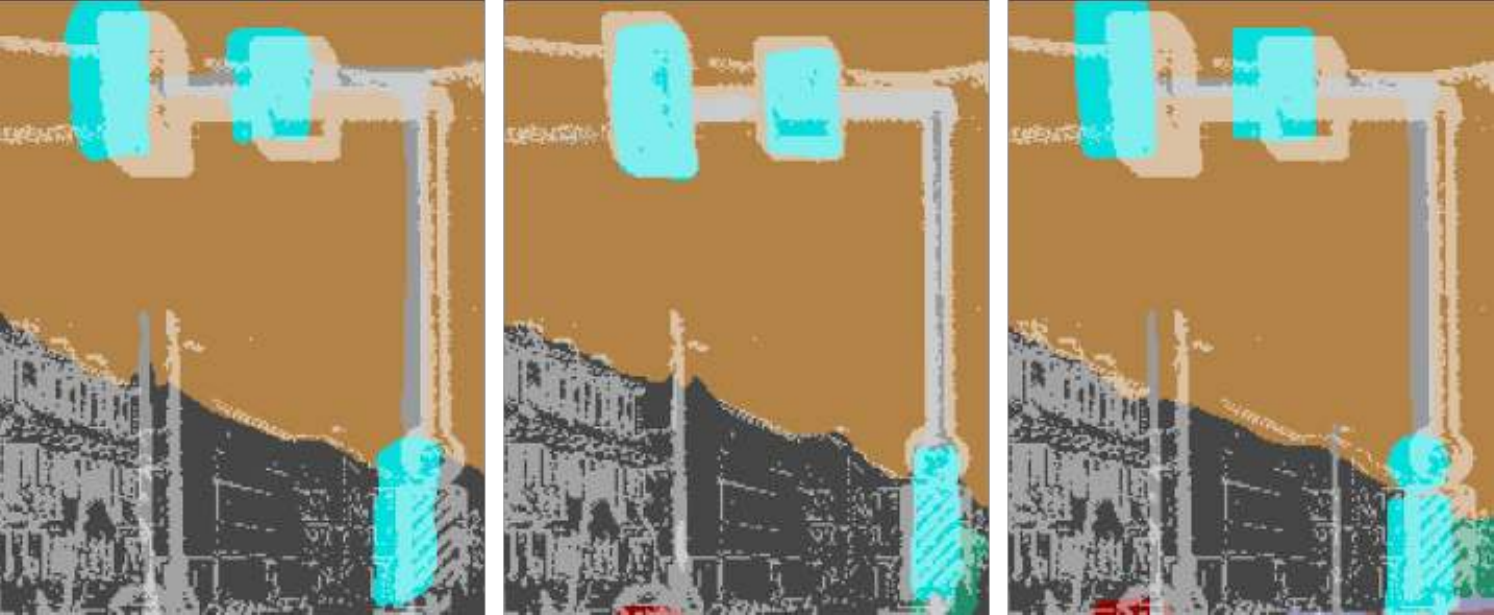}
    \caption{Train and test on DSEC (left), train on M3ED and test on DSEC (middle), and ground-truth on DSEC (right)}
    \label{fig:seg:misalignment}
    \end{subfigure}
    \end{subfigure}
    \caption{
    \small
    \textbf{Supervised semantic segmentation task for driving data in daytime settings from DSEC and M3ED.}
    \textbf{(a)} $\ff$-based semantic segmentation compares favorably to existing approaches on DSEC in mean intersection-over-union (mIOU) and accuracy. An $\ff$-based network trained on M3ED and tested on DSEC performs worse (bottom table), but this is due to temporal misalignment of ground-truth masks in DSEC (Fig.~\ref{fig:seg:misalignment}).
    \textbf{(b)} An $\ff$-based network trained and tested on M3ED daylight driving sequences achieves a high mIOU.
    \textbf{(c)} Objects of different sizes (cars, pedestrians, sidewalks, traffic lights) can be accurately segmented using event data (middle column) even if they lie in shadows or saturated regions of the RGB image (left column).
    \textbf{(d)} In DSEC, ground-truth segmentation masks on RGB images are misaligned with events. An $\ff$-based network trained on M3ED (where there is no misalignment) predicts quite accurately (middle image). An $\ff$-based network trained on DSEC seems better according to metrics in Fig.~\ref{fig:seg:dsec} (top) but produces rather incorrect predictions (left image).
    }
    \label{fig:segmentation}
    \end{figure}

We first show the plug-and-play nature of pre-trained $\ff$ using semantic segmentation tasks. Consider a SegFormer B3 architecture \cite{xie2021segformer} that is pretrained on the Cityscapes dataset \cite{Cordts2016Cityscapes}. This network takes RGB images as input. But it is easy to adapt it to take $\ff$ with $p$ channels as input by tiling and copying the weights of the first convolutional layer.
This same procedure can be used for baselines $\vv$, $\ii$, \new{TOS or AAE} representations.
We fine-tune this modified network on ground-truth annotations using the cross-entropy loss to obtain an $\ff$-based semantic segmentation model.

Fig.~\ref{fig:seg:qualitative} shows a qualitative evaluation of our approach. Fig.~\ref{fig:seg:m3ed} and \ref{fig:seg:dsec} show the accuracy and mean intersection-over-union (mIOU) across 11 classes \new{(mentioned in Sec.~\ref{s:app:details})} on M3ED and DSEC, respectively.
The $\ff$-based network achieves higher mIOUs on M3ED and DSEC than existing approaches such as ESS \cite{essdaniel} and EV-SegNet \cite{alonso2019EvSegNet}.
\new{
ESS uses voxelgrid representations of events and a pretrained e2vid encoder \cite{Rebecq19pami} to build features which are decoded into segmentation masks by supervised learning.
We adapted ESS to M3ED, but it did not perform well. This is perhaps because the e2vid encoder does not generalize to M3ED (different resolution and sensor hardware).
EV-SegNet is a supervised method that uses an event histogram-like representation \cite{maqueda2018event}.
}
The ability of $\ff$ to denoise event data and represent motion is useful for segmentation. Baseline $\vv$ and $\ii$-based approaches, which do not have this ability, perform poorly---in some cases by more than 11\% in mIOU---even if they are as good as $\ff$ on training data.
In Fig.~\ref{fig:seg:m3ed}, \ref{fig:seg:dsec} $\vv$, which retains temporal information, is more effective than $\ii$.
\new{
TOS and AEE are velocity-invariant representations.
This makes TOS perform better than $\vv$ in Fig.~\ref{fig:seg:m3ed}.
But AAE does not work well because it misses objects in parts of the image with few events.
}
An $\ff$-based model trained on M3ED and evaluated on DSEC obtains high mIOU and accuracy Fig.~\ref{fig:seg:dsec} (bottom). This shows that $\ff$ can generalize to new datasets with different sensors.

M3ED uses InternImage \cite{wang2023internimage} to pseudo-label RGB images while DSEC uses an older method \cite{DBLP:hierarchicalsemseg}.
While pseudo-labeling, it is important to synchronize timestamps between RGB and event cameras. This temporal misalignment is noticeable in DSEC, see Fig.~\ref{fig:seg:misalignment}. Therefore, an $\ff$-based network trained on M3ED achieves a slightly lower mIOU/accuracy on DSEC, see Fig.~\ref{fig:seg:dsec} (bottom). When $\ff$ is trained on DSEC in Fig.~\ref{fig:seg:dsec} (top), the mIOU is much higher---by more than 6\%---but this could be spurious. This phenomenon is also seen for ESS \cite{essdaniel} in Fig.~\ref{fig:seg:dsec}.
\new{Our result on transfer from M3ED to DSEC is perhaps slightly diluted due to this misalignment issue.}

\subsection{Unsupervised optical flow estimation}
\label{s:flow}

\begin{figure}
\centering
\begin{subfigure}[b]{0.55\linewidth}
    \centering
    \begin{subfigure}[b]{\linewidth}
    \centering
    \includegraphics[width=\linewidth]{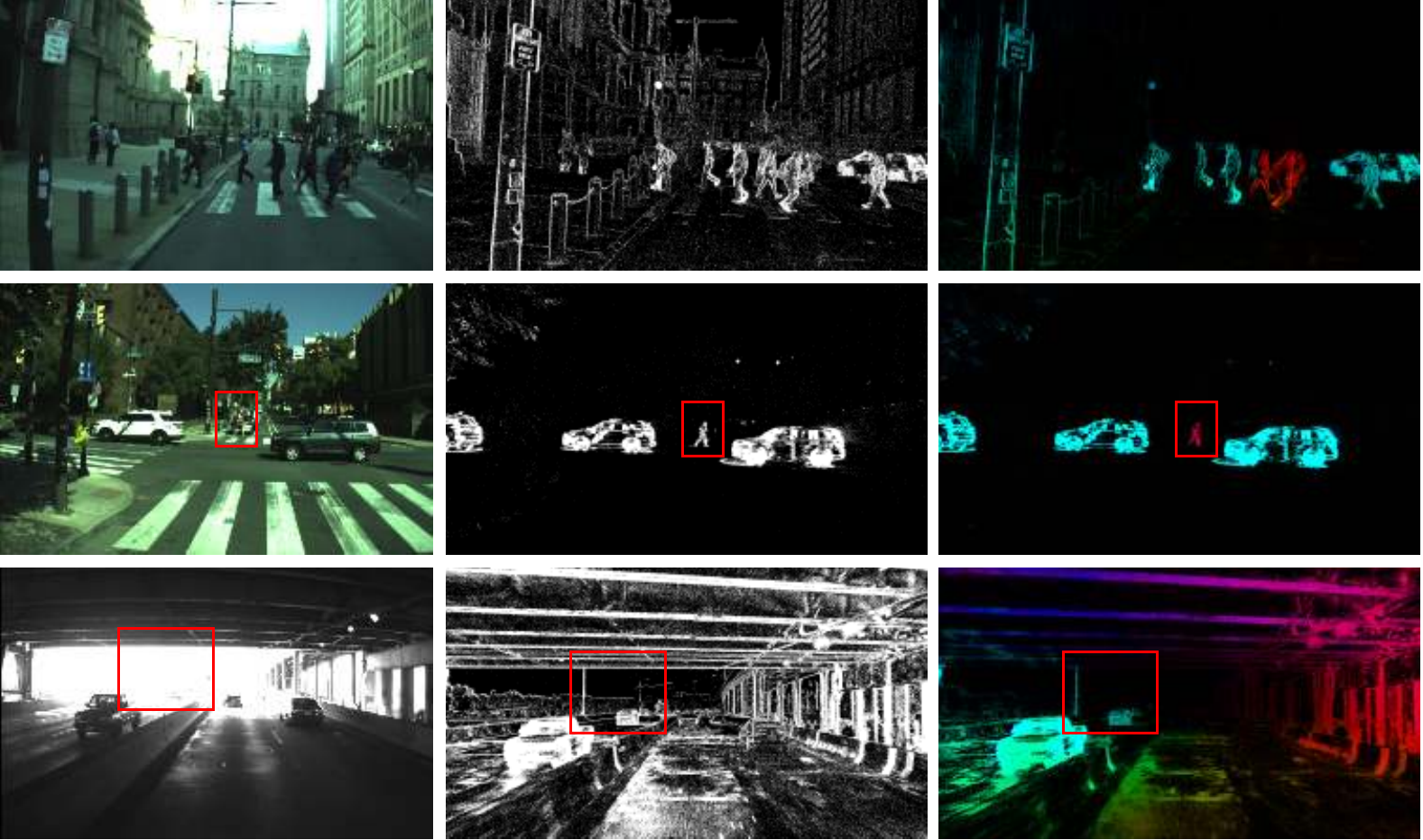}
    \caption{$\ff$-based optical flow on daytime urban driving data in M3ED}
    \label{fig:flow:qualitative}
    \end{subfigure}

    \vspace*{1ex}
    \begin{subfigure}[b]{\linewidth}
    \centering
    \includegraphics[width=\linewidth]{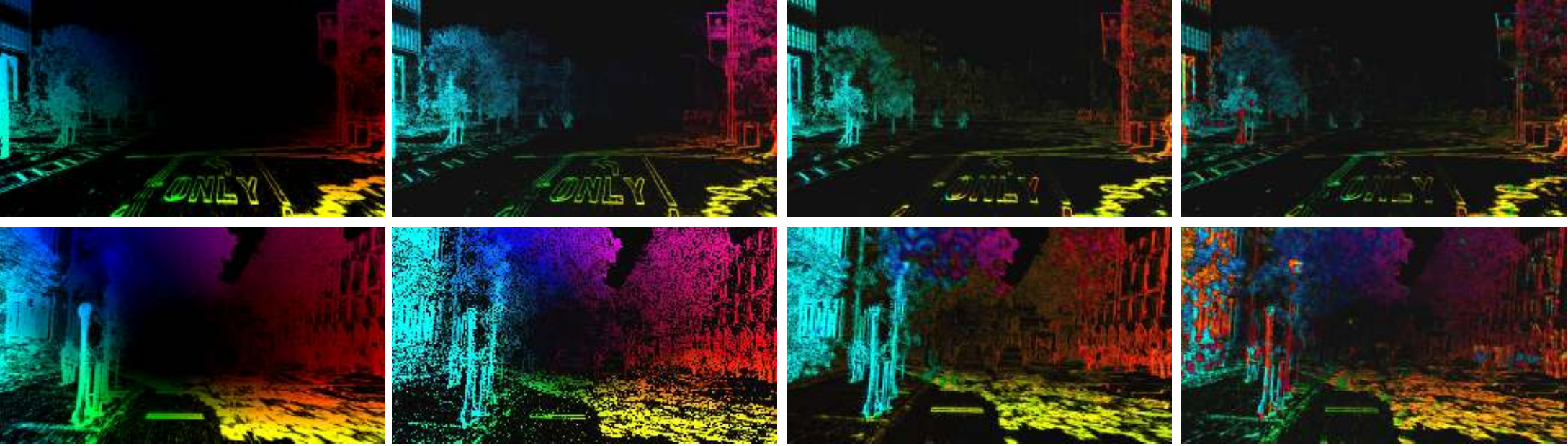}
    \caption{Optical flow from \new{E-RAFT}, $\ff$, $\vv$ and $\ii$ (left to right).}
    \label{fig:flow:ff_vv_ii}
    \end{subfigure}
\end{subfigure}%
\hfill
\begin{subfigure}[b]{0.44\linewidth}
\centering
    \begin{subfigure}[b]{\linewidth}
    \centering
    \begin{adjustbox}{width=\linewidth}
    \renewcommand{\arraystretch}{1.25}
    \begin{tabular}{l rrrrrrr}
        \toprule
        && \multicolumn{5}{c}{\textbf{Pixel error}}\\
        && \multicolumn{2}{c}{\textbf{45 Hz ground-truth}} && \multicolumn{2}{c}{\textbf{11.25 Hz ground-truth}} \\
        \cmidrule{3-4} \cmidrule{6-7}
        \textbf{Method} & \textbf{Latency} &\textbf{3PE} & \textbf{AEE} && \textbf{3PE} & \textbf{AEE} \\
        & \textbf{(ms)} &\textbf{(\%)} &&& \textbf{(\%)} &\\
        \midrule
        MultiCM \cite{shiba2024secrets}         && 0.21 & 0.28 && 5.68 & 1.05 \\
        EV-FlowNet \cite{Zhu2018UnsupervisedEVFlownet}    & 4.1 & 0.00 & 0.32 && 9.70  & 1.30 \\
        EV-MGRFlowNet \cite{evmgr_flownet}                && 0.02 & 0.28 && 6.22  & 1.10 \\
        $\ii (\D t = 50)$                 & 1.95 & 5.14 & 1.27 && 34.09 & 3.10 \\
        $\vv (\D t = 50)$                 & 1.88 & 1.90 & 0.97 && 34.91 & 3.09\\
        $\ff (\D t = 50)$                 & 2.5   & 0.23 & 0.36 && 10.83 & 1.50 \\
        E-RAFT* \cite{eraft}               & 48.75 & 1.70 & 0.24 && 1.12 & 0.72 \\
        \bottomrule
    \end{tabular}
    \end{adjustbox}
    \caption{MVSEC (outdoor\_day1)}
    \label{fig:flow:mvsec}
    \end{subfigure}\\
    \vspace*{1ex}
    \begin{subfigure}[b]{\linewidth}
    \centering
    \begin{adjustbox}{width=\linewidth}
    \renewcommand{\arraystretch}{1.25}
        \begin{tabular}{l rrrr}
        \toprule
        &&\multicolumn{2}{c}{\textbf{Pixel error}} & \textbf{Angular error}\\
        \textbf{Method} & \textbf{Latency} &\textbf{3PE} & \textbf{AEE} & \textbf{AAE} \\
        & \textbf{(ms)} &\textbf{(\%)} && \textbf{(deg)}\\
        \midrule
        \new{Guo et al.} \cite{Guo25iccv} & \new{15.12} & \new{11.24} & \new{1.78} & \new{6.44} \\
        MultiCM \cite{shiba2024secrets}
        & 10\textsuperscript{4} & 30.86 & 3.47 & 13.98 \\
        VSA-SM  \cite{vsa_sm}                        && 16.83 & 2.22 & 8.86  \\
        TamingCM \cite{tamingcm}                && 17.77 & 2.33 & 10.56 \\
        MotionPriorCM  \cite{motionpriorcm}  & 7.27 & 15.21 & 3.20 & 8.53  \\
        $\ii (\D t = 50)$                    & 2.89 & 73.19 & 9.59 & 43.90 \\
        $\vv (\D t = 50)$                    & 2.72 & 64.23 & 8.09 & 40.67 \\
        $\ff (\D t = 50)$                    & 3.92 & 27.93 & 2.92 & 9.15  \\
        E-RAFT* \cite{eraft}                  & 52.45 & 2.68 & 0.79 & 2.85 \\
        \bottomrule
        \end{tabular}%
    \end{adjustbox}
    \caption{DSEC (10 Hz ground-truth)}
    \label{fig:flow:dsec}
\end{subfigure}
\end{subfigure}\\
\vspace*{2ex}
\begin{subfigure}[b]{0.5\linewidth}
\centering
\begin{adjustbox}{width=\linewidth}
\renewcommand{\arraystretch}{1.25}
    \begin{tabular}{lrrrrr}
    \toprule
    \textbf{Train \& Test} &&&\multicolumn{2}{c}{\textbf{Pixel error}} & \textbf{Angular error}\\
    \textbf{dataset} &\textbf{Method} & \textbf{Latency} & \textbf{3PE} & \textbf{AEE} & \textbf{AAE}\\
    && \textbf{(ms)} &\textbf{(\%)} & & \textbf{(deg)}\\
    \midrule
    \multirow{7}{*}{\parbox{1.4cm}{Car \\Daytime}} &
    \new{E-RAFT} \cite{eraft} & \new{122} & \new{28.85} & \new{2.68} & \new{6.63} \\
    & \new{$\ff$-RAFT} & \new{52} & \new{18.12} & \new{2.16} & \new{5.87} \\
    & \new{TOS} \cite{TOS} & \new{15.03} & \new{92.54} & \new{20.96} & \new{79.84} \\
    & \new{AAE} \cite{ESVO2} & \new{11.09} & \new{95.79} & \new{21.27} & \new{74.09} \\
    &$\ii$  & 7.45 & 78.64 & 12.21 & 41.77\\
    &$\vv$  & 7.59 & 82.01 & 10.99 & 40.18\\
    &$\ff$  & 14 & 52.05 & 5.35  & 12.58\\
    \midrule
    \multirow{4}{*}{\parbox{1.4cm}{Spot}} &
    \new{E-RAFT} \cite{eraft} & \new{122} & \new{56.19} & \new{8.65} & \new{13.97} \\
    & \new{$\ff$-RAFT} & \new{52} & \new{55.10} & \new{8.13} & \new{12.57} \\
    & $\vv$   & 7.59 & 95.55 & 22.83 & 31.63\\
    & $\ff$    & 14 & 69.95 & 9.65  & 15.10\\
    \midrule
    \multirow{4}{*}{\parbox{1.4cm}{Falcon}} &
    \new{E-RAFT} \cite{eraft} & \new{122} & \new{56.53} & \new{6.80} & \new{16.90} \\
    & \new{$\ff$-RAFT} & \new{52} & \new{56.71} & \new{6.77} & \new{13.11} \\
    & $\vv$   & 7.59 & 91.02 & 11.83 & 31.59 \\
    & $\ff$    & 14 & 64.51 & 7.43  & 14.16 \\
    \bottomrule
\end{tabular}
\end{adjustbox}
\caption{M3ED (10 Hz ground-truth)}
\label{fig:flow:m3ed}
\end{subfigure}
\caption{
\small
\textbf{Optical flow estimation for different sensors, environments, and robotic platforms.}
\textbf{(a)}
Qualitatively, $\ff$-based optical flow (right column) is more
accurate despite
(i) pedestrians moving in different directions in the same region (top row),
(ii) fast and slow moving objects (cars vs. running pedestrians in middle row), and
(iii) saturated regions in the RGB data.
\textbf{(b-e)} A qualitative comparison corroborates the metrics in Fig.~\ref{fig:flow:mvsec}, \ref{fig:flow:dsec}, and \ref{fig:flow:m3ed} \new{(Lower values are better).} \new{E-RAFT and} $\ff$-based flow are spatially consistent, especially in terms of the flow direction and detailed structures. 3PE refers to the fraction of pixels with absolute flow error larger than 3. AEE refers to the average endpoint error (in pixels). AAE refers to the average angular error as suggested in \cite{baker2011database}.
}
\label{fig:flow}
\end{figure}

We next evaluate whether $\ff$ represents motion in the scene.
Optical flow in RGB data is typically estimated by enforcing brightness consistency across successive frames \cite{ma2004invitation}. This is difficult for event data due to sparsity and noise \cite{shiba2024secrets}. Existing workarounds use deblurring \cite{zhuetal,Shiba22sensors}, supervision \cite{eraft}, or RGB-based matching \cite{evflownet}. In this section, we develop an unsupervised procedure to estimate optical flow by enforcing brightness constancy on $p$-channel $\ff$ features.

Consider a neural network $\psi: \reals^p \times \Om \to \reals^2 \times \Om$ that takes $\ff(t, \cdot)$ as input to predict optical flow $\flow(t, u) \in \reals^2$ at time $t$ for all pixels $u \in \Om$. We first compute Gaussian pyramids $\ff_\s(t,u)$, $\ff_\s(t+\D t,u)$ and $\flow_\s(t, u)$ for multiple spatial scales $\s = 1, 2, \dots$ by smoothing and bilinear down-sampling by a factor of two at each scale. This provides global context and helps with sparse structures such as edges with large displacements.
We fit $\psi$ to enforce brightness consistency at each scale, along with some regularization, to minimize
\beq{
    \ell_{t,\Om}(\psi) = \f{1}{\abr{\Om}} \sum_\s 2^{2 \s} \sum_{u \in \Om} \ell_c
     \rbr{
     \f{1}{\sqrt Z}\
     \ff_\s(t,u) - \ff_\s(t+\D t, u + \flow_\s(t, u)))} + \l R(\D u),
     \label{eq:flow_objective}
}
where the denominator $Z \in \reals^n$ with the $i^{\text{th}}$ element $Z_i = \sum_u \rbr{\ff_\s(t,u)}_i^2 + \rbr{\ff_\s(t+\D t,u)}_i^2$ performs normalization. The square root and division are interpreted element-wise.
The Charbonnier loss $\ell_c(x) = n^{-1} \sum_c (x_c^2 + \e^2)^\b$ helps with outlier rejection \cite{sunetal_survey}.
Local approaches such as ours are susceptible to the aperture problem and cannot recover flow in regions without features. Regularization of the form
\[
     R(\flow) = \f{1}{\abr{\Om}} \sum_\s 2^{2 \s} \sum_{u \in \Om} \sum_{u' \in \Om_u} \ell_c \rbr{\flow_\s(t, u') - \flow_\s(t, u)}
\]
helps address this issue and smooths the predictions \cite{sun2010secrets}.

Fig.~\ref{fig:flow:qualitative} shows qualitative results of optical flow estimation.
Fig.~\ref{fig:flow:mvsec}, \ref{fig:flow:dsec} and \ref{fig:flow:m3ed} provide numerical comparisons on MVSEC, DSEC and M3ED, respectively.
\new{See Fig.~\ref{fig:flow:histogram} for the magnitude of ground-truth optical flow in these datasets.}
The $\ff$-based unsupervised approach is comparable to state-of-the-art methods, in terms of average end-point error (AEE).
$\ff$-based optical flow is
(i) very fast, e.g., 14 ms latency even on HD data in M3ED, and less than 4 ms latency on VGA and lower resolutions;
(ii) performs comparably to more specialized existing techniques across different metrics;
(iii) can be trained robustly across different platforms in M3ED, and
(iv) generalizes across illumination conditions without re-training (see Fig.~\ref{fig:robust:car_daytime_to_nighttime}).
\new{Supervised approaches, e.g., E-RAFT \cite{eraft}, are better than unsupervised ones such as ours, because they are non-causal and well-regularized even at regions with few events.}
$\vv$ and $\ii$ work quite poorly because, unlike $\ff$, they are not sufficiently denoised, and it is difficult to match them across time.
$\vv$ retains temporal information, so it is better than $\ii$, which retains only polarity.
\new{Velocity-invariant representations like TOS and AAE do not possess motion information and perform poorly on this task.}

Unsupervised methods use different
(i) representations of event data,
(ii) operations to warp features (linear in Eqn.~\ref{eq:flow_objective}), and
(iii) losses to compare warped features (Eqn.~\ref{eq:flow_objective}).
MultiCM \cite{shiba2024secrets} uses raw events but a contrast maximization loss.
MotionPriorCM \cite{motionpriorcm} uses a spline-based warp, which is useful for DSEC and M3ED where ground-truth flow is available only at 10 Hz and ground-truth flow can have large displacements \cite{shiba2024secrets}.
VSA-SM \cite{vsa_sm} uses multi-scale time-surfaces \cite{sironi2018hats,lagorce2016hots} and a piece-wise linear warp.
\new{Guo et al. \cite{Guo25iccv} estimates optical flow using a contrast maximization loss on events and a photometric loss on the reconstructed intensity images.}
Our approach primarily leverages the new $\ff$ representation.
The fact that we can predict optical flow faithfully with a small convolutional network with $\sim$28,000 parameters, a linear warp, and a brightness consistency loss, suggests that the representation $\ff$ encodes local motion faithfully.
For example, although MotionPriorCM has a smaller 3PE, the average endpoint error (AEE) of $\ff$ is better.
\new{$\ff$-based optical flow estimation is 9--16 times faster than E-RAFT (depending on the resolution), and 3.3 times faster than methods like Guo et al.---see Fig. 6.}

\begin{newenv}
\paragraph{Supervised optical flow using $\ff$-RAFT}
To demonstrate that $\ff$ features can be easily adapted to existing algorithms, we trained an $\ff$-based RAFT for supervised optical flow estimation as follows.%
\footnote{Training scripts of the original E-RAFT are not publicly available, so we implemented them ourselves.}
E-RAFT uses voxel grid representations of two successive 100 ms intervals of events to compute optical flow.
We instead use frozen $\ff$ features as inputs to the RAFT architecture \cite{teed2020raft} and train using the same supervised protocol as E-RAFT. This gives us $\ff$-RAFT, which uses $\ff$ features of two 20 ms windows that are 100 ms apart to estimate the optical flow over the 100 ms interval.
$\ff$-RAFT outperforms E-RAFT, unsupervised $\ff$ and other baselines for all three robotic platforms in M3ED, as shown in Fig.~\ref{fig:flow:m3ed}.
Optical flow is computed in RAFT via a correlation volume that matches features from two time windows. Our experiment shows that $\ff$ features are better suited for this formulation than voxel grids, perhaps because $\ff$ removes nuisance variability leading to better matching. The context encoder in RAFT can also use the motion information in $\ff$.
Our implementation of $\ff$-RAFT is $\sim$2.3$\times$ faster than E-RAFT because the latter builds voxel grids on the CPU.
\end{newenv}

\subsection{Monocular depth estimation}
\label{s:depth}

\begin{figure}
\centering

\begin{subfigure}[b]{0.5\linewidth}
\centering

\begin{subfigure}[b]{\linewidth}
\centering
\includegraphics[width=\linewidth]{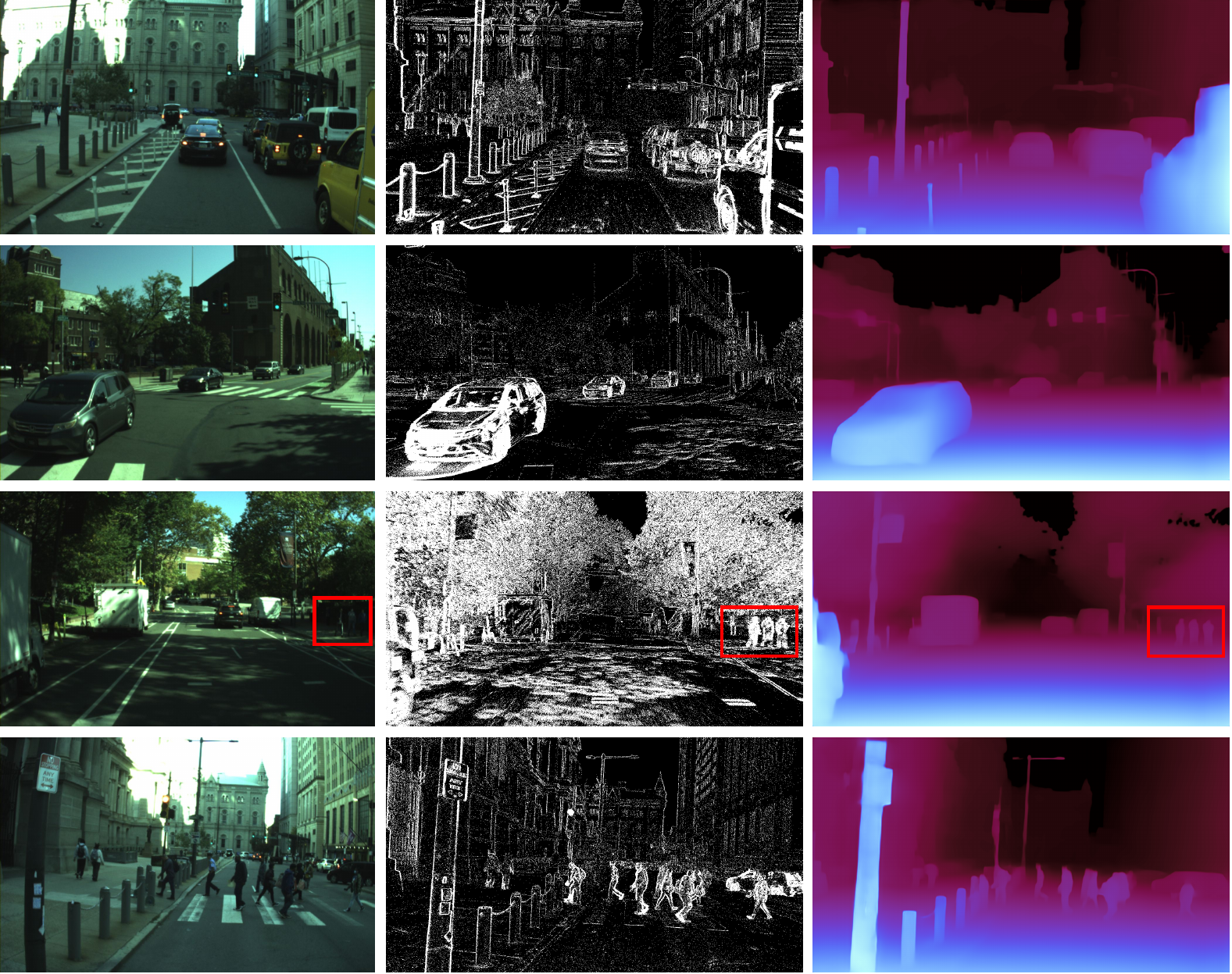}
\caption{Qualitative results on daytime urban driving in M3ED}
\label{fig:depth:qualitative}
\end{subfigure}\\
\vspace*{1ex}
\begin{subfigure}[b]{\linewidth}
\centering
\includegraphics[width=\linewidth]{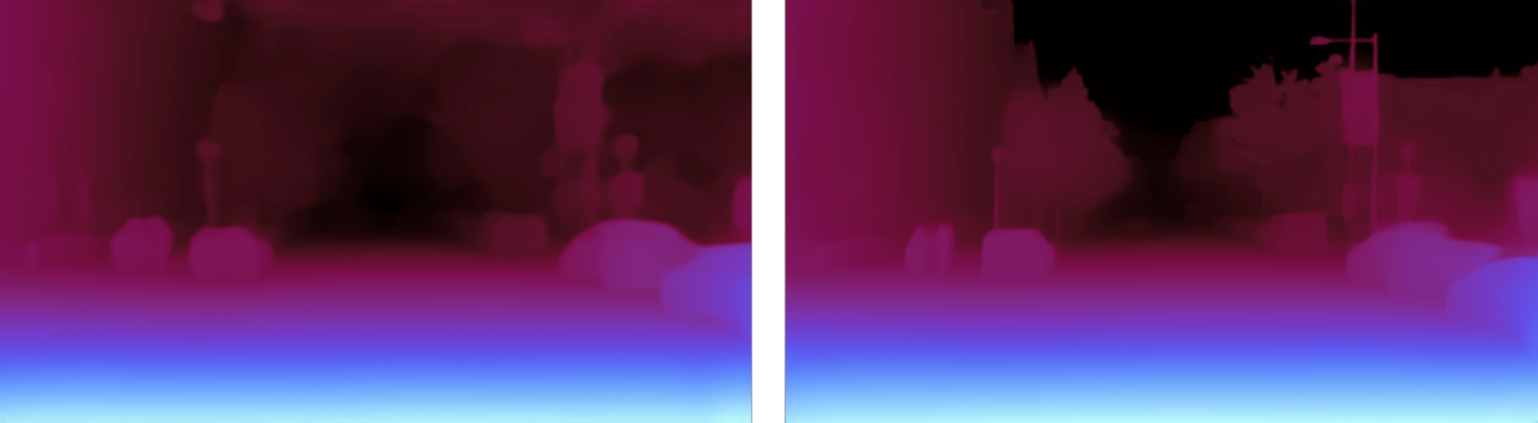}
\caption{Sharper object boundaries (right) with gradient-based regularization than without (left)}
\label{fig:depth:regularizer}
\end{subfigure}\\
\vspace*{1ex}
\begin{subfigure}[b]{\linewidth}
\centering
\includegraphics[width=\linewidth]{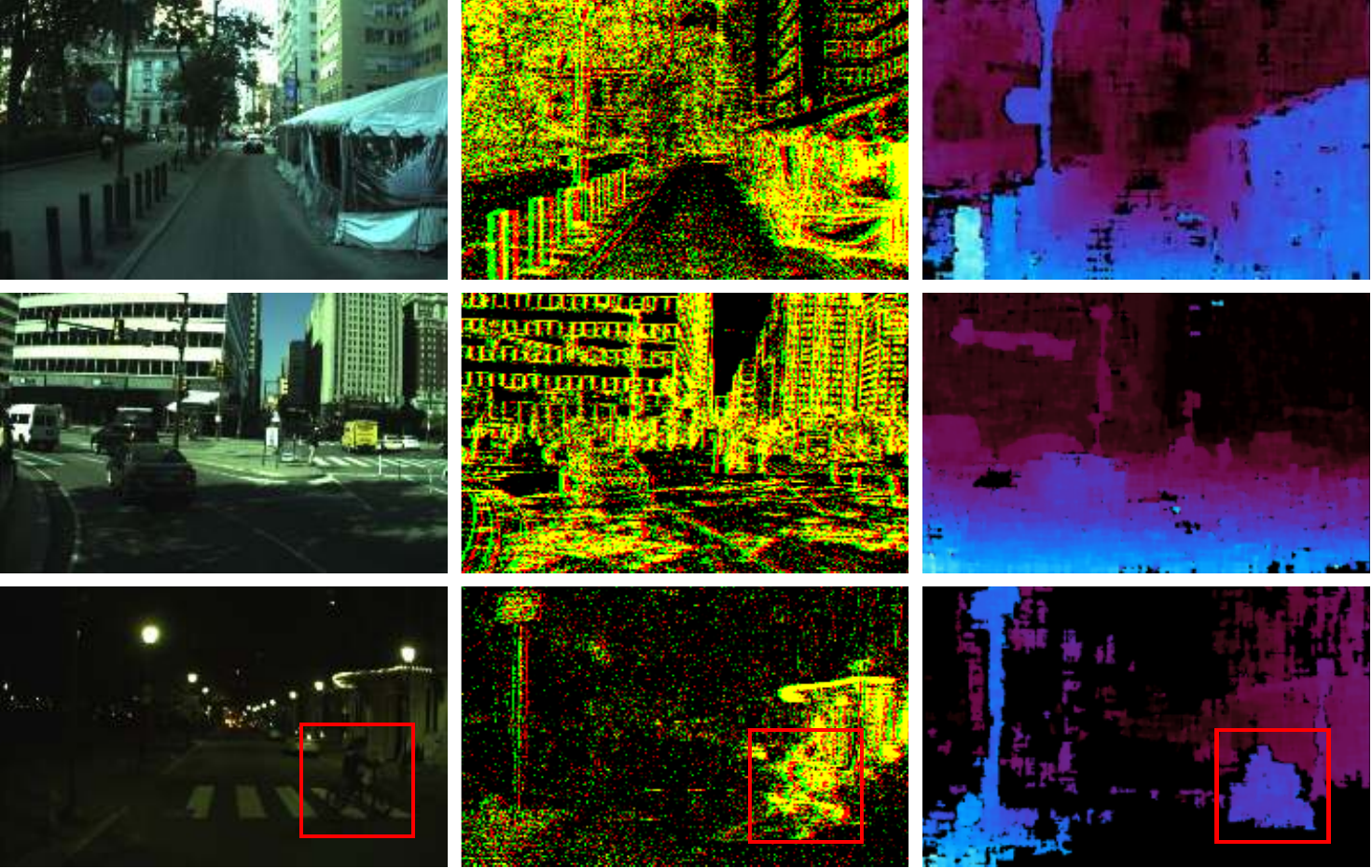}
\caption{Stereo disparity (unsupervised)}
\label{fig:depth:stereo}
\end{subfigure}
\end{subfigure}%
\hfill
\begin{subfigure}[b]{0.49\linewidth}
\centering

\begin{subfigure}[b]{0.85\linewidth}
    \centering
    \begin{adjustbox}{width=0.85\linewidth}
    \renewcommand{\arraystretch}{1.3}
    \begin{tabular}{l r r | rrr}
        \toprule
        \textbf{Method} & \textbf{Relative} & \textbf{RMSE} & \multicolumn{3} {c}{\textbf{Absolute error below}}\\
        & \textbf{$\ell_1$ error} & \textbf{(m)} & \textbf{10 m} & \textbf{20 m} & \textbf{30 m} \\
        \midrule
        RAM Net \cite{ramnet}
        & - & - & 1.39 & 2.17 & 2.76 \\
        EvT$^{+}$ \cite{EventTransformerplus}
        & - & - & 1.24 & 1.91 & 2.36 \\
        \midrule
        MDDE \cite{learningmonoculardensedepth}
        & 0.47 & 9.89 & 2.62 & 3.23 & 3.80 \\
        DTL \cite{wang2021dual}
        & 0.43 & - & 2.31 & 3.01 & 3.59 \\
        EReFormer \cite{Ereformer}
        & 0.29 & - & 1.41 & 2.21 & 2.79 \\
        SSL \cite{hagenaars2025device}
        & - & - & 2.44 & 3.12 & 3.79\\
        $\ii (\D t = 50)$
        & 0.34 & 7.27 & 1.71 & 2.62 & 3.11 \\
        $\vv (\D t = 50)$
        & 0.32 & 7.17 & 1.57 & 2.45 & 2.91 \\
        $\ff (\D t = 50)$
        & 0.31 & 6.29 & 1.51 & 2.43 & 2.76 \\
        \bottomrule
    \end{tabular}%
    \end{adjustbox}
    \caption{MVSEC (outdoor\_day1 and outdoor\_night1)}
    \label{fig:depth:mvsec}
\end{subfigure}\\
\vspace*{1ex}

\begin{subfigure}[b]{\linewidth}
    \centering
    \begin{adjustbox}{width=\linewidth}
    \renewcommand{\arraystretch}{1.25}
        \begin{tabular}{l r r rr | rrr}
        \toprule
        \textbf{Train \&} & \textbf{Method} & \textbf{Latency} & \textbf{Relative} & \textbf{RMSE} & \multicolumn{3}{c}{\textbf{Pixels (\%) with} $\boldsymbol \delta$ \textbf{below}}\\
        \textbf{Test} & & \textbf{(ms)} & \textbf{$\ell_1$ error} & \textbf{(m)} & $\bf 1.25$ & $\bf 1.25^2$ & $\bf 1.25^3$ \\
        \midrule
        \multirow{3}{*}{\parbox{1.3cm}{Car \\Daytime}} &
        \new{EReFormer} \cite{Ereformer} & \new{39.05} & \new{0.70} & \new{10.41} & \new{0.26} & \new{0.58} & \new{0.79} \\
        & \new{TOS} \cite{TOS} & \new{19.38} & \new{0.21} & \new{6.27} & \new{0.72} & \new{0.87} & \new{0.96} \\
        & \new{AAE} \cite{ESVO2} & \new{15.71} & \new{0.18} & \new{5.64} & \new{0.75} & \new{0.91} & \new{0.97} \\
        & $\ii$  & 12.94 & 0.26 & 7.84 & 0.69 & 0.84 & 0.94 \\
        & $\vv$  & 13.51 & 0.17 & 5.88 & 0.77 & 0.92 & 0.98 \\
        & $\ff$  & 21.66 & 0.09 & 3.15 & 0.91 & 0.99 & 1.00 \\
        & $\ff$ (low-latency)  & 14.36 & 0.10 & 3.18 & 0.92 & 0.98 & 0.99 \\
        \midrule
        \multirow{3}{*}{\parbox{1.3cm}{Spot}} &
        \new{EReFormer} \cite{Ereformer} & \new{39.05} & \new{1.19} & \new{11.62} & \new{0.30} & \new{0.46} & \new{0.59}\\
        & $\ii$    & 12.94 & 0.18 & 3.04 & 0.77 & 0.91 & 0.95\\
        & $\vv$    & 13.51 & 0.18 & 3.09 & 0.78 & 0.90 & 0.92\\
        & $\ff$    & 21.66 & 0.16 & 2.80 & 0.79 & 0.93 & 0.98\\
        \midrule
        \multirow{3}{*}{\parbox{1.3cm}{Falcon}} &
        \new{EReFormer} \cite{Ereformer} & \new{39.05} & \new{0.81} & \new{8.89} & \new{0.29} & \new{0.54} & \new{0.70} \\
        & $\ii$    & 12.94 & 0.21 & 4.17 & 0.68 & 0.87 & 0.95\\
        & $\vv$    & 13.51 & 0.19 & 3.53 & 0.75 & 0.93 & 0.98\\
        & $\ff$    & 21.66 & 0.17 & 3.44 & 0.72 & 0.93 & 0.99\\
        \bottomrule
    \end{tabular}%
    \end{adjustbox}
    \caption{M3ED}
    \label{fig:depth:m3ed}
    \end{subfigure}\\
    \vspace*{1ex}

    \begin{subfigure}[b]{\linewidth}
    \centering

    \begin{adjustbox}{width=\linewidth}
    \renewcommand{\arraystretch}{1.25}
        \begin{tabular}{l cr rrr}
        \toprule
        \textbf{Method} & \textbf{Gradient} & \textbf{1PE} & \textbf{2PE} & \textbf{MAE} & \textbf{RMSE} \\
        & \textbf{regularization} & \textbf{(\%)} & \textbf{(\%)}\\
        \midrule
        SSL \cite{hagenaars2025device}
        & \multirow{4}{*}{\cmark} & 84.92 & 70.47 & 4.95 & 6.27\\
        $\ii$  & & 65.25 & 40.94 & 2.47 & 3.59 \\
        $\vv$  & & 62.44 & 37.59 & 2.21 & 3.26 \\
        $\ff$ \footnotesize{($\tilde \disp_{\text{stage-1}}$ trained on M3ED)} & & 60.60 & 34.91 & 2.07 & 2.99 \\
        \midrule
        $\ii$  & \multirow{4}{*}{\xmark} & 64.33 & 40.48 & 2.52 & 3.71 \\
        $\vv$  & & 62.84 & 37.9 & 2.25 & 3.27 \\
        $\ff$ \footnotesize{($\tilde \disp_{\text{stage-1}}$ trained on M3ED)} & & 61.06 & 35.94 & 2.12 & 3.07 \\
        $\ff$ \footnotesize{\new{(directly trained on DSEC)}}  & & 59.38 & 34.28 & 2.07 & 3.03\\
        \bottomrule
        \end{tabular}
    \end{adjustbox}
    \caption{DSEC (disparity)}
    \label{fig:depth:dsec}
\end{subfigure}

\end{subfigure}
\caption{
\small
\textbf{Monocular metric depth estimation for different sensors, environments, and robotic platforms.}
\textbf{(a)} $\ff$ depth has sharp object boundaries at different distances and sizes, e.g., bollards (top row), traffic lights (second row), and pedestrians (bottom two rows). Predictions are accurate in far-field regions (sky), robust to spurious events (shadow of the car in middle row), and vary smoothly in spite of high texture (trees in third row).
Top three rows use the standard $\ff$ architecture while the bottom row uses a low-latency variant designed for edge inference (see Sec.~\ref{s:app:jetson}).
\textbf{(b)} Gradient regularization in Eqn.~\ref{eq:depth_regularizer}
results in sharp object boundaries.
\textbf{(c)} Stereo disparity (right column) from $\ff$ features in M3ED (overlayed upon each other in the middle column). Color indicates magnitude, with black indicating low-confidence regions where matching was unreliable. In the third row, the road sign and the bicycle are barely visible in the RGB image at night but yield confident disparity estimates.
\textbf{(d--f)}
In Fig.~\ref{fig:depth:mvsec} relative $\ell_1$ error is the average (over pixels) of $\abs{d/d^* - 1}$ for the predicted ($d$) and true depth ($d^*$) at a pixel. RMSE is the root mean-squared error in meters. Average absolute error $\abs{d - d^*}$ is reported for pixels with ground-truth depth below 10m, 20m, and 30 m. In Fig.~\ref{fig:depth:m3ed}, because the scenes are quite diverse, we report the fraction of pixels (\%) at which $\delta = \max(d/d^*, d^*/d)$ is below different thresholds.
For Fig.~\ref{fig:depth:dsec}, we evaluate the disparity; 1PE and 2PE stand for 1 or 2 pixel error, and MAE stands for the mean-absolute error.
}
\label{fig:depth}
\end{figure}

To demonstrate that $\ff$ retains fine-grained information about motion and structure in the scene, we develop a procedure to estimate metric monocular depth.
In driving scenarios we often observe camera trajectories with no changes in roll, pitch, or translation along the yaw axis over short time horizons. In such sequences, optical flow at a sufficiently large number of pixels can provide depth up to a global scale \cite{ma2004invitation}[Section 5.4.4].
Therefore, unlike RGB still images which lack motion information, monocular relative depth estimation from event data is a well-posed problem. Estimating metric monocular depth is ill-posed because the global scale depends on translational velocity.

We adapt a pre-trained network (Depth Anything V2 Base model \cite{depth_anything_v2} with 97.5 M parameters) by tiling and copying weights in the first layer to take $p$-channel $\ff$ as input instead of RGB images. Fine-tuning directly on LiDAR depth is non-trivial for two reasons: (i) the amount of LiDAR data is relatively small for many event datasets (see Fig.~\ref{fig:datasets}), and (ii) LiDAR depth is very sparse when re-projected onto high-resolution event camera frames.
A good strategy is therefore to (i) first fine-tune a $\ff$-based network to predict dense pseudo-labeled disparity from a pre-trained DepthAnything V2 Large model on RGB images, and then (ii) fine-tune to predict metric disparity from LiDAR data.%
\footnote{For low-resolution event cameras in MVSEC, LiDAR depth is sufficiently dense when projected on the pixel-space, and we forgo the first stage.}

Consider a network $\psi: \reals^p \times \Om \to \reals \times \Om$ that takes $\ff(t, \cdot)$ as input to predict disparity $\disp(t, u) \in \reals_+$ at time $t$ for all pixels $u \in \Om$. Let $\disp^*_{\text{pseudo}}(t, u)$ denote the pseudo-labeled disparity. Let $\disp^*_{\text{metric}}(t, u)$ denote metric disparity computed from LiDAR, which is only available on a subset $\Om' \subset \Om$ of the pixels. We work with normalized disparity $\tilde \disp(t,u)$ obtained by subtracting the median disparity from $\disp(t, u)$ and dividing by the average deviation from the median \cite{midas}. The objective used to fit $\psi$ in the first stage is
\beq{
    \ell_t(\psi) =
    \f{1}{\abr{\Om}} \sum_u \abr{\tilde \disp(t, u) - \tilde \disp^*_{\text{pseudo}}(t, u)}
    + \l R (\tilde \disp, \tilde \disp^*_{\text{pseudo}}).
    \label{eq:depth_objective_1}
}
The first term is the discrepancy between the prediction $\tilde \disp(t, u)$ and the pseudo-labeled disparity $\tilde \disp^*_{\text{pseudo}}(t,u)$. The regularizer
\beq{
    R(\disp, \disp') = \f{1}{\abr{\Om}} \sum_{\s} 2^{2 \s} \sum_u \norm{\nabla_u \disp_\s(t, u) - \nabla_u \disp'_\s(t, u)}_1.
    \label{eq:depth_regularizer}
}
encourages accurate object boundaries by matching the gradient of the predictions with the gradient of the pseudo-labeled disparity from RGB images.
Averaging across spatial scales $\s$ (computed by down-sampling by a factor of two, without smoothing) encourages the network to build global context.
Let $\disp_{\text{stage-1}}(t, u)$ be the predicted disparity after the first stage; this is potentially incorrect up to scale.
The objective minimized in the second stage is
\beq{
    \ell_t(\psi) =
       \f{1}{\abr{\Om'}}
    \sum_u \rbr{\log \f{\disp(t, u)}{\disp^*_{\text{metric}}(t,u)}}^2 - \f{1}{2 \abr{\Om'}^2} \rbr{\sum_u \log \f{\disp(t, u)}{\disp^*_{\text{metric}}(t,u)}}^2
    + \l R (\tilde \disp, \tilde \disp_{\text{stage-1}})
    \label{eq:depth_objective_2}
}
The first two terms are the scale-invariant ``SiLog'' loss between the predicted and metric disparity \cite{eigenetal}. The regularizer in the third term ensures that predicted object boundaries match those of $\tilde \disp_{\text{stage-1}}$.
We could have used $R (\tilde \disp, \tilde \disp_{\text{pseudo}})$ as the regularizer, but this leads to temporal aliasing because RGB images and LiDAR data often arrive at slightly different time instants.
We can compute $\disp_{\text{stage-1}}(t, u)$ at the exact instant of the LiDAR timestamp because $\ff$ can be queried for any $t$. In contrast, $\tilde \disp_{\text{pseudo}}$ can only be computed at the time instant of the RGB frame.

\new{
Fig.~\ref{fig:depth} shows that $\ff$-based metric depth estimates are quite accurate, often better than existing approaches such as EReFormer \cite{Ereformer} except for MVSEC.
Root mean square error (RMSE) of $\ff$-based depth estimates is about 3 m on M3ED, and about 6 m on MVSEC.
There are no existing results on depth estimation on M3ED.
We therefore retrained EReFormer on M3ED. EReFormer performs poorly on M3ED, perhaps due to sparsity and the small number of LiDAR data points. $\ff$ and other baseline methods alleviate this issue with the two-stage training procedure described above.
Training $\ff$ on DSEC and testing on DSEC (last row in Fig.~\ref{fig:depth:dsec}) does not significantly improve predictions compared to $\ff$ stage-1 trained on M3ED. This suggests that $\ff$ depth can transfer from M3ED to DSEC.
In Fig.~\ref{fig:depth:dsec}, $\ff$ depth is slightly worse with gradient regularization than without, even if actual predictions in Fig.~\ref{fig:depth:regularizer} are much more accurate. This is because depth estimates are evaluated only on pixels with LiDAR measurements in DSEC.
}
\new{
Even on MVSEC, $\ff$ performs better than approaches like MDDE or DTL, which use synthetic training data or RGB images.
It is only slightly worse than RAM Net \cite{ramnet} and EvT$^{+}$ \cite{EventTransformerplus} that use both RGB images and events, for inference, in daylight conditions.
In general, $\vv$ outperforms $\ii$. On M3ED, velocity-invariant representations, TOS and AAE, also outperform $\ii$, while AAE slightly outperforms $\vv$.
}
We also report results for a low-latency variant of $\ff$ designed to be memory bandwidth efficient for edge inference (see Sec.~\ref{s:app:jetson}). Monocular depth estimation using this variant is about 35\% faster while being comparable in evaluation metrics to the standard $\ff$ architecture on M3ED day time driving sequences.

\subsection{$\ff$-based approaches work robustly across robotic platforms, lighting conditions, dynamic vision sensors, and event rates}
\label{s:robustness}

\paragraph{Robotic platforms}
Fig.~\ref{fig:flow:m3ed} and \ref{fig:depth:m3ed} show that $\ff$ obtains good performance on optical flow and depth estimation tasks on Spot and Falcon.
Data from these robots is quite different than car driving sequences. Optical flow can be very large for periodic gaits in quadrupeds and attitude changes while flying.
Compared to fronto-parallel driving scenes, indoor scenes in some of this data contain objects of different sizes and types viewed from very different viewpoints.
\new{The training strategy and hyper-parameters of experiments in this section are the same as those described in Sec.~\ref{s:segmentation}, \ref{s:flow}, and \ref{s:depth}.}
Fig.~\ref{fig:robust:car_to_spot_and_falcon} shows that optical flow and depth using $\ff$ trained on daytime urban driving but tested on indoor data from Spot, and outdoor data from Falcon seem quite good, e.g., notice the depth for pipes, cars, and pedestrians, and consistent optical flow across the field of view.
\new{Quantitatively however, all methods including $\ff$ transfer poorly to such settings, see~Tab.~\ref{tab:flow:m3ed_car_daytime_to_other_robots} and \ref{tab:depth:m3ed_car_daytime_to_other_robots}.}
\new{Fig.~\ref{fig:evimo2} shows some more qualitative results of $\ff$-based optical flow on the EVIMO2 dataset \cite{burner2022evimo2eventcameradataset} with handheld camera motion and multiple independently moving objects.}

\begin{figure}
\centering
\begin{subfigure}[b]{0.45\linewidth}

    \begin{subfigure}[b]{\linewidth}
    \textbf{\small Spot}\\[0.5ex]
    \includegraphics[width=\linewidth]{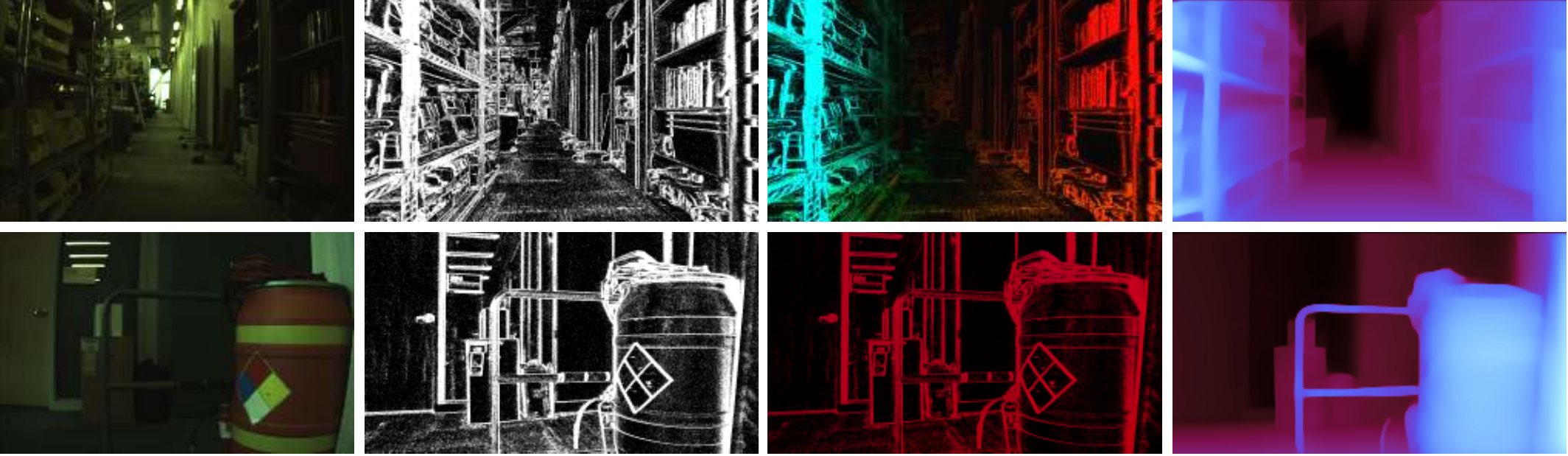}
    \textbf{\small Falcon}\\[0.5ex]
    \includegraphics[width=\linewidth]{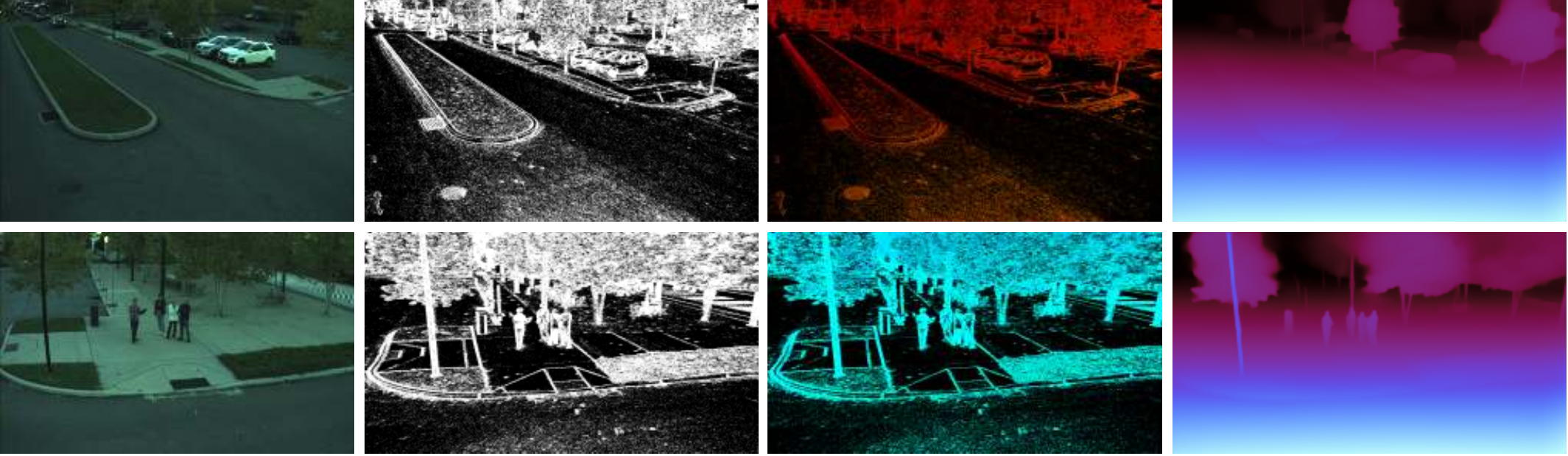}
    \caption{Generalization to robot platforms}
    \label{fig:robust:car_to_spot_and_falcon}
    \end{subfigure}

    \vspace*{1ex}
    \begin{subfigure}[b]{\linewidth}
    \centering
    \begin{adjustbox}{width=\linewidth}
    \renewcommand{\arraystretch}{1.2}
    \begin{tabular}{llrrr}
        \toprule
        && \multicolumn{2}{c}{\textbf{Pixel error}} & \textbf{Angular error}\\
        \textbf{Test Dataset} & \textbf{Method} & \textbf{3PE (\%)} & \textbf{AEE} & \textbf{AAE (deg)}\\
        \midrule
        \multirow{5}{*}{\parbox{1.4cm}{Car \\Nighttime}} &
        \new{E-RAFT} \cite{eraft} & \new{36.75} & \new{4.13} & \new{13.05} \\
        & \new{$\ff$-RAFT} & \new{19.22} & \new{2.35} & \new{8.52} \\
        & $\ii$    & 93.32 & 23.38  & 63.51\\
        & $\vv$    & 82.83 & 12.75  & 49.79\\
        & $\ff$    & 59.39 & 6.51   & 12.72\\
        \bottomrule
        \vspace*{0.25ex}
    \end{tabular}
    \end{adjustbox}
    \includegraphics[width=\linewidth]{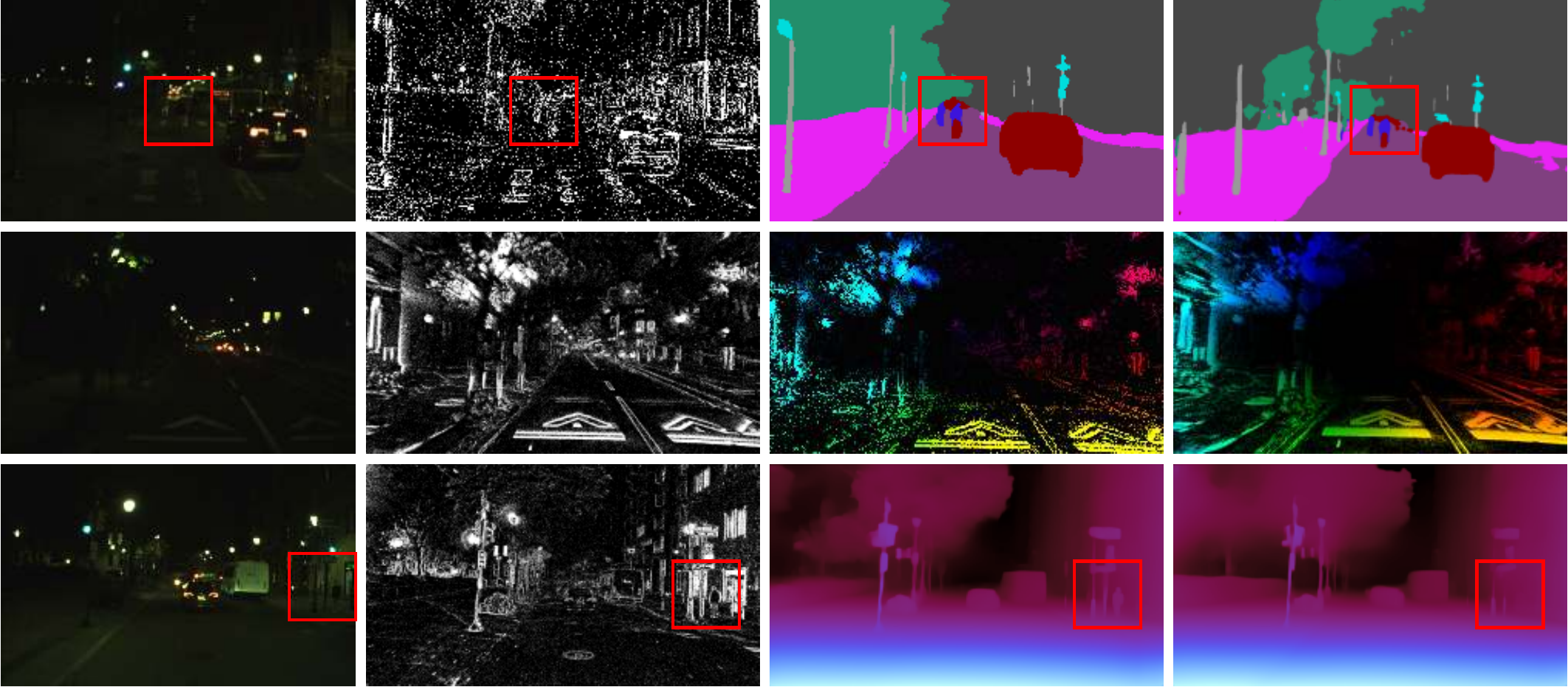}
    \caption{Generalization to nighttime data}
    \label{fig:robust:car_daytime_to_nighttime}
    \end{subfigure}

    \vspace*{1ex}
    \begin{subfigure}[b]{\linewidth}
    \centering
    \includegraphics[width=\linewidth]{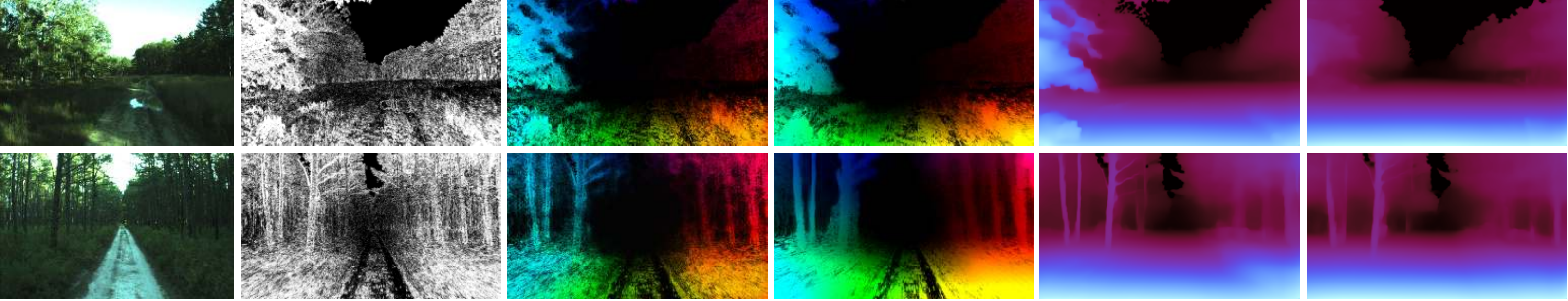}
    \caption{Generalization to forest environments}
    \label{fig:robust:car_daytime_to_forest}
    \end{subfigure}

\end{subfigure}%
\hfill
\begin{subfigure}[b]{0.54\linewidth}
\centering
    \begin{subfigure}[b]{\linewidth}
    \centering
    \includegraphics[width=\linewidth]{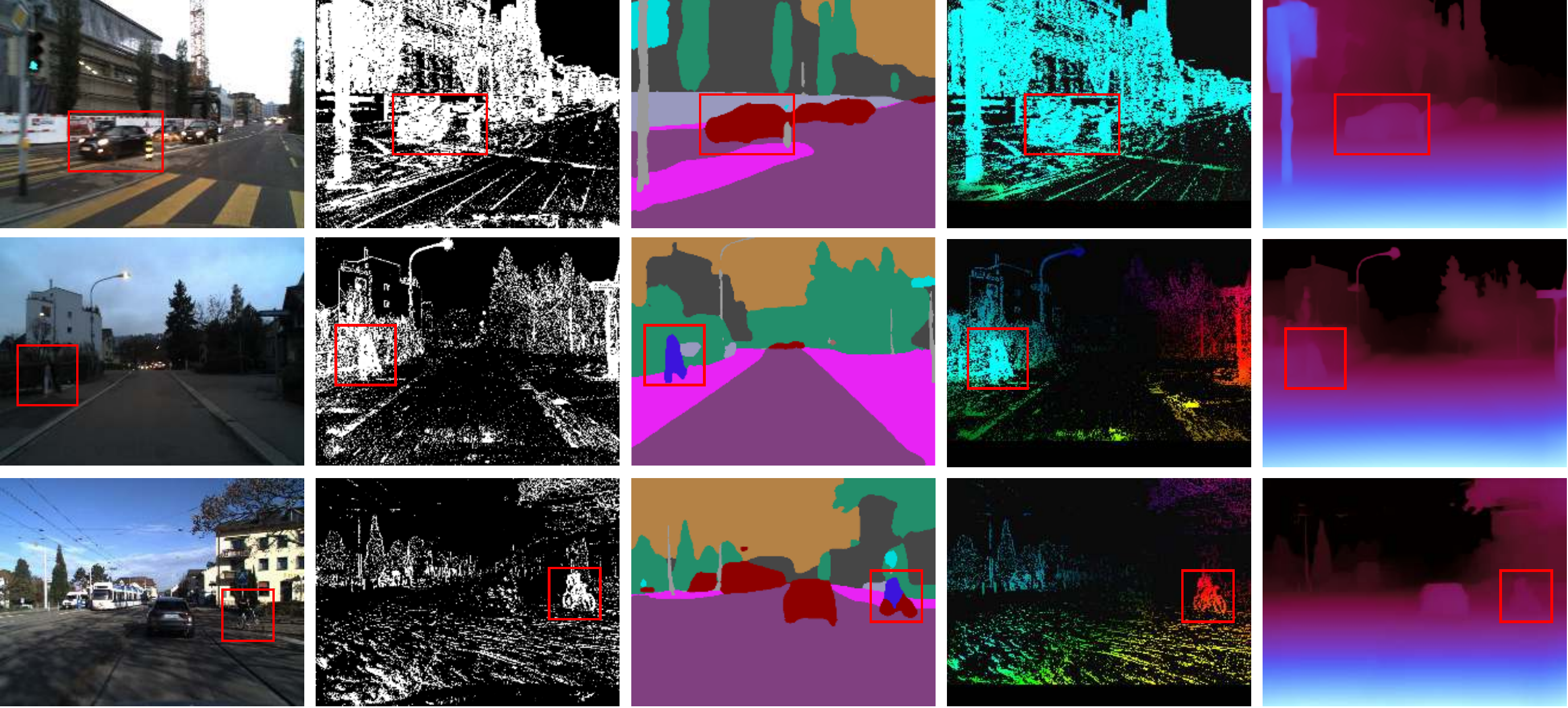}
    \caption{Generalization to different event sensors and resolution}
    \label{fig:robust:car_m3ed2dsec}
    \end{subfigure}

    \vspace*{1ex}
    \begin{subfigure}[b]{\linewidth}
    \centering
    \includegraphics[width=\linewidth]{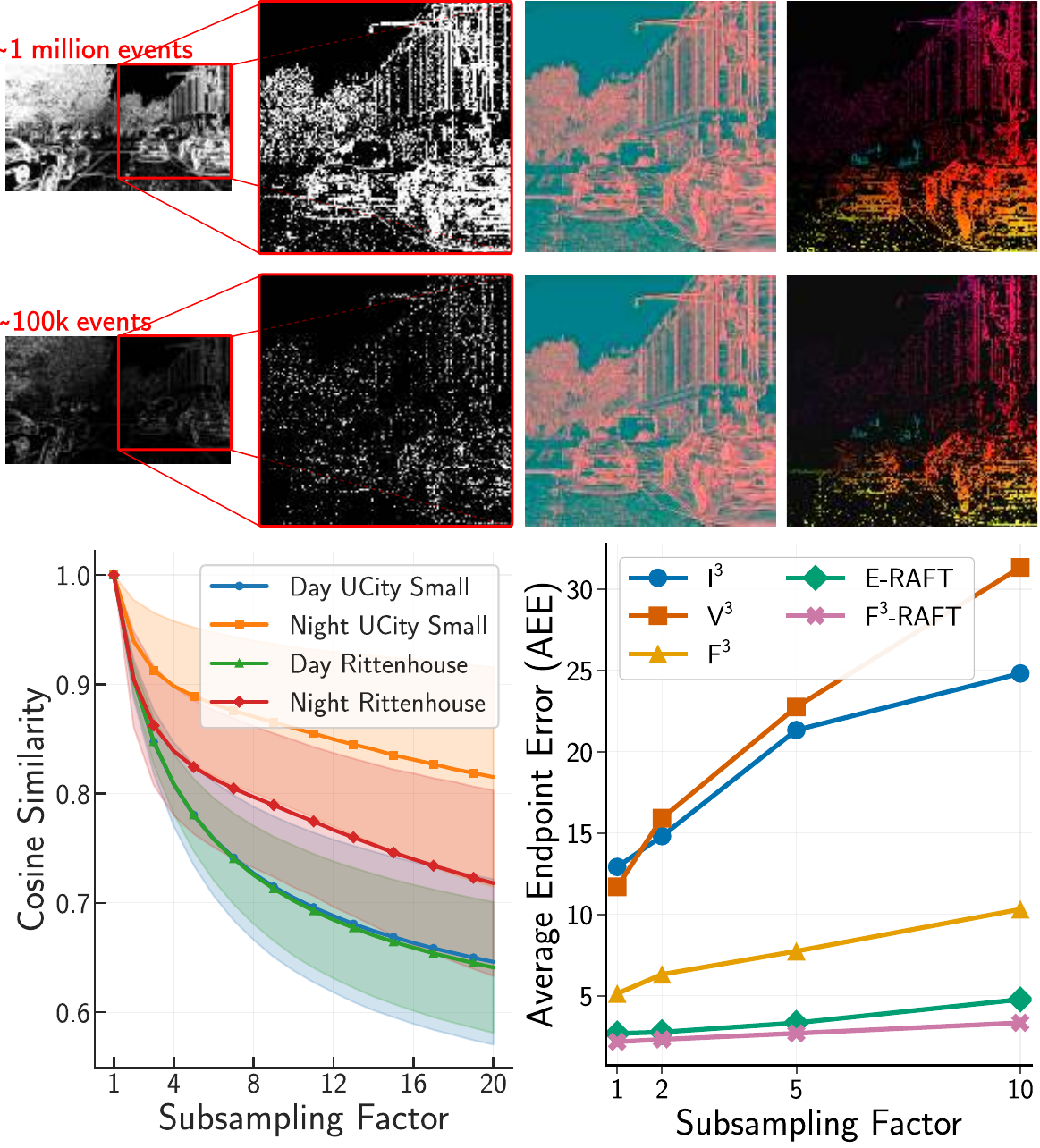}
    \caption{Robustness to different event rates}
    \label{fig:robust:subsampling}
    \end{subfigure}

\end{subfigure}
\caption{
\small
\textbf{Robustness across robotic platforms, lighting conditions, dynamic vision sensors, and event rates}
\textbf{(a)} RGB (first column), events (second), optical flow (third), and depth (fourth) using $\ff$ trained on daytime urban driving data when evaluated on Spot (top) and Falcon (bottom).
\textbf{(b, top)} Quantitative evaluation of optical flow on Car Nighttime using methods that were trained on Car Daytime.
\textbf{(b, bottom)}  First two columns show RGB images and event frames, respectively. Last two columns show qualitative results for segmentation from $\ff$ and $\vv$ (top row), flow from $\ff$ and E-RAFT (middle), and depth from $\ff$ and $\vv$ (bottom).
\textbf{(c)} Optical flow from $\ff$ and E-RAFT (columns 3 and 4), and depth from $\ff$ and $\vv$ (columns 5 and 6) for off-road driving data (RGB and event frames in columns 1 and 2); training was performed only on car daytime scenarios.
\textbf{(d)} $\ff$-based segmentation (column 3), flow (column 4), and depth (column 5) trained on M3ED car daytime driving data generalizes to DSEC data with a different sensor and resolution.
\textbf{(e, top)} First three principal components (using PCA) in column 3 for different event rates (1 million vs.\@ 100k) and $\ff$-based optical flow in column 4.
\textbf{(e, bottom)} Cosine similarity of $\ff$ features at different levels of sub-sampling with respect to features of the original event stream (left).
\new{Average endpoint error (AEE) on the M3ED Car Daytime data under different sub-sampling rates.}
\new{Subsampling is performed by dropping events uniformly at random.}
}
\label{fig:robust}
\end{figure}

\paragraph{Lighting and environmental conditions}
Fig.~\ref{fig:robust:car_daytime_to_nighttime} (top)shows that without any additional training, $\ff$-based optical flow (unsupervised and $\ff$-RAFT) trained on daytime urban driving data can predict very well even on nighttime driving sequences.
$\ff$-RAFT is more robust here than E-RAFT.
Fig.~\ref{fig:robust:car_daytime_to_nighttime} (bottom) gives a qualitative understanding of $\ff$-based segmentation, optical flow, and depth in this setting.
$\ff$-based approaches are quite good, while baseline approaches such as $\vv$ hallucinates objects and produce noisy depth estimates.
Optical flow from E-RAFT seems slightly more accurate than $\ff$.
Fig.~\ref{fig:robust:car_daytime_to_forest} shows that without additional training, $\ff$-based optical flow and depth trained on urban driving generalize to forested areas with much higher event rates.
Notice the sharp boundaries at the branches and tree trunks for both tasks, accurate depth of the sky, and smooth depth gradient on the forest floor.
\new{E-RAFT generalizes comparably to $\ff$. Depth from $\vv$ is quite inaccurate; it ignores the tree and smooths out fine branches.
}

\paragraph{Different sensors and event rates}
Fig.~\ref{fig:robust:car_m3ed2dsec} shows that without additional training, $\ff$ trained on M3ED data from Prophesee EVK4 1280 $\times$ 720 sensor can predict accurately on DSEC which has a smaller 640 $\times$ 480 sensor.
Note the heavily saturated regions in the event frame due to fast-moving objects (top row), objects in the foreground blurring into the background (middle row), and objects that are impossible to discern in the RGB images (bottom row).
Fig.~\ref{fig:robust:subsampling} shows that $\ff$ features are resilient to significant sub-sampling of input events. The top panel shows that $\ff$ features are visually similar even with 10\% of the original events, even if the event frames look uninformative. This is an important property for both robustness and real-time applications, also see Tab.~\ref{tab:subsampling_speed}.
Bottom left panel shows that $\ff$ features computed on 25\% of the input events are similar to those from all events (cosine similarity $\sim$0.8).
Bottom right panel provides quantitative results of optical flow for different methods on M3ED Car Daytime test set with sub-sampled event streams.
Average endpoint error (AEE) degrades by less than 10\% even with 50\% of the input events for $\ff$. Non-causal RAFT-based methods are more resilient than causal unsupervised ones because they use two event intervals, and thereby more events.
$\ff$-RAFT is more robust than E-RAFT. With 20\% of the events, it achieves roughly the same AEE as E-RAFT does with all events.


\section{Discussion and Related Work}
\label{s:discussion}

\paragraph{Principles of representation learning}
A representation is a minimal statistic of the data that is sufficient for a task \cite{achille2018emergence}. This broad idea has been used to develop architectures that address symmetries and invariances in the data \cite{brunaInvariantScatteringConvolution2013,cohen2021equivariant,xu2023se3equivariant}.
In contrast, foundation models \cite{ijepa,oquab2023dinov2,he2022masked} learn inductive biases by forcing features to be insensitive to transformations of the data.
Almost all existing approaches in computer vision belong to one of these paradigms.
This paper is a step towards developing these paradigms for event data.
We provide a mathematical argument to specialize the idea behind spatiotemporal foundation models like V-JEPA \cite{assran2025v} and Video MAEs \cite{tong2022videomae} to event data.
We argue that event data should be projected into a data-dependent basis that can be learned by forcing the representation to predict future events.

Time-surfaces \cite{sironi2018hats,lagorce2016hots,zhu2023event,bisulco2025ev}, voxel-grids \cite{rebecq2019high,mostafavi2021learning}, and event frames modify events to conform to dense, grid/frame-based inputs required by deep networks.
\new{Also see recent surveys on event representations \cite{zheng2024deeplearningeventbasedvision} and techniques to tailor them to specific tasks \cite{ghosh2025eventstereodepth,GUAMANRIVERA2025129899,zheng2024deeplearningeventbasedvision}.}
Voxel grids are fast and the de facto choice for many tasks \cite{nam2022stereo}, but they have a large memory footprint.
They do not exploit the sparsity of events, and as shown, they are more susceptible to noise.
Time surfaces with user-designed filter banks may not exploit the structure in event data.
Histograms \cite{liu2018adaptive,zihao2018unsupervised} cannot resolve overlapping structures and motion because they do not represent precise timing information.
Asynchronous architectures build spatiotemporal graphs  \cite{schaefer2022aegnn} \new{or tokens \cite{fang2026event2vecprocessingneuromorphicevents} from events}. Such architectures are useful because one could query the representation at any time \cite{scheerlinck2019asynchronous}.
In this sense, they are similar to $\ff$.
Increasingly, researchers are moving away from hand-crafted features of event data \cite{huang2023eventpoint,zubic2023chaos}.
\new{There is also work on new biologically-motivated task-specific event-generation models to reduce the bandwidth required for processing events \cite{So_2025_NGS}.}

$\ff$ has a few unique properties compared to existing event representations. Our architecture exploits sparsity because it only executes on events. The hash-encoder can be computed asynchronously for every event. Averaging in time and focal loss in the training objective makes $\ff$ resilient to noise.
$\ff$ can be queried at any time, which makes it very useful for robot autonomy.
A large majority of computer vision algorithms and architectures are designed for three-channel images.
As we showed, $\ff$ can be plugged directly into any of these approaches. Deep networks trained on RGB data can be easily fine-tuned to use events with $\ff$.
Our architecture works as-is for building global features with richer semantics like those in the visual cortex, by using more convolutional layers and predicting larger future patches.

\paragraph{Event representations for robotics}
Cameras support a broad range of robotic perception tasks. They are low-power, lightweight, and have a large field of view.
Event cameras expand these capabilities. They can enable robots to operate in low-light and low-latency settings. However, engineering challenges make it difficult to build event perception on SWaP-constrained platforms \cite{chaney2023m3ed}.
There have been demonstrations of real-time obstacle avoidance \cite{falanga2020dynamic}, navigation \cite{falanga2018foldable}, landing \cite{paredes2024fully}, object tracking and grasping \cite{mueggler2014event,wang2022ev,mitrokhin2018event,sanket2020evdodgenet}. These approaches demonstrate the utility of event-based sensing, but they have not been shown to generalize to different robots, tasks, or environments.
In contrast, $\ff$ is a fast and general representation for event data. The same pre-trained $\ff$ representation can be used for different tasks, environmental settings, and robot platforms. Our paper is a step towards real-time and event perception.

CPUs and GPUs are ill-suited to handle low-latency, high-throughput, and sparse event data. We can compute $\ff$ features 120--440 Hz on HD and VGA resolutions, respectively, and predict tasks at 25--75 Hz at HD resolution. The key reason for this efficiency is that $\ff$ exploits the sparsity of event data. This work brings event camera processing computationally at par with RGB-based methods. It shows that event processing can be fast even on commodity hardware. Neuromorphic computing \cite{davies2021advancing} or ASICs that perform computation directly at the pixel \cite{carey2013100} can further improve the energy usage of $\ff$.
\new{
In Sec.~\ref{s:app:jetson} we also describe how to optimize $\ff$ for memory-bandwidth bound edge devices. We reduced the memory bandwidth by pre-computing the hash encoding and fusing CUDA kernels. We also reduced the cost of convolutional layers in $\ff$ by aggressively down/up-sampling using a U-Net architecture.
With these optimizations, $\ff$ for VGA resolution events can be computed at $\sim$1280 Hz on Nvidia 4090 and $\sim$82 Hz on Nvidia Jetson Orin NX.
}
Looking ahead, we imagine $\ff$ being implemented very close to the camera circuitry, say on an FPGA, with the event camera driver acting as an analog of the mammalian retina. In this sense, $\ff$ reduces the intrinsic redundancy of event data into low-level features useful for many downstream tasks. We designed $\ff$ with such an implementation in mind: hash tables and convolutions both map well onto an FPGA, and $\ff$ features are robust to aggressive sub-sampling of the event stream (Fig.~\ref{fig:robust:subsampling}). We have not built one here, but our results suggest it would be fruitful.

This paper shows that event cameras can significantly advance the state-of-the-art in robot perception. $\ff$ enables dense depth prediction and scene segmentation at frame rates much faster than LiDAR or RGB cameras at night-time---without training on such data.
\new{Our overarching vision is that a single event camera could, someday, substitute these expensive, domain-specific sensors.
}
Self-supervised pre-training followed by task-specific supervision could be a good strategy for achieving this vision.
But it has not yet been exploited for event data, perhaps, because annotating event data is difficult; although there is some very recent work \cite{cao2026generative}.
$\ff$ is a good example of the efficacy of this paradigm. It was trained using self-supervised objectives and benefited from cross-modal supervision from multi-sensor robotic platforms.

\begin{newenv}

\paragraph{Limitations and directions for future work}

$\ff$ computes the representation within a fixed time window. It would be interesting to build a recurrent variant of $\ff$ to handle situations when the number of incoming events changes over time.
Global $\ff$ features can allow us to trace the trade-off depicted in Fig.~\ref{fig:fig1}(A). Such features would enable $\ff$ to predict at longer time-scales across large occlusions or disocclusions.
The training objective in $\ff$ does not use information of the robot dynamics or the environment. Therefore, $\ff$ features are quite redundant, see Fig.~\ref{fig:fig1}(D) and Fig.~\ref{fig:robust:subsampling}. It is of interest to develop more compressed representations that are as versatile as $\ff$.
\end{newenv}

\section{Materials and Methods}
\label{s:methods}

\subsection*{Details of the $\ff$ architecture}
We use a $\varphi$ that can learn multi-scale spatiotemporal features of individual events, see Sec.~\ref{s:instantiation} and Fig.~\ref{fig:architecture}. The $L$ scales are independent grids with resolution that is a geometric progression between \( R_{\min} \in \R^3 \) to \( R_{\max}  \in \R^3 \).
When $\Om$ is HD pixel space and $\Delta t = 20$, $R_{\max} = (8, 8, 1)$ and $R_{\max} = (180, 320, 8)$.
At each level $l$ and resolution $R_l$, the $F$-dimensional feature of an event coordinate $(s,u)$ is a trilinear interpolation of the features of the 8  vertices enclosing the point. Feature vectors are concatenated across levels to get $\varphi(s,u) \in \reals^n$ with $n \equiv LF$.
For coarse resolutions, the number of vertices, \(\prod_{i=1}^{3} (R_{li} + 1)\), is smaller than the size of the hash table $T$, which leads to a unique feature vector for each vertex. For finer resolutions, the number of grid points at a level can exceed $T$. In this case, we can use a hash function $h: \integers_+ \times \Om \to \{1,\dots,T\}$ as suggested in \cite{mueller2022instant} to index into the hash table.

We want $\rho$ to be a universal approximator acting on $\bar \varphi(t, \cdot)$. To keep the computation efficient, we assume that a future event at $(s,u)$ with $s \in [t, t+\D t)$ only depends on spatio-temporally nearby past events in $[t-\D t, t) \times \Om_u$. The representation $\hat \xi(t,\cdot)$ is thus quite local. This is a reasonable choice for most natural scenes because the optical flow, which dictates the region around a pixel that is predictive of an event at that pixel, is bounded.
For all experiments, $\Om_u$ has size $37 \times 37$.
ConvNeXt V2 \cite{Woo2023ConvNeXtV2} blocks with $7 \times 7$ kernels and 32 channels are used for $\rho: \Om_u \times \R^n \to \Om_u \times\R^p$. Unlike the original ConvNeXt V2 architecture, we do not downsample in these blocks. The representation $\ff(t, \cdot)$ is thus supported on the entire spatial domain $\Om$. Our largest model for $\rho$ has $\sim$38K trainable parameters.

\subsection*{A robust and fast training procedure for $\ff$}
The objective in Eqn.~\ref{eq:objective_detail} is used to fit $\varphi$ and $\rho$. To build resilience to noise, we randomly sub-sample past events $e^-$ during training and keep future events $e^+$ unchanged.
We reimplemented ConvNeXt V2 for $\rho$ with sparse convolutions from MinkowskiEngine \cite{minkowskiengine} and TorchSparse++ \cite{tang2023torchsparse++}; the latter was more efficient. In small networks like $\ff$, fused dense kernels on structured, fixed-size inputs outweigh the benefits of using sparse convolutions.
Our dense implementation takes $\sim$8.4 ms to process $\sim$200K HD events during inference. In comparison, our TorchSparse++ implementation requires $\sim$23 ms.
All experiments in the paper use compiled versions of dense convolutions.
We expect sparse convolutions to be better supported in the future.
\new{Also see Tab.~\ref{tab:subsampling_speed} which discusses the time required to compute the hash encoding in $\ff$ for different event rates.}
All inference statistics are measured on a single Nvidia 4090 GPU, unless stated otherwise. We train $\ff$ using a distributed setup across two GPUs in half-precision. All downstream tasks are trained on a single GPU using full precision.
\new{Sec.~\ref{s:app:jetson} describes how to optimize $\ff$ for devices like Nvidia Jetson.}

\subsection*{Implementation details of downstream tasks}

We use a SegFormer B3 network ($\sim$47.3M parameters) pretrained on Cityscapes to process $p$-channel $\ff$ features. Training and testing use standard train/test splits of each dataset. Transfer from M3ED to DSEC is evaluated as follows. Pseudo-labels from RGB for event data in M3ED contain unannotated pixels due to homography-based warping. This does not affect evaluation in M3ED because it appears in both train and test splits. But DSEC segmentation labels have a different pattern of unannotated regions. To mitigate this mismatch, we use 600 $\times$ 800 crops of M3ED HD data for training. For testing transfer from M3ED (HD) to DSEC (VGA), we pad the latter appropriately.

Sec.~\ref{s:flow} described a neural network $\psi: \R^p\times \Om \to \R^2 \times \Om$ to compute unsupervised optical flow. We train $\psi$ to minimize the objective in Eqn.~\ref{eq:flow_objective} with $\ff$ frozen.
In M3ED, DSEC, and MVSEC, we set the number of $\sigma$ levels to 5, 4, and 2, respectively. This choice is based on the event camera resolutions and the typical amount of motion in the scene. The smoothness regularization parameter in Eqn.~\ref{eq:flow_objective} is fixed to $\lambda = 10^{-3}$ for all experiments.
Four ConvNeXt V2 blocks with 9 $\times$ 9 kernels and $p$-channels are used for $\psi$. This gives $\sim$28K weights and a receptive field of 33 $\times$ 33. We use the same architecture for all datasets and robotic platforms.
To enable fair comparisons, we project $\vv$ and $\ii$ to $p$ channels before passing them to $\psi$. For $\vv$, this is done using a linear layer followed by layer normalization \cite{baLayerNormalization2016} on the $\Delta t$ channels. For $\ii$, the same projection is applied to the two polarity channels.

For depth estimation, training of $\psi$ (DepthAnythingV2 Base architecture) is done in two stages. The first stage focuses on learning accurate object boundaries, while the second stage recovers metric depth. In both stages, a gradient matching regularizer Eqn.~\ref{eq:depth_regularizer} encourages sharp object boundaries. We use 4 spatial scales $\sigma$ and set the regularization $\lambda$ to 0.3 for both objectives Eqn.~\ref{eq:depth_objective_1} and \ref{eq:depth_objective_2} across all experiments. During training, we restrict the loss to pixels with disparity below 384. We use random spatial crops of size 518 $\times$ 518 for input events and target disparity, to improve training efficiency.


\clearpage 

%
\bibliographystyle{sciencemag}
\bibliography{bib/pratik,bib/25_richeek_eventfield,bib/rebuttal}

%
%
%
%
%
%


\section*{Acknowledgments}

We are grateful for the discussions with Daniel Gehrig on event representations.
\paragraph*{Funding:}
RD and PC were supported by grants from the National Science Foundation (IIS-2145164, CCF-2212519), IoT4Ag ERC under NSF Grant EEC-1941529 and DSO National Laboratories, Singapore.


\subsection*{Supplementary materials}
Supplementary Text\\
Figures Fig.~\ref{fig:lossalphajustification} to \ref{fig:evimo2}\\
Tables Tab.~\ref{tab:m3ed_sequences} to \ref{tab:jetson_latency}\\
References \textit{(96-\arabic{enumiv})}\\ 
Movie S1


\newpage


\begin{appendix}

\renewcommand\thesection{S.\arabic{section}}
\renewcommand\thefigure{S.\arabic{figure}}
\renewcommand\thetable{S.\arabic{table}}
\renewcommand{\thesubsection}{\thesection.\arabic{subsection}}
\setcounter{figure}{0}
\setcounter{section}{0}
\setcounter{table}{0}
\setcounter{equation}{0}
\setcounter{page}{1} 
\setcounter{theorem}{0}

\clearpage

\begin{center}
\section*{Supplementary Materials for\\ \scititle}

\author{
	Richeek~Das$^\ast$,
	Kostas~Daniilidis,
	Pratik~Chaudhari$^\ast$\\
	\small$^\ast$Corresponding authors. Email: richeek@seas.upenn.edu, pratikac@seas.upenn.edu\\
}
\end{center}


\section{Proof of Theorem~\ref{thm:main}}
\label{s:app:proofs}

We assume that there exist constants $c_1$ and $c_2$ such that
(i) the operator norm $\norm{A}_2 \leq c_1$, and
(ii) if $N_\LL(e) = \min_{\BB \in \LL} \norm{e_\BB}_0$ is the minimal sparsity of the signal $e$ in the library, then $N_\LL(A e) \leq c_2 N_\LL(e)$, i.e.,  the operator $A$ does not decrease sparsity of the signal by more than a factor $c_2$.

Consider the objective in Eqn.~\ref{eq:objective}. The joint minimizers $\hat{A}, \hat{\B}$ of this objective can be used to recover the future statistic $\xi^+$. In Theorem~\ref{thm:main}, we show that doing so incurs a risk that is additively off from the oracle dynamics risk by a constant factor, and the ideal denoising risk by a logarithmic factor. In this case, the denoising oracle has access to the best basis $\B \in \LL$ for denoising the events, and the knowledge of which coordinates projected in this basis have energy larger than the noise variance. On the other hand, the risk for the dynamics oracle is the regression error on the denoised past statistics provided by the ideal denoising oracle with full knowledge of the future statistics $\xi^+$. This can be thought of as an errors-in-variables regression problem, where $\xi^{-}$ is a nuisance, which is only partially known to both the dynamics and denoising oracles.

Proving the aforementioned strong properties for the joint minimizers $\hat A, \hat \B$ is non-trivial. We break our proof idea into two major steps. We show that a two-step procedure, first estimating an ideal basis for denoising, followed by regression to find the dynamics operator, satisfies similar claims for recovering the future statistic $\xi^+$ with high probability. Second, we show that joint estimates of $\hat A, \hat \B$ that minimize Eqn.~\ref{eq:objective} can recover $\xi^+$ at least as well as the two-step procedure, concluding our proof.

\paragraph{Two-step estimation of the dynamics operator and the denoising basis}
Consider the following denoising complexity functional
\beq{
    D(\xi, \tilde{\xi}) = \norm{\xi - \tilde{\xi}}^2 + \L_n N_{\LL} (\tilde{\xi}).
    \label{eq:denoising_complexity}
}
To find the ideal basis for denoising, in this two-step procedure, we follow the strategy of \cite{donoho1994ideal}. Let the empirically denoised estimate $\tilde \xi^- = \argmin_{\tilde \xi} D(e^-, \tilde \xi)$. Observe that this minimization is equivalent to applying the hard thresholding operator on a basis representation of $e^-$, given by $\tilde \BB = \argmin_{\B \in \LL} \sum_i \min((e_{\BB}^-)_i^2, \L_n)$. Specifically, $(\tilde \xi^-_{\tilde\BB})_i = \ind{\abs{(e_{\tilde \BB}^-)_i} > \sqrt{\L_n}} (e^-_{\tilde \BB})_i$.

From \cite{donoho1994ideal}[Theorem 1], we know that this estimator $\tilde \xi^-$ incurs a risk that is only a logarithmic factor off the ideal risk, with high probability. We state it here for ease of reference. With probability at least $\pi_n = 1 - e / M_n$,
\beq{
    \norm{\xi^- - \tilde \xi^-}^2 \leq D(\xi^-, \tilde{\xi}^-) \leq (1 - 8 /\l)^{-1} \L_n R_{\text{denoising}}(\xi^-, \LL)
    \label{eq:donoho_thm}
}
where, $R_{\text{denoising}}(\xi^-, \LL) = \sum_i \min((\xi^-_{\BB^*})_i^2, 1)$ and $\BB^* = \argmin_{\B \in \LL} \sum_i \min((\xi^-_{\BB})_i^2, 1)$ is the ideal basis for denoising $e^-$. Similarly, the signal that attains this ideal denoising risk can be defined as $\xi^{*-}$, where $(\xi^{*-}_{\BB^*})_i = \ind{\abs{(\xi_{\BB^*}^-)_i} > 1} (e^-_{\BB^*})_i$. Note that obtaining this ideal basis and estimate requires us to know which coordinates of $\xi^-$, when projected on a basis, are larger than 1. This information is unknown when we are finding our empirical estimate $\tilde \xi^-$.

Broadly, we are interested in performing regression using these denoised variables to recover $\xi^+$. Theoretical interest is in comparing the dynamics risk incurred by using $\tilde \xi^-$ (empirical estimate) or $\xi^{*-}$ (ideal estimate) as the regression predictors. 
Next, we introduce a theoretical denoising estimator $\xi^{0-} = \argmin_{\tilde{\xi}} D(\xi^-, \tilde{\xi})$ requiring knowledge of $\xi^-$. The purpose of this estimator is to serve as a bridge in the mathematical analysis to compare the properties of $\tilde \xi^-$ and $\xi^{*-}$. We define the empirical and theoretical dynamics operators as $\tilde A = \argmin_A \norm{e^+ - A \tilde \xi^-}$ and $A^0 = \argmin_A \norm{\xi^+ - A \xi^{0-}}$ respectively. We  relate $\norm{\xi^+ - \tilde A \tilde \xi^-}$ (empirical dynamics risk) and $\norm{\xi^+ - A^0 \xi^{0-}}$ (theoretical dynamics risk) by the following steps,
\beqs{
    \norm{\e^+ - \tilde A \tilde \xi^-}^2 = \norm{\xi^+ - \tilde A \tilde \xi^-}^2 + \norm{\nu^+}^2 + 2 \langle \nu^+, \xi^+ - \tilde A \tilde \xi^-\rangle \leq \norm{e^+ - A^0 \tilde \xi^-}^2
}
since, $\tilde A$ is the empirical minimizer. Rearranging and combining, we get,
\beq{
   \norm{\xi^+ - \tilde A \tilde \xi^-}^2 \leq \norm{\xi^+ - A^0 \tilde \xi^-}^2 + 2\langle \nu^+, \tilde A\tilde \xi^- - A^0 \xi^{0-} \rangle + 2\langle \nu^+, A^0(\xi^{0-} - \tilde \xi^-) \rangle
   \label{eq:empirical_dynamics_ub}
}
We write $\norm{\xi^+ - A^0 \tilde \xi^-}^2$ in terms of a more useful quantity $\norm{\xi^+ - A^0 \xi^{0-}}^2$, the theoretical dynamics risk,
\beq{\aed{
    \norm{\xi^+ - A^0 \tilde \xi^-}^2 &= \norm{\xi^+ - A^0 \xi^{0-}}^2 + \norm{A^0(\xi^{0-} - \tilde \xi^-)}^2 \\
    &+ 2 \langle A\xi^- - A^0 \xi^{0-}, A^0(\xi^{0-} - \tilde \xi^-) \rangle + 2 \langle \zeta, A^0(\xi^{0-} - \tilde \xi^-) \rangle.
}
\label{eq:theoretical_dynamics_eq}
}
Now, to upper bound the empirical dynamics risk, we need to obtain bounds for the quantities on the right-hand side of Eqn.~\ref{eq:empirical_dynamics_ub} and Eqn.~\ref{eq:theoretical_dynamics_eq}. We bound two quantities on the right-hand side of Eqn.~\ref{eq:theoretical_dynamics_eq} as
\beq{
\aed{
    \norm{A^0(\xi^{0-} - \tilde \xi^-)}^2
    &\leq \norm{A^0}^2(\norm{\xi^{0-} - \xi^-}^2 + \norm{\tilde \xi^- - \xi^-}^2) \\
    &\leq 2c_1 \frac{\l -4}{\l -8} \L_n R_{\text{denoising}}(\xi^-,\LL)
}
\label{eq:proof_ineq1}
}
with probability at least $\pi_n$ using Eqn.~\ref{eq:donoho_thm}, and
\beq{
\langle A\xi^- - A^0 \xi^{0-}, A^0(\xi^{0-} - \tilde \xi^-) \rangle \leq 2\sqrt{c_1}\rbr{\norm{\xi^+ - A^0 \xi^{0-}}^2 + D(\xi^-, \tilde \xi^-)}.
\label{eq:proof_ineq2}
}
However, the remaining terms are inner products of bounded quantities with the noise $\nu^+$ and $\zeta$. We can upper bound them by the following random variable,
\beqs{
    W(k_1, k_2) = \sup\cbr{\inner{\nu}{\xi_1 - \xi_2} : \norm{\xi_1 - \xi_2}^2 \leq k_1; \forall i \in \cbr{1,2}, \L_n N_{\LL}(\xi_i) \leq k_2}
}
where $\nu \in N(0,I)$. It can be shown that with probability greater than $\pi_n$, we have $W(k_1,k_2) \leq 2\sqrt{k_1k_2} / \l \leq (k_1 + k_2) / \l$ for all $k_1$ and $k_2$. The proof for this statement follows from a similar argument as the one in \cite{donoho1994ideal}[Lemma 2]. We use this concentration inequality to upper bound the inner products in Eqn.~\ref{eq:empirical_dynamics_ub} and Eqn.~\ref{eq:theoretical_dynamics_eq}. The following two inequalities hold with probability at least $\pi_n$,
\beq{
\aed{
\inner{\nu^+}{\tilde A\tilde \xi^- - A^0 \xi^{0-}}
&\leq W\rbr{4\norm{\xi^+ - \tilde A \tilde \xi^-}^2, c_2D(\xi^-, \tilde \xi^-)}\\
&\leq \frac{1}{\l} \rbr{4\norm{\xi^+ - \tilde A \tilde \xi^-}^2 + c_2D(\xi^-, \tilde \xi^-)},\\
\inner{\nu^+}{A^0(\xi^{0-} - \tilde \xi)}
&\leq W\rbr{4c_1 D(\xi^-,\tilde \xi^-), c_2 D(\xi^-,\tilde \xi^-)}\\
&\leq \frac{4 \sqrt{c_1c_2}}{\l} D(\xi^-,\tilde \xi^-).
}
\label{eq:proof_ineq4}
}
The same inequality as presented in Eqn.~\ref{eq:proof_ineq4}, holds for $\inner{\zeta}{A^0(\xi^{0-} - \tilde \xi)}$. In addition, we also observe a simple relation between the theoretical $(A^0,\xi^{0-})$ and the ideal estimates $(A^*, \xi^{*-})$. With probability greater than $\pi_n$,
\beq{
    \norm{\xi^+ - A^0 \xi^{0-}}^2 \leq \norm{\xi^+ - A^* \xi^{*-}}^2 + \frac{2c_1\l}{\l-8} \L_n R_{\text{denoising}} (\xi^-, \LL).
    \label{eq:theoretical_ideal}
}
This ideal dynamics risk $R_{\text{dynamics}}(\xi, \LL) = \norm{\xi^+ - A^* \xi^{*-}}^2$. Rearranging and combining the inequalities Eqns.~\ref{eq:empirical_dynamics_ub} to \ref{eq:theoretical_ideal}, we have with probability greater than $1 - c' e / M_n$, that
\beq{
    \norm{\xi^+ - \tilde A \tilde \xi^-}^2 \leq c_3' R_{\text{dynamics}}(\xi, \LL) + c_4' \L_n R_{\text{denoising}}(\xi^-, \LL)
    \label{eq:twostep_proofstatement}
}
where $c_3'$ and $c_4'$ are constants dependent on $c_1,c_2$ and $\l$. This concludes our analysis of the two-step estimation procedure, showing that our empirical risk is only a constant factor off the ideal dynamics risk and a logarithmic factor off the ideal denoising risk. Now, we extend this discussion to show that we can jointly estimate the dynamics operator $A$ and the denoising basis $\B$, using a single optimization objective to achieve equivalent asymptotic guarantees.

\paragraph{Joint estimation of the dynamics operator and the denoising basis} Consider the following new complexity functional that can be optimized to jointly recover $A$ and $\BB$,
\beq{
    D(e^-, \xi^+, A, \B) = \norm{\xi^+ - A \hat \xi^-}^2 + \L_n \norm{\xi_{\B}^-}_0
}
where $(\hat \xi_{\BB}^-)_i = \ind{|(e_{\BB}^-)_i| > \sqrt{\L_n}} (e_{\BB}^-)_i$. We jointly optimize this complexity functional to get the new empirical estimates as $\hat A, \hat \B = \argmin_{A,\B} D(e^-, e^+, A, \B)$. Similarly, we define the new theoretical estimates as $A', \B' = \argmin_{A,\B} D(e^-, \xi^+, A, \B)$, serving as a means to compare between the joint empirical estimates $\hat A, \hat \B$ and the two-step estimates $\tilde A, \tilde B$. We go through the following steps to relate $D(e^-, \xi^+, \hat A, \hat \B)$ and the two-step empirical risk $\norm{\xi^+ - \tilde A \tilde \xi^-}^2$, which we have shown to have certain desired properties in Eqn.~\ref{eq:twostep_proofstatement}.
\beqs{
    D(e^-, e^+, \hat A, \hat \B) = D(e^-, \xi^+, \hat A, \hat \B) + \norm{\nu^+}^2 + 2\inner{\nu^+}{\hat A \hat \xi^-} \leq D(e^-, e^+, A', \B')
}
since $\hat A, \hat \B$ are the minimizers of $D(e^-, e^+, \cdot, \cdot)$. Rearranging and combining the terms, we get,
\[
    D(e^-, \xi^+, \hat A, \hat \B) \leq D(e^-, \xi^+, A', \B') + 2 \inner{\nu^+}{A' \xi'^- - \hat A \hat \xi^-}.
\]
Now notice that, $\inner{\nu^+}{A' \xi'^- - \hat A \hat \xi^-} \leq W\rbr{4D(e^-, \xi^+, \hat A, \hat \B), c_2D(e^-, \xi^+, \hat A, \hat \B)}$. Combined with the fact that $A', \B'$ are the minimizers of $D(e^-, \xi^+, \cdot, \cdot)$, we can further extend the above inequality to connect the two of our desired quantities,
\beqs{
    D(e^-, \xi^+, \hat A, \hat \B) \leq \norm{\xi^+ - \tilde A \tilde \xi^-}^2 + D(\xi, \tilde \xi^-) + 2W\rbr{4D(e^-, \xi^+, \hat A, \hat \B), c_2D(e^-, \xi^+, \hat A, \hat \B)}.
}
Finally, combining this with the concentration inequality for $W(\cdot,\cdot)$, Eqn.~\ref{eq:donoho_thm} and Eqn.~\ref{eq:twostep_proofstatement}, we have with probability greater than $1 - c''e/M_n$,
\beqs{
    \norm{\xi^+ - \hat A \hat\xi^{-}}^2 \leq D(e^-, \xi^+, \hat A, \hat \B) \leq c_3 R_{\text{dynamics}}(\xi, \LL) + c_4\L_n R_{\text{denoising}} (\xi^-, \LL).
}
This concludes our proof of Theorem~\ref{thm:main}.

\section{Details of the training objective in Eqn.~\ref{eq:objective_detail}}
\label{s:app:focal_loss}

Eqn.~\ref{eq:objective_detail} is the binary cross-entropy loss when $\a=0.5$ and $\g=0$. Parametric regression to minimize this loss often leads to miscalibration and overfitting to the data \cite{calibration}, especially for event data, which is sparse. This adversely affects generalization to new sequences, particularly under objectives such as ours in this paper, which are un/self-supervised. During training, $\ff$ features can end up being trivial if input and output event spatio-temporal volumes are very similar. However, with $\g > 0$, minimizing the focal loss in Eqn.~\ref{eq:objective_detail} minimizes the upper bound of the entropy-regularized KL divergence \cite{focallosscalibration}
\[
    \ell_t(\varphi, \rho, \psi)  \geq \sum_{s \in [t, t+\D t]} \sum_{u \in \Om} \KL(p_e, p_{\hat e}) + \text{H}(p_e) - \g \text{H}(p_{\hat e})
\]
where $p_e(s, u) = [e(s, u), 1-e(s,u)]$ is a probability distribution corresponding to events $e(s,u)$ at time $s$ and pixel $u$, and $p_{\hat e}(s,u) = [\hat e(s,u), 1-\hat e(s,u)]$ is the probability distribution corresponding to the predicted events $\hat e(s,u) = \psi(s, u; \ff(t, u))$. The quantity $\KL$ denotes the Kullback-Leibler divergence and $\text{H}(\cdot)$ stands for the Shannon entropy. 

\begin{figure}
    \centering
    \includegraphics[width=0.5\linewidth]{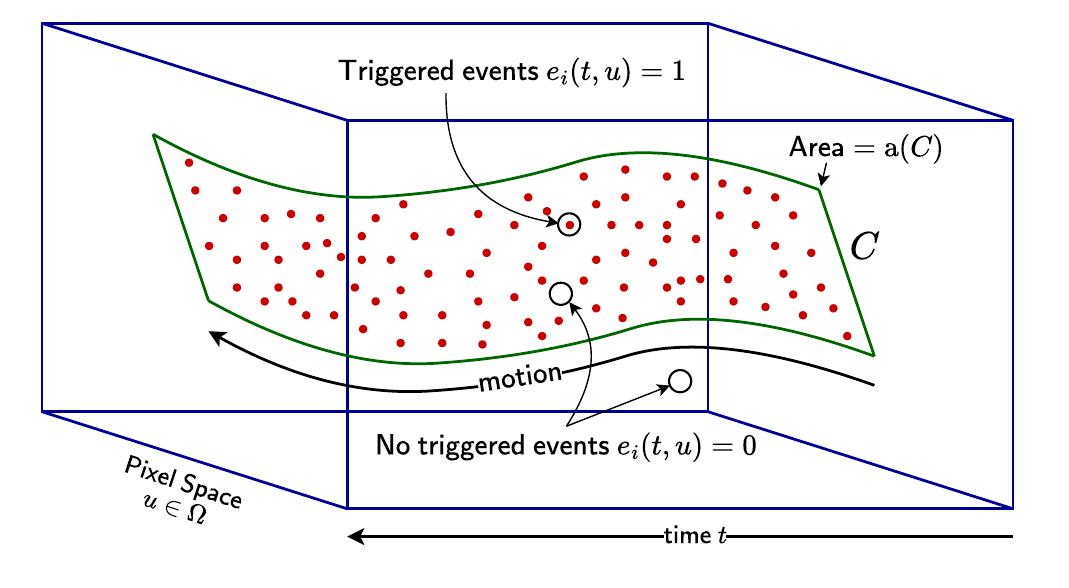}
    \caption{Events (in red) on a spatiotemporal surface $C$. Lemma.~\ref{lem:stochastic} states that if the weight in the cross-entropy loss $\a$ is too small, then the representation that minimizes the objective predicts that the entire region $C$ does not have any events. The parameter $\a$ must be chosen carefully to avoid such a trivial representation.}
    \label{fig:lossalphajustification}
\end{figure}
In the above expression, the regularization using the coefficient $\g$ encourages predictions $\hat e$ with a large entropy, and thereby prevents the model from becoming overconfident. We set $\g=2$ for all experiments in this work.
Entropy regularization is important for generalization and preventing feature collapse. Selecting $\a$ as suggested in Eqn.~\ref{eq:objective_detail} helps address class imbalance and also addresses the stochasticity of event data.

The same underlying scene statistic $\xi$ and the camera trajectory $x$ can produce different event streams due to the noise described in Eqn.~\ref{eq:event_generation}. Solving this inverse problem to recover $\xi$ is difficult, under these constraints of unknown scene statistics, dynamics, and noise.
Suppose we have a scene $\xi$ shown in Fig.~\ref{fig:lossalphajustification} that consists of a set of non-overlapping event-generating spatiotemporal surfaces $C$. Let $e_i(t, u)$ be the different realizations of events from the same scene, where $i$ denotes the realization and $t \in \integers_+, u \in \Om$. Let us denote the area of $C$ by $\text{a}(C)$ and $N = \sum_{s,u} \E_i[e_i(s,u)]$ be the average number of events, with $\mu = N/\text{a}(C)$.
Lemma.~\ref{lem:stochastic} motivates the objective in Eqn.~\ref{eq:objective_detail} using this scene.

\begin{lemma}
\label{lem:stochastic}
If $e_i(t,u)$ is a Bernoulli random variable with parameter $\mu$ that is independent and identically distributed for all $(t,u) \in C$ and zero when $(t,u) \notin C$. If $\hat e: \integers_+ \times \Om \to \{1, 0\}$ minimizes
\[
    \ell(e, \hat e) = \sum_{i} \sum_{t,u} \ind{e_i(t,u) \neq \hat e(t,u)} \rbr{\a e_i(t,u) + (1-\a)(1 - e_i(t,u))},
\]
then $\hat e(t,u) = \ind{(t,u) \in C}$ when $\a > 1 - 2 \mu$ and $\hat e(t,u) = 0$ otherwise.
\end{lemma}
\begin{proof}
Observe that
\[
    \ell = \sum_{i} \sum_{t,u \in C} \ind{e_i(t,u) \neq \hat e(t,u)} \rbr{\a e_i(t,u) + (1-\a)(1 - e_i(t,u))} + \sum_{u,t \notin C} \ind{\hat{e}(u,t) \neq 0}.
\]
The minimum is achieved when $\hat e(u,t)$ equals zero for all $t,u \notin C$. Now since $e_i(u,t)$ are independent and identically distributed, if we set $\hat e \equiv \hat e(t,u)$
\[
\aed{
    \argmin_{\hat e} \ell
    &= \argmin_{\hat e} \sum_{i} \sum_{t,u \in C} \rbr{e_i(1 - \hat e) + (1 - e_i)\hat e} \rbr{\a e_i + (1-\a)(1 - e_i)}\\
    &= \argmin_{\hat e} \sum_{u,t \in C} \hat e \E_{i}[1 - \a - 2e_i (1-2\a) - 4\a e_i^2]\\
    &= \argmin_{\hat e} \hat e \rbr{ 1 - \a - 2 \mu}.
}
\]
This shows that for all $(t,u)$, the predicted events $\hat e(t,u) = 1$ when $\a > 1 - 2\mu$ while $\hat e(t,u) = 0$ when $\a \leq 1 -2 \mu$ or when $(t,u) \notin C$.
\end{proof}

Intuitively, when events are very sparse and stochastic, the representation that seeks to predict future events will be trivial. It will predict all zeros if the weight $\a$ is not chosen carefully. In typical scenes, the sparsity of events is $\sim$95\%. The above lemma suggests that we need $\a > 0.9$ to learn a non-trivial representation (here $\a=0.5$ corresponds to giving equal weight to events or non-events) .

\section{Experimental details}
\label{s:app:details}

This section provides details of the training procedures for $\ff$ and downstream tasks. It also discusses pre-processing and evaluation methodology on the different datasets used in our analysis. In some cases, e.g., optical flow estimation, we calculated ground-truth labels using geometric vision techniques. We will release all data publicly to aid reproducibility.

\begin{table}[h!]
    \centering
    \caption{\textbf{Summary of M3ED sequences used in this work.}
    This table lists all M3ED sequences used in our experiments, indicating the availability and usage of each modality for training and testing.}
    \label{tab:m3ed_sequences}

    \footnotesize
    \renewcommand{\arraystretch}{1.3}
    \adjustbox{max width=0.81\textwidth}{
        \begin{tabular}{cll cccc}
            \toprule
            \shortstack[l]{\textbf{Robotic}\\ \textbf{Platform}} & \shortstack[l]{\textbf{Train/Test}\\ \textbf{Split}} & \textbf{Sequence} & \shortstack[l]{\textbf{Semantic}\\ \textbf{Segmentation}} & \shortstack[l]{\textbf{Pseudo}\\ \textbf{Depth}} & \shortstack[l]{\textbf{LiDAR}\\ \textbf{Depth}} & \shortstack[l]{\textbf{Optical}\\ \textbf{Flow}} \\
            \midrule
            \multirow{8}{*}{\rotatebox{90}{\textbf{Car}}} & \multirow{4}{*}{\shortstack[l]{\textbf{Daytime}\\\textbf{Train}}} & urban day penno big loop      & \cmark & \cmark & \cmark & \cmark \\
            & & urban day penno small loop    & \cmark & \cmark & \cmark & \cmark \\
            & & urban day ucity big loop      & \cmark & \cmark &  & \cmark \\
            & & urban day city hall           & \cmark & \cmark & \cmark & \cmark \\
            \cmidrule{2-7}
            & \multirow{2}{*}{\shortstack[l]{\textbf{Daytime}\\\textbf{Test}}} & urban day rittenhouse         & \cmark & & \cmark & \cmark \\
            & & urban day ucity small loop    & \cmark &  & \cmark & \cmark \\
            \cmidrule{2-7}
            & \multirow{2}{*}{\shortstack[l]{\textbf{Nighttime}\\\textbf{Test}}} & urban night rittenhouse       & & & & \cmark \\
            & & urban night ucity small loop  & & & & \cmark \\
            \midrule
            \multirow{17}{*}{\rotatebox{90}{\textbf{Spot}}} & \multirow{13}{*}{\textbf{Train}} & indoor building loop  & & \cmark & \cmark & \cmark \\
            & & indoor stairs                & & \cmark & \cmark & \cmark \\
            & & indoor stairwell             & & \cmark & \cmark & \cmark \\
            & & outdoor day art plaza loop   & & \cmark & \cmark & \cmark \\
            & & outdoor day rocky steps      & & \cmark & \cmark & \cmark \\
            & & outdoor day skatepark 1      & & \cmark & \cmark & \cmark \\
            & & outdoor day skatepark 3      & & \cmark & & \cmark \\
            & & outdoor day srt green loop   & & \cmark & \cmark & \cmark \\
            & & outdoor day srt under bridge 1 & & \cmark & \cmark & \cmark \\
            & & outdoor day penno building loop & & & & \cmark \\
            & & outdoor night penno building loop & & & & \cmark \\
            & & outdoor night penno plaza lights  & & & \cmark & \cmark \\
            \cmidrule{2-7}
            & \multirow{4}{*}{\textbf{Test}} & indoor obstacles & & & \cmark & \cmark \\
            & & outdoor day penno short loop      & & & \cmark & \cmark \\
            & & outdoor day skatepark 2           & & & \cmark & \cmark \\
            & & outdoor day srt under bridge 2    & & & \cmark & \cmark \\
            & & outdoor night penno short loop    & & & \cmark & \cmark \\
            \midrule
            \multirow{15}{*}{\rotatebox{90}{\textbf{Falcon}}} & \multirow{10}{*}{\textbf{Train}} & indoor flight 2      & & \cmark & \cmark & \cmark \\
            & & indoor flight 3             & & \cmark & \cmark & \cmark \\
            & & outdoor day fast flight 2   & & \cmark & \cmark & \cmark \\
            & & outdoor day fast flight 3   & & \cmark & & \cmark \\
            & & outdoor day penno cars      & & \cmark & \cmark & \cmark \\
            & & outdoor day penno parking 2 & & \cmark & \cmark & \cmark \\
            & & outdoor day penno parking 3 & & \cmark & & \cmark \\
            & & outdoor day penno trees     & & & \cmark & \cmark \\
            & & outdoor night high beams    & & & \cmark & \cmark \\
            & & outdoor night penno parking 2 & & & \cmark & \cmark \\
            \cmidrule{2-7}
            & \multirow{5}{*}{\textbf{Test}} & indoor flight 1 & & & \cmark & \cmark \\
            & & outdoor day fast flight 1   & & & \cmark & \cmark \\
            & & outdoor day penno parking 1 & & & \cmark & \cmark \\
            & & outdoor day penno plaza     & & & \cmark & \cmark \\
            & & outdoor night penno parking 1 & & & \cmark & \cmark \\
            \bottomrule
        \end{tabular}
    }
\end{table}

\subsection{$\ff$}
\label{app:details:f3}

For all experiments while training $\ff$ across different datasets and robotic platforms, we use the AdamW optimizer with a learning rate schedule that decays linearly from $5\times 10^{-5}$ to $5\times 10^{-6}$, a weight decay of 0.01, and a batch size of 8. Training is conducted for a total of 200 epochs.

For the M3ED dataset, when evaluating $\ff$ on each platform (car, spot and falcon), we train separate models on data from each platform using the training splits in Tab.~\ref{tab:m3ed_sequences}. Note that training $\ff$ only requires event data (not the other modalities). In total, this amounts to approximately 0.42, 0.48, and 0.3 hours of event recordings to train $\ff$  on car, spot and falcon, respectively.
$\ff$ trained on DSEC uses the standard training split provided in the original dataset. This split includes 41 sequences, a mix of day and night driving scenarios, totaling approximately 0.7 hours of event data.
For MVSEC, we train $\ff$ using only the event data from the ``outdoor\_day2'' sequence, which is the standard training sequence used in prior works to perform downstream tasks. This sequence contains approximately 0.18 hours of event data.

\subsection{Semantic segmentation}

We keep the optimizer and hyper-parameters fixed across all methods and datasets during training. We use the AdamW optimizer with a learning rate schedule that decays linearly from $6\times 10^{-5}$ to $6\times 10^{-6}$ and a weight decay of 0.01. A batch size of 8 is used throughout. We train for a total of 200 epochs on M3ED and 100 epochs on DSEC.

\new{%
For both DSEC and M3ED semantic segmentation tasks, target classes are the same. Both of them use a reduced merged set of 11 labels out of the 19 Cityscapes labels, as introduced in DSEC-Semantic. These 11 target classes are background, building, fence, person, pole, road, sidewalk, vegetation, car, wall, and traffic sign.
}

\paragraph{Evaluation on M3ED}
We train networks on pseudo-labeled segmentation masks obtained on RGB images corresponding to sequences listed under the ``car daytime train'' section in Tab.~\ref{tab:m3ed_sequences}. Every other event-label pair is skipped. RGB images and semantic labels for the ``car urban day ucity big loop'' sequence are not publicly available. So we used grayscale images in the dataset to generate our own pseudo-labels using InternImage \cite{wang2023internimage} (which was also used in the original dataset). This gives a total 12,356 event-label pairs for training. As test data, we use sequences listed under ``car daytime test'' in Tab.~\ref{tab:m3ed_sequences}. We only evaluate on event-label pairs where 20 ms windows centered on the semantic label timestamps contain more than 200K events; this gives a total of 13,501 test samples.

\paragraph{Evaluation on DSEC}
We conduct two experiments: (1) testing cross-dataset transfer from M3ED to DSEC, and (2) evaluating directly on the DSEC train-test split. To evaluate transfer, we train models using the full ``car daytime'' semantic segmentation dataset of M3ED. All available car daytime train and test sequences in Tab.~\ref{tab:m3ed_sequences} are used for training. These models are evaluated on DSEC and reported under ``Not trained on DSEC'' in Fig.~\ref{fig:seg:dsec}. For the ``Trained on DSEC'' section in the figure, we train models from scratch following the train-test split introduced in ESS \cite{essdaniel}.

\subsection{Optical flow estimation}

To learn $\psi$, we use the same optimizer and hyper-parameters across all methods and datasets during training. We use the AdamW optimizer with a fixed learning rate of $10^{-3}$ and a weight decay of 0.01. We used a batch size of 8, and trained for 100 epochs on each dataset.

\begin{figure}[h!]
    \centering
    \includegraphics[width=0.6\linewidth]{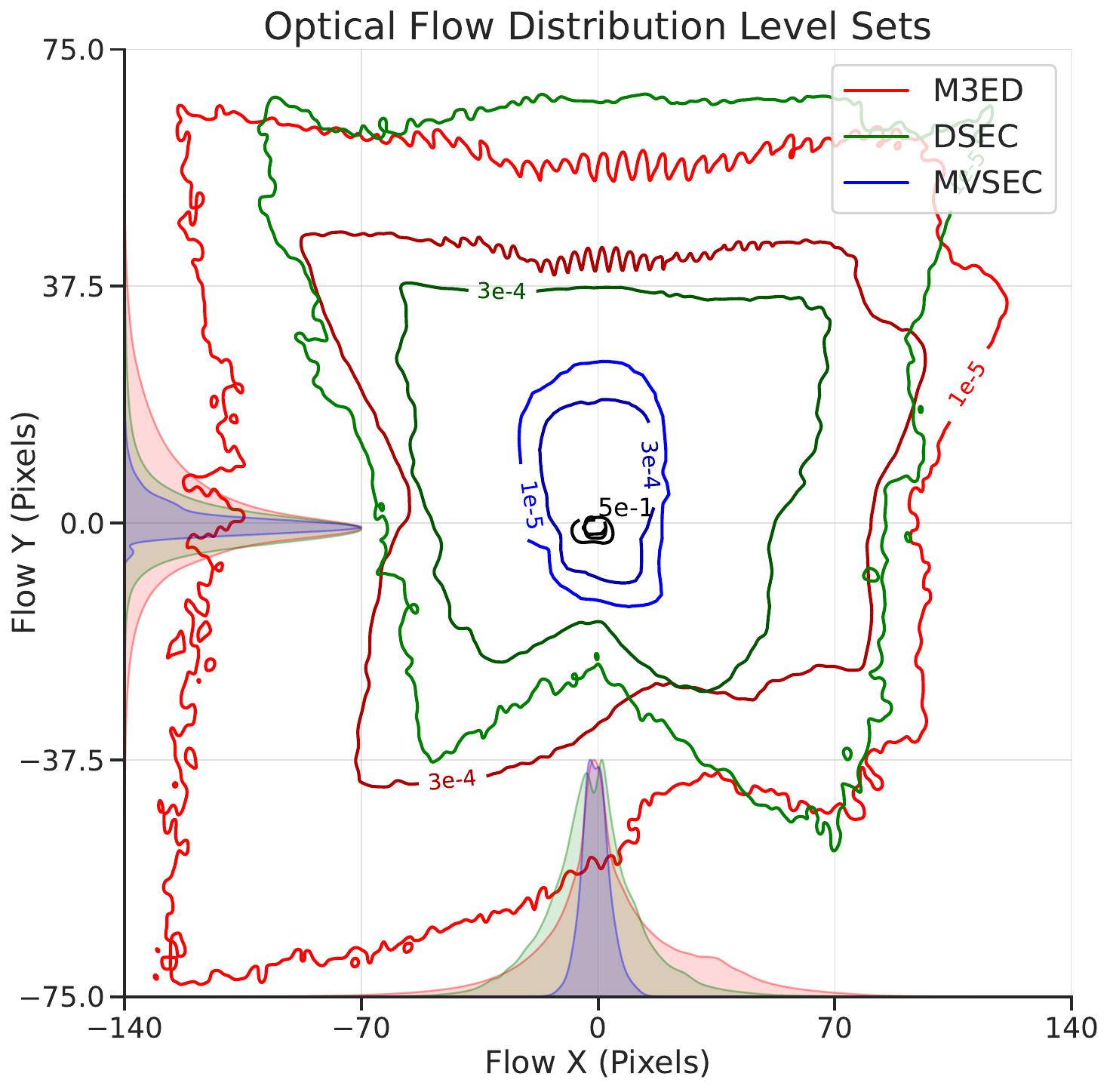}
    \caption{\textbf{Optical flow for different sensors and environments is quite different, especially M3ED and DSEC.}
    Due to differences in sensor resolutions, the typical motion patterns of different robotic platforms, the magnitude of ground-truth optical flow can vary enormously. A contour at 0.5 indicates that flow is within that region for 50\% of the data. For low-resolution data in MVSEC, the flow is larger along the Y-axis, but not more than about 30 pixels. For higher resolution MVSEC and M3ED data, the magnitude of flow is much larger overall, but with skewed probability distributions.}
    \label{fig:flow:histogram}
\end{figure}

\begin{figure}[t]
    \centering
    \includegraphics[width=0.7\linewidth]{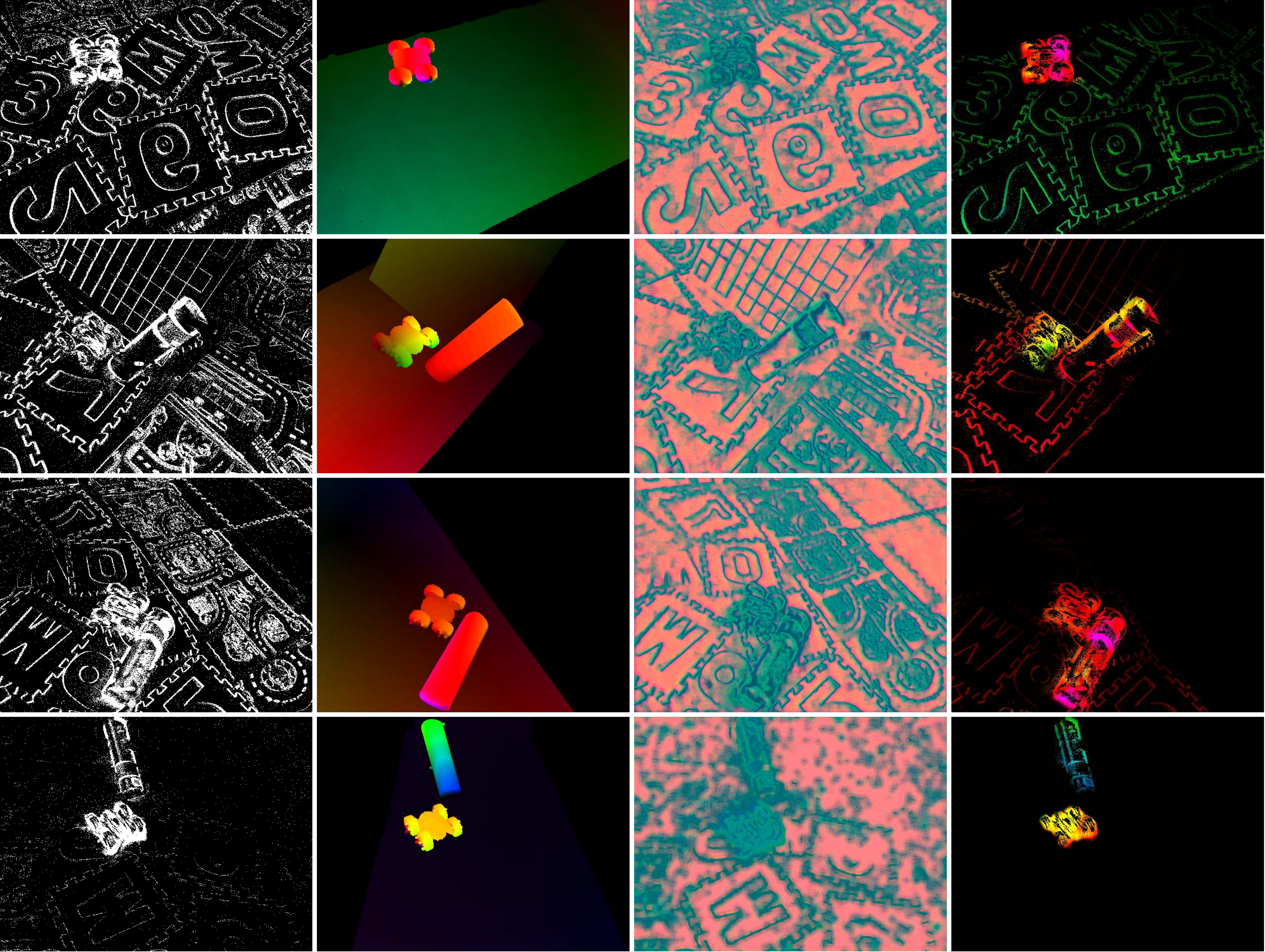}
    \caption{\new{\textbf{Example scenes from the EVIMO2 dataset demonstrating $\ff$-based optical flow estimation under randomized handheld camera motion and multiple independently moving objects.} From left to right, the columns show 20 ms of events plotted as a frame, ground-truth optical flow over the next 20 ms, top three PCA components of $\ff$ computed over the displayed events, and $\ff$-based optical flow estimated from the computed $\ff$ and filtered with active event and valid ground-truth pixel locations. The top row shows a scene with the camera moving on the left and a remote-controlled car moving to the right. Rows 2 and 3 show scenes with the remote-controlled car hitting a cylinder, causing it to tip over or roll, respectively. Bottom row shows a scene with a cylinder tipping over towards the top left and a car moving to the bottom right. These scenes capture very different local speeds and directions of moving objects -- generating interesting event patterns that encode these dynamics.
    $\ff$, using these events, captures the local structure and motion of the scene, and is robust to the randomized handheld camera motion and the multiple independently moving objects.}}
    \label{fig:evimo2}
\end{figure}

\paragraph{Evaluation on M3ED}
We use event sequences from the car, spot and falcon robotic platforms to train our unsupervised event-based optical flow prediction networks, as described in Sec.~\ref{s:flow}. The ``Car Daytime,'' ``Spot,'' and ``Falcon'' models shown in Fig.~\ref{fig:flow:m3ed} are trained on their respective training sequences listed in Tab.~\ref{tab:m3ed_sequences}. Since we only use events to train these networks, this adds up to the same amount of data used to train the respective $\ff$ models, discussed above in the details of the training procedure for $\ff$.

We need ground truth optical flow to evaluate these models, which is not publicly available for M3ED. However, most of the M3ED sequences have time-synchronized and reprojected LiDAR depth maps along with the camera poses computed from LiDAR odometry (Faster-LIO \cite{bai2022faster}). We adopt a similar strategy as that of MVSEC \cite{mvsec} to compute ground truth optical flow using these known quantities.

Suppose we would like to compute the per-pixel flow given the poses of the event camera at two time instances, $t_0$ and $t_0 + \D t$, specified as rotation matrices $R_t \in \mathrm{SO}(3)$ and translation vectors $p_t \in \R^3$, along with the depth map $Z_{t_0} \in \R_+ \times \Om$ at time $t_0$.
Under a linearity assumption, translational and angular velocities for the event camera frame, $\dot p_{t_0}$ and $\w_{t_0}$ respectively, are
\[
    \aed{
    \dot p_{t_0} &= \frac{p_{t_0+ \D t} - p_{t_0}}{\D t}\\
    \hat{\w}_{t_0} &= \frac{\log \rbr{R_{t_0}^\top R_{t_0 + \D t}}}{\D t},
    }
\]
where $\log(\cdot)$ denotes the matrix logarithm and $\hat \w$ is the skew-symmetric representation of the angular velocity $\w \in \R^3$. To mitigate any abrupt changes in the velocities, we apply a 10-point moving average filter to $\dot{p}$ and $\w$. Now, let $u = (u_1, u_2) \in \Om$ denote the coordinates of the undistorted event camera pixel space and $Z = Z_{t_0}(u) \in \R_{+}$ be its corresponding depth in the scene at time $t_0$. The image-plane motion field due to ego-motion
\beqs{
    \dot{u} =
    \begin{bmatrix}
    - \frac{1}{Z} & 0 & \frac{u_1}{Z} & u_1u_2 & -(1 + u_1^2) & u_2 \\
    0 & - \frac{1}{Z} & \frac{u_2}{Z} & 1 + u_2^2 & -u_1u_2 & -u_1
    \end{bmatrix}
    \begin{pmatrix}
    \dot{p}_{t_0}\\
    \w_{t_0}.
    \end{pmatrix}
}
This formulation differentiates the perspective projection of 3D points in a rigid body motion. Hence, $\dot u$, although it captures the effect of egomotion on static objects, fails to account for independently moving objects in the scene. The pixels where independently moving objects are present are marked as invalid and not used for evaluation (such pixels have been computed beforehand in the original M3ED dataset).
Optical flow is obtained by scaling the instantaneous motion field $\dot u$ by the time interval $\D t$ between the depth and pose measurements. Hence, $\dot u \D t$ is the 2D displacement field that represents the apparent motion of pixel $(u_1, u_2)$ from time $t_0$ to $t_0 + \D t$. We use this quantity as the ground truth to evaluate our flow prediction models.

For evaluating on ``Car Daytime'' in Fig.~\ref{fig:flow:m3ed}, we use the corresponding sequences listed in Tab.~\ref{tab:m3ed_sequences}. Since the LiDAR frequency in M3ED is 10Hz, we set $\D t$ to 100 ms for calculating the ground truth flow. Depth maps provided in M3ED have been filtered for independently moving objects, but they are sometimes inaccurate. We manually go through the car daytime test sequences and remove all examples containing moving objects in the scene. We also filter out examples containing fewer than 100 valid flow measurement pixels or where 20 ms window centered around the optical flow start timestamp contains fewer than 200K events. This results in 5,774 test samples for a car driving in an urban and daytime setting.

To evaluate the robustness of our flow prediction framework under challenging illumination conditions, we test models trained on car daytime sequences on nighttime driving scenarios in Fig.~\ref{fig:robust:car_daytime_to_nighttime}. We generate ground truth optical flow for the car nighttime test sequences listed in Tab.~\ref{tab:m3ed_sequences} using the same procedure as above (without manual filtering because it is not easy to see objects in RGB images at nighttime).
After filtering based on valid flow pixel count and event rates, there are  5,216 testing examples.
Similarly, for the results in Fig.~\ref{fig:flow:m3ed} for spot and falcon, we evaluate the trained models on their respective test sequences listed in Tab.~\ref{tab:m3ed_sequences}. Unlike the car daytime sequences, these scenes contain few independently moving objects, and we did not perform manual filtering. However, we do filter samples based on the number of valid flow pixels and event rates. This yields 4,195 and 2,057 test examples for the spot and falcon platforms, respectively.

\begin{table}[!htpb]
    \centering
    \begin{adjustbox}{width=0.5\linewidth}
    \renewcommand{\arraystretch}{1.1}
    \begin{tabular}{l l r r r}
        \toprule
        \textbf{Test} & \textbf{Method} & \textbf{3PE} & \textbf{AEE} & \textbf{AAE} \\
        \textbf{dataset} &  & \textbf{(\%)} & & \textbf{(deg)} \\
        \midrule
        \multirow{4}{*}{\parbox{1.4cm}{Spot}} &
        {E-RAFT} \cite{eraft} & 89.26 & 18.72 & 41.00 \\
        & {$\ff$-RAFT} & 83.29 & 19.30 & 42.84 \\
        & $\vv$    & 97.58 & 27.88 & 59.41 \\
        & $\ff$    & 93.36 & 21.81 & 46.97 \\
        \midrule
        \multirow{4}{*}{\parbox{1.4cm}{Falcon}} &
        {E-RAFT} \cite{eraft} & {89.87} & {13.05} & {37.33} \\
        & {$\ff$-RAFT}  & {88.08} & {12.27} & {39.46} \\
        & $\vv$    & 95.25 & 18.99  & 41.67 \\
        & $\ff$    & 92.61 & 13.76  & 40.42 \\
        \bottomrule
    \end{tabular}
    \end{adjustbox}
    \caption{
    \textbf{M3ED optical flow estimation -- trained on car daytime, tested on spot and falcon.}
    Vast differences in the dynamics of the spot and falcon sequences compared to the car daytime sequences lead to poor transfer performance for all methods.
    For spot test sequences, $\ff$, E-RAFT, and $\ff$-RAFT drop in performance from 9.65, 8.65, and 8.13 AEE when trained on spot training sequences (Fig.~\ref{fig:flow:m3ed}), to 21.81, 18.72, and 19.30 AEE when trained on car daytime sequences, respectively.
    Similarly, significant drops in performance are observed for the Falcon test sequences when trained on car daytime sequences compared to training on Falcon training sequences.
    }
    \label{tab:flow:m3ed_car_daytime_to_other_robots}
\end{table}

\new{Additionally, in Tab.~\ref{tab:flow:m3ed_car_daytime_to_other_robots} we perform a quantitative evaluation of the optical flow estimation models trained on M3ED car daytime sequences on spot and falcon test sequences.
A qualitative demonstration of this transfer is shown in Fig.~\ref{fig:robust:car_to_spot_and_falcon}.
Tab.~\ref{tab:flow:m3ed_car_daytime_to_other_robots} verifies that the transfer performance is poor for all methods.
This is expected given the significant differences in the dynamics of the spot and falcon sequences compared to the car daytime sequences.
Spot and Falcon sequences contain jerky gait and motion with 6 degrees of freedom, respectively, compared to the smooth forward motion of a car.
This leads to very different distributions of optical flow in the spot, falcon, and car sequences.
}


\paragraph{Evaluation on DSEC} Since our proposed optical flow method in Sec.~\ref{s:flow} is unsupervised, we only use event data to train our models; ground truth optical flow obtained via LiDAR is not used. We use the train split suggested by the authors of the DSEC dataset for training. The results reported in Fig.~\ref{fig:flow:dsec} are obtained using the evaluation procedure on the DSEC website \footnote{\url{dsec.ifi.uzh.ch/uzh/dsec-flow-optical-flow-benchmark}}.

\paragraph{Evaluation on MVSEC}
We train our unsupervised method on the ``outdoor\_day2'' sequence and test on the 222.4s -- 240.4s interval of the ``outdoor\_day1'' sequence, as introduced in \cite{evflownet}. There are some specific conventions that are often used in the existing literature for evaluation, which we also follow. We evaluate our optical flow estimates at 45Hz and 11.25Hz intervals, denoted as ``dt=1'' and ``dt=4'' in the literature. We evaluate flow only on the top 193 rows, which ignores the hood of the car under the camera. Additionally, pixels with valid ground truth flow and at least one event in the flow duration are considered while evaluating the estimate. We report the results for our method and baselines in Fig.~\ref{fig:flow:mvsec}.

\paragraph{Evaluation on EVIMO2} \new{We would like to demonstrate the applicability of our $\ff$-based unsupervised optical flow estimation method beyond driving, quadruped, and quadrotor flight scenarios. 
$\ff$ is designed to capture the local structure and motion of a scene with great detail. Even with randomized fast motion such as handheld camera motion, predicting the future event spatio-temporal locations is possible over short future time horizons.}

\new{To demonstrate this, we train $\ff$-based unsupervised optical flow on the EVIMO2 dataset \cite{burner2022evimo2eventcameradataset} (``samsung\_mono'' train-test split), containing handheld event recordings in a variety of table-top scenes with multiple independently moving objects.
We show qualitative examples of $\ff$-based unsupervised optical flow on the test sequences in Fig.~\ref{fig:evimo2}.
}

\subsection{Monocular depth estimation}

Just like the training of $\ff$, semantic segmentation, and optical flow estimation, we use the same optimizer and hyper-parameters for all our pseudo and metric $\psi$ disparity models. We use the AdamW optimizer with a fixed learning rate of $6\times 10^{-6}$ and a weight decay of 0.01. A batch size of 8 is used throughout. Both training stages---pseudo depth prediction and metric depth prediction---are run for 100 epochs.

\paragraph{Evaluation on M3ED} To train the pseudo-depth model in stage-1, we generate monocular relative disparity maps from RGB images using the Depth Anything V2-Large \cite{depth_anything_v2} model. These maps are reprojected into the event camera frame. We use the pseudo-labeled sequences listed under car daytime, spot, and falcon train splits in Tab.~\ref{tab:m3ed_sequences}. This results in approximately 45K, 22K, and 17K valid pairs of events and pseudo-labeled disparity maps for training the car daytime, spot, and falcon stage-1 models, respectively.

To train the second stage, fine-tuning on metric depth for retaining sharp object boundaries, we use LiDAR ground truth data present in the M3ED sequences in conjunction with relative depth estimates from stage-1 as discussed in Sec.~\ref{s:depth}. Ground truth LiDAR depth maps marked available in the car daytime, spot, and falcon train splits in Tab.~\ref{tab:m3ed_sequences} are used for training the second stage for that particular robotic platform. In M3ED, some sequences do not have LiDAR observations in Tab.~\ref{tab:m3ed_sequences}. Also note that we add some of the available nighttime sequences in this stage of the training for the spot and falcon platforms. Pseudo-labeled disparity maps generated from these nighttime sequences are not used in stage-1 training, since the RGB images are too dark to recover any meaningful monocular depth information. This provides us with a total of 5,223, 8,724, and 6,957 samples for training the second stage on car, spot and falcon, respectively.

We evaluate these car daytime, spot and falcon monocular metric depth (second-stage) models on their respective test sequences listed in Tab.~\ref{tab:m3ed_sequences}. For all sequences, we evaluate only on depth maps that contain at least 10 valid LiDAR points within a maximum depth of 80 meters, and where the 20 ms window centered on the depth timestamp contains more than 200K events. Additionally, we exclude 15\% of the depth samples from the start and end of the falcon sequences, because depth measurements during drone takeoff and landing can be highly erratic. After applying these criteria, we obtain 5,387, 4,756, and 2,118 test examples for car, spot and falcon, respectively. For all metrics reported in Fig.~\ref{fig:depth:m3ed}, evaluation is restricted to pixels with valid LiDAR points at depths less than 80 meters.

\begin{table}[!htpb]
    \centering
    \begin{adjustbox}{width=0.75\linewidth}
    \renewcommand{\arraystretch}{1.2}
    \begin{tabular}{l l r r r r r}
        \toprule
        \textbf{Test} & \textbf{Method} & \textbf{Relative} & \textbf{RMSE} & \multicolumn{3}{c}{\textbf{Pixels (\%) with} $\boldsymbol\delta$ \textbf{below}} \\
        \textbf{dataset} & & \textbf{$\ell_1$ error} & \textbf{(m)} & $\mathbf{1.25}$ & $\mathbf{1.25^2}$ & $\mathbf{1.25^3}$ \\
        \midrule
        \multirow{3}{*}{\parbox{1.3cm}{Spot}} &
        {EReFormer} \cite{Ereformer} & 3.83 & 19.21 & 0.10 & 0.20 & 0.30\\
        & $\vv$    & 1.63 & 9.70 & 0.13 & 0.23 & 0.34 \\
        & $\ff$    & 1.50 & 9.69 & 0.13 & 0.24 & 0.35 \\
        \midrule
        \multirow{3}{*}{\parbox{1.3cm}{Falcon}} &
        {EReFormer} \cite{Ereformer} & 1.51 & 14.04 & 0.22 & 0.42 & 0.58 \\
        & $\vv$    & 0.96 & 10.00 & 0.28 & 0.52 & 0.69\\
        & $\ff$    & 0.88 & 9.02 & 0.33 & 0.57 & 0.72\\
        \bottomrule
    \end{tabular}
    \end{adjustbox}
    \caption{
    \textbf{M3ED monocular depth estimation -- trained on car daytime, tested on spot and falcon.}
    Vast differences in dynamics and depth ranges of the spot and falcon sequences compared to the car daytime sequences lead to poor transfer performance for all methods.
    For spot test sequences, $\ff$ and $\vv$-based depth estimates have an RMSE of 2.80 and 3.09 m when trained on spot training sequences (Fig.~\ref{fig:depth:m3ed}), which increases to 9.69 and 9.70 m when trained on car daytime sequences, respectively. For falcon test sequences, $\ff$ and $\vv$-based depth estimates have an RMSE of 3.44 and 3.53 m when trained on falcon training sequences, which increases to 9.02 and 10.00 m when trained on car daytime sequences, respectively.
    }
    \label{tab:depth:m3ed_car_daytime_to_other_robots}
\end{table}

\new{
Additionally, in Tab.~\ref{tab:depth:m3ed_car_daytime_to_other_robots} we perform a quantitative evaluation of the monocular depth estimation models trained on M3ED car daytime sequences on spot and falcon test sequences.
A qualitative demonstration of this transfer is shown in Fig.~\ref{fig:robust:car_to_spot_and_falcon}.
In Tab.~\ref{tab:depth:m3ed_car_daytime_to_other_robots}, similar to our optical flow estimation observations, monocular depth estimation models trained on M3ED car daytime sequences show poor transfer performance on spot and falcon test sequences.
The reasons pertain to significant differences in the dynamics and depth distributions of the spot, falcon, and car sequences, and scene-specific factors as elaborated in Sec.~\ref{s:robustness}.
}

\paragraph{Evaluation on DSEC} For gradient supervision in Fig.~\ref{fig:depth:dsec}, we use the same stage-1 model trained on car daytime sequences as described in the evaluation details for M3ED. In other words, we do not generate pseudo-labeled depth maps using monocular RGB images for DSEC---stage-1 is not retrained on DSEC pseudo-labels. For the second stage of training, or just training for metric disparity from scratch, we use the train split provided in the original DSEC dataset. We evaluate our methods on the disparity benchmark on the DSEC website \footnote{\url{https://dsec.ifi.uzh.ch/uzh/disparity-benchmark}}.

\paragraph{Evaluation on MVSEC} As described in Sec.~\ref{s:depth}, MVSEC event camera resolution is quite low (346 $\times$ 260) and therefore we do not need to use the two-stage training approach that we used for M3ED and DSEC. We directly train on events and metric depth map pairs. Following prior works like \cite{learningmonoculardensedepth}, we use the ``outdoor\_day2'' sequence in MVSEC for training. Two sequences ``outdoor\_day1'' and ``outdoor\_night1'' are used for testing. Unless mentioned otherwise, for all the metrics in Fig.~\ref{fig:depth:mvsec}, we only evaluate the methods on pixels with ground truth depth less than 80 meters.

\subsection{Stereo Depth}

$\ff$ features are spatially consistent and encode meaningful information about the structure in the scene. This enables depth estimation (with correct scale) by matching $\ff$ features computed for a stereo pair of event cameras. We illustrate this idea under the ``Stereo Disparity'' section of downstream tasks in Fig.~\ref{fig:architecture}. To emphasize the spatial consistency of $\ff$, we demonstrate that traditional RGB-based block matching methods can effectively match stereo $\ff$ features at a given time instant.

We extract events from the stereo camera pair up to time $t$ and pass them to the trained featurizer, obtaining two $\ff(t,\cdot)$ feature maps---one per camera. These $p$-channel features are then rectified assuming known camera intrinsics and the relative extrinsic transformation between the stereo views. We reduce the dimensionality of these rectified feature maps by retaining only the top three principal components for each view. Disparity is then computed using the resulting 3-channel images using the OpenCV implementation of Semi-Global Block Matching (SGBM) \cite{stereosgbm}. This approach mirrors the one used for optical flow in Sec.~\ref{s:flow}, where we matched $\ff$ across time. Here, we instead perform matching across space, leveraging the same structural properties of the $\ff$ representation.

All three datasets discussed in this paper---M3ED, DSEC, and MVSEC---are equipped with stereo event cameras, making it possible to analyze this proposal further. We show qualitative results for the stereo-matching algorithm on M3ED sequences in Fig.~\ref{fig:depth:stereo}. This task is quite challenging, given the diversity of data in M3ED and the use of a simple feature similarity-based matcher in our method. We show that we can match $\ff$ features under vastly different lighting conditions and scenes, even with the same choice of hyper-parameters in the SGBM algorithm. This demonstrates the spatial consistency of $\ff$, under pixel-wise similarity metrics.


\section{\new{Additional Results}}
\label{s:app:additional_results}

\subsection{$\ff$ representations are bandwidth efficient}

To study and contextualize the bandwidth efficiency of $\ff$ representations, we compare the performance of $\ff$-RAFT, E-RAFT \cite{eraft}, and RGC-lin \cite{So_2025_NGS} for optical flow estimation on the TartanAir \cite{tartanair2020iros} dataset. This analysis, reported in Tab.~\ref{tab:ttairv2}, is similar to the one presented in \cite{So_2025_NGS}[Table 1].
Here, the RAFT-based methods are trained and evaluated on the same set of events generated with a contrast threshold-based event simulator \cite{das2026neurosim}, mimicking a standard dynamic vision sensor.
Whereas, So et al. \cite{So_2025_NGS} utilize custom events specialized per task called ``retinal ganglion cell (RGC) events''. These events are inspired by the mammalian retinal circuitry and are optimized to be more bandwidth-efficient for optical flow estimation. For metrics reported in Tab.~\ref{tab:ttairv2}, RGC-methods refer to an IDNet model \cite{IDNet} trained with ground-truth supervision on RGC events simulated in TartanAir sequences.

We evaluate the performance of these methods on the same TartanAir train-test split as mentioned in \cite{So_2025_NGS}[S1.2] for different event bandwidths.
For training and testing the RAFT-based methods, we simulate event data for the TartanAir sequences following the video interpolation details outlined in \cite{So_2025_NGS}[S1.2]. 
On these interpolated videos, we use a contrast threshold-based event simulator from Neurosim \cite{das2026neurosim} with different thresholds of 0.3, 0.5, and 0.7, to simulate events with different bandwidths.
We generate ground-truth optical flow for training and testing, following the TartanAir documentation\footnote{\url{https://tartanair.org/modalities.html}}.

For E-RAFT and $\ff$-RAFT, we follow the same training procedure as outlined in Sec.~\ref{s:flow} and under ``Implementation details of downstream tasks'' in Sec.~\ref{s:app:details}. We train these models on the events generated with a contrast threshold of 0.3 and test on all three contrast thresholds.
In Tab.~\ref{tab:ttairv2}, we report the bandwidth of events for each contrast threshold, and the performance of these methods in terms of average endpoint error (AEE), 1PE, and 3PE metrics. These metrics are described in Fig.~\ref{fig:flow} (d-f).
We find that E-RAFT is quite bandwidth inefficient -- RGC-lin$_\text{lite}$ achieves a lower AEE (2.75 vs 2.90) than E-RAFT with almost half the bandwidth (2.10M vs 4.11M events per second).
However, $\ff$-RAFT is significantly more bandwidth efficient than E-RAFT. This observation is consistent with our subsampling robustness experiments in Fig.~\ref{fig:robust:subsampling}.
$\ff$-RAFT is similar to RGC methods in terms of bandwidth efficiency. With a lower bandwidth of 1.76M events per second, $\ff$-RAFT achieves slightly worse AEE (2.84 vs 2.75) than RGC-lin$_\text{lite}$. With a higher bandwidth (2.77M and 4.11M events per second), $\ff$-RAFT achieves better AEE (2.51 and 2.16, respectively) than RGC-lin$_\text{lite}$. We make similar observations for the RGC-lin model.
This analysis demonstrates that even without task-specialized events (like RGC events), $\ff$-based models achieve competitive bandwidth efficiency by leveraging general dynamic vision sensor-like events.

\begin{table}[!htpb]

\centering
\begin{adjustbox}{width=\linewidth}
\renewcommand{\arraystretch}{1.2}
\begin{tabular}{l rr rr rr | rr}

\toprule
& \multicolumn{2}{c}{DVS (threshold 0.3)} & \multicolumn{2}{c}{DVS  (threshold 0.5)} & \multicolumn{2}{c}{DVS  (threshold 0.7)} & & \\
\cmidrule(lr){2-3} \cmidrule(lr){4-5} \cmidrule(lr){6-7}
& E-RAFT & $\mathrm{F}^3$-RAFT & E-RAFT & $\mathrm{F}^3$-RAFT & E-RAFT & $\mathrm{F}^3$-RAFT & RGC-lin$_\text{lite}$ & RGC-lin \\
\midrule
Bandwidth (Ms$^{-1}$)        & 4.11 & 4.11 & 2.77 & 2.77 & 1.76 & 1.76 & 2.10 & 3.80 \\
$\downarrow$ AEPE & 2.90  & 2.16  & 3.19  & 2.51  & 3.57  & 2.84  & 2.75  & 2.42 \\
$\downarrow$ 1PE  & 59.3  & 49.1  & 64.9  & 58.6  & 70.7  & 66.0  & 52.4  & 47.1 \\
$\downarrow$ 3PE  & 22.1  & 14.2  & 25.9  & 18.7  & 29.8  & 22.2  & 20.7  & 17.4 \\
\bottomrule
\end{tabular}
\end{adjustbox}
\caption{\textbf{TartanAir \cite{tartanair2020iros} Optical Flow quantitative evaluation of different methods utilizing different event bandwidths.}
We show quantitative results for E-RAFT \cite{eraft} and $\ff$-RAFT trained on events simulated with a contrast threshold formulation \cite{das2026neurosim}. Different contrast thresholds (here 0.3, 0.5, 0.7) essentially change the bandwidth of the simulated events.
We compare the bandwidth-efficiency of our method $\ff$-RAFT with the baseline E-RAFT, and RGC-lin \cite{So_2025_NGS} models.
}
\label{tab:ttairv2}
\end{table}

\subsection{$\ff$ can be optimized for inference on edge hardware}
\label{s:app:jetson}

We first discuss the time required to compute the hash encoding in $\ff$ for different event rates in Tab.~\ref{tab:subsampling_speed}.

\begin{table}[!htpb]
\centering
\begin{adjustbox}{width=0.75\linewidth}
\renewcommand{\arraystretch}{1.2}
\begin{tabular}{r r r r r r}
\toprule
Factor & Events (M) & Hash (ms) & Hash Speedup & Full (ms) & Full Speedup \\
\midrule
1 & 1.0 & 1.989 & 1.000 & 8.656 & 1.000 \\
2 & 0.5 & 1.015 & 1.960 & 7.604 & 1.138 \\
4 & 0.25 & 0.567 & 3.507 & 7.073 & 1.224 \\
8 & 0.125 & 0.557 & 3.569 & 6.739 & 1.284 \\
10 & 0.1 & 0.557 & 3.569 & 6.654 & 1.301 \\
\bottomrule
\end{tabular}
\end{adjustbox}
\caption{
\textbf{Event encoding speed benchmark with different numbers of events.}
$\ff$ hash encoder is faster by about 3.5$\times$ when the event rate is 25\% of the original rate (1 million events within the 20 ms temporal window). The speedup for the sparse hash-encoding is not proportional to the reduction in the number of events because the overhead of CUDA kernel launches at these low event rates overshadows the fact that the GPU is processing fewer events. When the number of events is 25\% of the original number, $\ff$ computation is faster by about 1.2$\times$ (compare this to 3.5$\times$ for the hash encoder discussed previously). Dense convolutions are the bottleneck of $\ff$ at low event rates. This bottleneck can be alleviated using sparse convolutions as discussed above. With current implementations, for typical event rates, $\ff$ implemented with dense convolutions is faster than the one with sparse convolutions using TorchSparse++.
}
\label{tab:subsampling_speed}
\end{table}

Next, to test the real-time performance of $\ff$-based downstream methods for robotics, we
deploy $\ff$-based optical flow and monocular depth estimation models on a Jetson Orin Nano NX 16 GB, mounted on a quadrotor with a VGA event camera (DVXplorer Micro).
However, for system-on-module platforms like the Jetson Orin Nano, memory bandwidth and CUDA core counts are roughly 16$\times$ lower than PCIe-based desktop GPUs like RTX 4090, but the GPU clock speed is only $\sim$4$\times$ slower. Jetson-like edge hardware is more memory-bound than compute-bound, and hence, optimizing for memory access patterns and rates is crucial for low-latency inference.
We make three key optimizations to $\ff$ to achieve low latency on edge hardware, while retaining the core predictive representation principle of $\ff$.

\begin{enumerate}
    \item
    We fuse operations into a single kernel launch as much as possible. Computation intermediates remain in GPU registers and shared memory rather than multiple global DRAM reads and writes. We ensure that the inference graph of $\ff$ we build contains no breaks, e.g., CPU-GPU synchronization or data-dependent control flow. This allows the compiler backend (Torch Inductor) to fuse adjacent memory-bound operations such as layer normalization, element-wise activations, and residual additions. As discussed in Sec.~\ref{s:methods}, we already performed this optimization for training and inference, in the original implementation, on desktop GPUs.

    \item
    The $\ff$ hash encoder ($\varphi$) looks up and interpolates each event's pixel and time coordinate through $L$ resolution levels. In a naive implementation, this requires $8L$ independent hash lookups per event, where typically we set $L = 4$.
    During inference, we pre-compute the hash table for each spatio-temporal location, reducing memory accesses per event by 32$\times$, albeit at the cost of a larger memory footprint.
    This is particularly beneficial on the Jetson's 68 GBs$^{-1}$ memory bus. For context, RTX 4090 has a much larger 1008 GBs$^{-1}$ memory bandwidth.

    \item
    We redesigned the original $\ff$ pooling and smoothing ConvNext-V2 based architecture $(\rho)$ to further reduce memory access costs. For this Jetson-based implementation, we adopt a U-Net architecture for $\rho$. This new design down-samples spatial resolution aggressively in early layers, without proportionally increasing the channel dimension. This substantially reduces the size of intermediate feature tensors throughout the network.
    The majority of the computation now happens at these smaller tensors, which are up-sampled and decoded at the end with a small ConvNext-V2 block to build $\ff$ features at the original spatial resolution. Operating on smaller tensors for the majority of the computation reduces the total memory traffic of a forward pass, improving throughput on the memory-bandwidth-limited Jetson. Note that $\ff$ with this modified architecture also follows the core predictive representation principle proposed in the paper and the Deep Sets \cite{zaheer2017deep} architecture.

\end{enumerate}

Beyond these enhancements, we compile $\ff$ and the downstream networks using Ahead-Of-Time Inductor (AOTI) compilation via PyTorch 2 \cite{pytorch2}, while calibrating and quantizing the model to 16-bit floating point precision, for a higher memory throughput.
AOTI compilation serializes the optimized model into a shared library that can be loaded and executed directly through LibTorch in C++, without JIT compilation.
This is essential for deployment in ROS~2 C++ nodes -- required to enable communication protocols like shared memory IPC to handle the high volume of event camera data.

The inference latency of $\ff$ operating at VGA resolution on an RTX 4090 drops from 2.3 ms to 1.83 ms with the hash precomputation and further down to 0.78 ms with the redesigned $\rho$ architecture.%
\footnote{We also observe similar reductions in latency for $\ff$ on HD (1280$\times$720) resolution events on an RTX 4090 -- 8.35 ms ($\sim$120 Hz) to 2.23 ms ($\sim$450 Hz).}
On a Jetson Orin Nano NX 16 GB, we see a similar reduction in inference latency from 58.93 ms to 53.60 ms and then further down to 12.20 ms. Therefore, $\ff$ goes from being $\sim$25$\times$ slower on Jetson compared to RTX 4090, to being only $\sim$15$\times$ slower after performing these optimizations on both hardware. Memory-bound optimizations help more on the Jetson.
Table~\ref{tab:jetson_latency} reports the inference latencies of $\ff$ feature computation, $\ff$-based unsupervised optical flow, and monocular depth estimation on a Jetson Orin Nano NX 16 GB and an RTX 4090, considering VGA resolution events as input and VGA resolution outputs.

\begin{table}
\centering
\begin{adjustbox}{width=0.8\linewidth}
\renewcommand{\arraystretch}{1.2}
\begin{tabular}{l r r}
\toprule
\textbf{Model} & \textbf{RTX 4090 (ms)} & \textbf{Jetson Orin Nano NX (ms)} \\
\midrule
$\ff$                           & 0.78 & 12.20 \\
$\ff$ optical flow        & 2.78 & 28.86 \\
$\ff$ monocular depth (DepthAnything V2 S) & 2.00 & 25.01 \\
$\ff$ monocular depth (DepthAnything V2 B) & 2.98 & 45.66 \\
\bottomrule
\end{tabular}
\end{adjustbox}
\caption{
\textbf{Inference latency of $\ff$-based models on RTX 4090 and Jetson Orin Nano NX 16 GB, for VGA resolution event data.}
$\ff$-based optical flow and monocular depth estimation latencies range from 25-45 ms on the Jetson, which is suitable for real-time robotics applications using edge hardware.
}
\label{tab:jetson_latency}
\end{table}

$\ff$ after these optimizations still follows the core predictive representation principle and the Deep Sets \cite{zaheer2017deep} architecture. These modifications should yield similar performance metrics and robustness on downstream tasks -- albeit with much lower latency. We show evidence of this in Fig.~\ref{fig:depth:m3ed}. The low-latency variant of $\ff$-based monocular depth estimation model is 35\% faster than the original model, while achieving similar evaluation metrics on M3ED car daytime test sequences.

\end{appendix}

\end{document}